\documentclass{article}
\expandafter\def\expandafter\normalsize\expandafter{%
    \normalsize
    \setlength\abovedisplayskip{2pt}
    \setlength\belowdisplayskip{2pt}
    \setlength\abovedisplayshortskip{2pt}
    \setlength\belowdisplayshortskip{2pt}
}
\usepackage{microtype}
\usepackage{graphicx}
\usepackage{subfigure}
\usepackage{booktabs} % for professional tables
\usepackage{amsthm}
\usepackage{mdwmath}
\usepackage{color}
\usepackage{mdwtab}
\usepackage{amsmath}
\usepackage{amsfonts}
\usepackage{enumitem}
\usepackage{caption}
\usepackage{makecell}
\usepackage{authblk}
\usepackage{algorithmic}
\usepackage{algorithm}
\usepackage[noadjust]{cite}
\usepackage{dsfont}
\usepackage{hyperref}
\usepackage[preprint,nonatbib]{neurips_2020}
\newtheorem{theorem}{Theorem}
\newtheorem{Definition}{Definition}

\newtheorem{Corollary}{Corollary}
\newtheorem{Lemma}{Lemma}
\newtheorem{Remark}{Remark}

\newtheorem{Assumption}{Assumption}

\title{Stochastic-Sign SGD for Federated Learning with Theoretical Guarantees}
\author{%
\textbf{Richeng Jin}\thanks{North Carolina State University, Email: rjin2@ncsu.edu.},~~~
\textbf{Yufan Huang}\thanks{North Carolina State University, Email: yhuang20@ncsu.edu.},~~~
\textbf{Xiaofan He}\thanks{Wuhan University, Email: xiaofanhe@whu.edu.cn.},~~~
\textbf{Huaiyu Dai}\thanks{Corresponding author; North Carolina State University, Email: hdai@ncsu.edu.},~~~
\textbf{Tianfu Wu}\thanks{North Carolina State University, Email: tianfu\textunderscore wu@ncsu.edu.}
}
\begin{document}
\maketitle
\begin{abstract}
Federated learning (FL) has emerged as a prominent distributed learning paradigm. FL entails some pressing needs for developing novel parameter estimation approaches with theoretical guarantees of convergence, which are also communication efficient, differentially private and Byzantine resilient in the heterogeneous data distribution settings. Quantization-based SGD solvers have been widely adopted in FL and the recently proposed {\scriptsize SIGN}SGD with majority vote shows a promising direction. However, no existing methods enjoy all the aforementioned properties. In this paper, we propose an intuitively-simple yet theoretically-sound method based on {\scriptsize SIGN}SGD to bridge the gap. We present Stochastic-Sign SGD which utilizes novel stochastic-sign based gradient compressors enabling the aforementioned properties in a unified framework. We also present an error-feedback variant of the proposed Stochastic-Sign SGD which further improves the learning performance in FL. We test the proposed method with extensive experiments using deep neural networks on the MNIST dataset and the CIFAR-10 dataset. The experimental results corroborate the effectiveness of the proposed method.
\end{abstract}

\section{Introduction}
\label{Introduction}
%The past decade has witnessed the great success that deep neural networks have achieved in many fields, including computer vision and automatic speech recognition. Nonetheless, training deep neural networks may take weeks due to the large size of the model and the training dataset. One effective and promising way to reduce the training time is to use distributed learning \cite{dean2012large}.
Recently, Federated Learning (FL) has become a prominent distributed learning paradigm since it allows training on a large amount of decentralized data residing on devices like mobile phones \cite{mcmahan2017communication}. However, FL imposes several critical challenges. First of all, the communication capability of the mobile devices can be a significant bottleneck. Furthermore, the training data on a given worker is typically based on its usage of the mobile devices, which results in heterogeneous data distribution. In addition, the local data usually contains some sensitive information of a particular mobile device user. Therefore, there is an pressing need to develop a privacy-preserving distributed learning algorithm. Finally, similar to many distributed learning methods, FL may suffer from malicious participants. As is shown in \cite{chen2017distributed}, even a single Byzantine worker, which may transmit arbitrary information, can severely disrupt the convergence of distributed gradient descent algorithms. However, to the best of our knowledge, no existing methods can cope with all the aforementioned challenges.
\setcounter{footnote}{0}
To alleviate the communication burden of the workers, there have been various gradient quantization methods \cite{alistarh2017qsgd,wen2017terngrad,bernstein2018signsgd1,wu2018error,agarwal2018cpsgd} in the literature, among which the recently proposed {\scriptsize SIGN}SGD with majority vote \cite{bernstein2018signsgd2} is of particular interest due to its robustness and communication efficiency.\footnote{Note that all the algorithms considered in this work use the idea of majority vote. Therefore, we ignore the term ``with majority vote" in the following discussions for the ease of presentation.} In {\scriptsize SIGN}SGD, during each communication round, only the signs of the gradients and aggregation results are exchanged between the workers and the server, which leads to around 32$\times$ less communication than full-precision distributed SGD. Nonetheless, it has been shown in \cite{chen2019distributed} that {\scriptsize SIGN}SGD fails to converge when the data on different workers are heterogeneous (i.e., drawn from different distributions), which is one of the most important features in FL.

%In particular, during each communication round, the workers only transmit the signs of their gradients and therefore use around 32$\times$ less communication than full-precision distributed SGD. After receiving the signs of gradients, the parameter server updates the model based on a majority vote rule, which reduces the impact of potentially faulty information and improves the robustness. It is worth mentioning that the idea of sign-based methods has inspired many popular optimizers, among which ADAM \cite{kingma2014adam} is one of the most commonly used ones. Nonetheless, it has been shown in \cite{chen2019distributed} that {\scriptsize SIGN}SGD fails to converge when the data on different workers are heterogeneous (i.e., drawn from different distributions), which is essentially one of the most important features in FL.

In this work, inspired by the idea of adding carefully designed noise before taking the sign operation in \cite{chen2019distributed}, we present Stochastic-Sign SGD, which is a class of stochastic-sign based SGD algorithms. In particular, we first propose a stochastic compressor $sto\text{-}sign$, which extends {\scriptsize SIGN}SGD to its stochastic version {\scriptsize Sto-SIGN}SGD. In this scheme, instead of directly transmitting the signs of gradients, the workers adopt a two-level stochastic quantization and transmit the signs of the quantized results. We note that different from the existing 1-bit stochastic quantization schemes (e.g., QSGD \cite{alistarh2017qsgd}, cpSGD \cite{agarwal2018cpsgd}), the proposed algorithm also uses the majority vote rule in gradient aggregation, which allows the server-to-worker communication to be 1-bit compressed and ensures robustness as well. Then, to further resolve the privacy concerns, a differentially private stochastic compressor $dp\text{-}sign$ is proposed, which can accommodate the requirement of $(\epsilon,\delta)$-local differential privacy \cite{dwork2014algorithmic}. The corresponding algorithm is termed as {\scriptsize DP-SIGN}SGD. We then prove that {\color{black}both of the proposed algorithms converge to the neighborhood of the (local) optimum under heterogeneous data distribution}. In addition, assuming that there are $M$ normal (benign) workers, it is shown that the Byzantine resilience of the proposed algorithms is upper bounded by $|\sum_{m=1}^{M}(\boldsymbol{g}_{m}^{(t)})_{i}|/b_{i}, \forall i$, where $(\boldsymbol{g}_{m}^{(t)})_{i}$ is the $i$-th entry of worker $m$'s gradient at iteration $t$ and $b_{i} \geq \max_{m}(\boldsymbol{g}_{m}^{(t)})_{i}$ is some design parameter. Particularly, $b_{i}$ depends on the data heterogeneity (through $\max_{m}(\boldsymbol{g}_{m}^{(t)})_{i}$). As a special case, the proposed algorithms can tolerate $M-1$ Byzantine workers when the normal workers can access the same dataset (i.e., $(\boldsymbol{g}_{m}^{(t)})_{i} = (\boldsymbol{g}_{j}^{(t)})_{i}, \forall 1\leq j, m\leq M$), which recovers the result of {\scriptsize SIGN}SGD. We also introduce weighted vote and Top-$k$ sparsification based schemes to improve resilience against attackers and differential privacy, respectively.

Finally, we extend the proposed algorithm to its error-feedback variant, termed as Error-Feedback Stochastic-Sign SGD. In this scheme, the server keeps track of the error induced by the majority vote operation and compensates for the error in the next communication round. Both the convergence and the Byzantine resilience are established. Extensive simulations are performed to demonstrate the effectiveness of all the proposed algorithms.

\section{Related Works}
\textbf{Gradient Quantization:} To accommodate the need of communication efficiency in distributed learning, various gradient compression methods have been proposed. Most of the existing works focus on unbiased methods \cite{tang2018communication,jiang2018linear}. QSGD \cite{alistarh2017qsgd}, TernGrad \cite{wen2017terngrad} and ATOMO \cite{wang2018atomo} propose to use stochastic quantization schemes, based on which a differentially private variant is proposed in \cite{agarwal2018cpsgd}. Due to the unbiased nature of such quantization methods, the convergence of the corresponding algorithms can be established.

The idea of sharing the signs of gradients in SGD can be traced back to 1-bit SGD \cite{seide20141}. Despite that sign-based quantization is biased in nature, \cite{carlson2015stochastic} and \cite{bernstein2018signsgd1,bernstein2018signsgd2} show theoretical and empirical evidence that sign-based gradient schemes can converge well in the homogeneous data distribution scenario. {\color{black}\cite{safaryan2021stochastic} shows that the convergence of {\scriptsize SIGN}SGD can be guaranteed if the probability of wrong aggregation is less than $1/2$.} In the heterogeneous data distribution case, \cite{chen2019distributed} shows that the convergence of {\scriptsize SIGN}SGD is not guaranteed and proposes to add carefully designed noise to ensure a convergence rate of $O(d^{\frac{3}{4}}/T^{\frac{1}{4}})$. However, their analysis assumes second order differentiability of the noise probability density function and cannot be applied to some commonly used noise distributions (e.g., uniform and Laplace distributions). In addition, their analysis requires that the variance of the noise goes to infinity as the number of iterations grows, which may be unrealistic in practice. {\color{black}\cite{safaryan2021stochastic} proposes Stochastic Sign Descent with Momentum (SSDM) to accommodate the data heterogeneity, and annother independent work proposes FedCOMGATE \cite{haddadpour2021federated}. Compared to SSDM and FedCOMGATE, the proposed Stochastic-Sign SGD is stateless and therefore more suitable for cross-device FL \cite{kairouz2019advances}. Moreover, the Byzantine resilience of Stochastic-Sign SGD is further quantified.}

\textbf{Error-Compensated SGD:} Instead of directly using the biased approximation of the gradients, \cite{seide20141} corrects the quantization error by adding error feedback in subsequent updates and observes almost no accuracy loss empirically. \cite{wu2018error} proposes the error-compensated quantized SGD in quadratic optimization and proves its convergence for unbiased stochastic quantization. \cite{stich2018sparsified} proves the convergence of the proposed error compensated algorithm for strongly-convex loss functions and \cite{alistarh2018convergence} proves the convergence of sparsified gradient methods with error compensation for both convex and non-convex loss functions. In addition, \cite{karimireddy2019error} proposes {\scriptsize EF\text{-}SIGN}SGD, which combines the error compensation methods and {\scriptsize SIGN}SGD; however, only the single worker scenario is considered. \cite{zheng2019communication} further extends it to the multi-worker scenario and the convergence is established. However, it is required in these two works that the compressing error cannot be larger than the magnitude of the original vector, which is not the case for some biased compressors like {\scriptsize SIGN}SGD. \cite{tang2019doublesqueeze} considers more general compressors and proves the convergence under the assumption that the compressors have bounded magnitude of error. However, to the best our knowledge, none of the existing works consider the Byzantine resilience of the error-compensated methods.

\textbf{Byzantine Tolerant SGD in Heterogenous Environment:} There have been significant research interests in developing SGD based Byzantine tolerant algorithms, most of which consider homogeneous data distribution, e.g., Krum \cite{blanchard2017machine}, ByzantineSGD \cite{alistarh2018byzantine}, and the median based algorithms \cite{yin2018byzantine}. \cite{bernstein2018signsgd2} shows that {\scriptsize SIGN}SGD can tolerate up to half ``blind" Byzantine workers who determine how to manipulate their gradients before observing the gradients.

To accommodate the need for robust FL, some Byzantine tolerant algorithms that can deal with heterogeneous data distributions have been developed. \cite{li2019rsa} proposes to incorporate a regularized term with the objective function. However, it requires strong convexity and can only converge to the neighborhood of the optimal solution. \cite{cong2019slsgd} uses trimmed mean to aggregate the shared parameters. {\color{black}\cite{data2021byzantine} adopts the RAGE algorithm in \cite{steinhardt2017resilience} for robust aggregation. Despite that these methods provide certain Byzantine resilience, none of them take the communication efficiency into consideration.}

\textbf{Our Contributions.} This paper makes three main contributions to the field of FL as follows.
\begin{enumerate}[topsep=0pt,parsep=0pt,partopsep=5pt]
    \item We derive a sufficient condition for the convergence of sign-based gradient descent methods in the presence of data heterogeneity, based on which we propose the framework of Stochastic-Sign SGD, which utilizes the stochastic-sign based gradient compressors to overcome the convergence issue of {\scriptsize SIGN}SGD given heterogeneous data distribution. In particular, two novel stochastic compressors, $sto\text{-}sign$ and $dp\text{-}sign$, are proposed, which extend {\scriptsize SIGN}SGD to {\scriptsize Sto-SIGN}SGD and {\scriptsize DP-SIGN}SGD, respectively. {\scriptsize DP-SIGN}SGD is shown to improve the privacy and the accuracy simultaneously, without sacrificing any communication efficiency. We further improve the learning performance of the proposed algorithm by incorporating the error-feedback method.
    \item {\color{black}We prove that {\scriptsize Sto-SIGN}SGD converges to the neighborhood of the (local) optimum in the heterogeneous data distribution scenario. As the number of workers increases, the gap between the converged solution and the (local) optimum decreases.}
    \item {\color{black}We theoretically quantify the Byzantine resilience of the proposed algorithm, which depends on the heterogeneity of the local datasets of the workers. To further improve the Byzantine resilience of {\scriptsize Sto-SIGN}SGD, a reputation based weighted vote mechanism is proposed and its effectiveness is validated by simulations.}
\end{enumerate}

\section{Problem Formulation}
In this paper, we consider a typical federated optimization problem with $M$ normal workers as in \cite{mcmahan2017communication}. Formally, the goal is to minimize the finite-sum objective of the form
\begin{equation}
\min_{w\in \mathbb{R}^d}F(w)~~~~ \text{where}~~~~ F(w) \overset{\mathrm{def}}{=} \frac{1}{M}\sum_{m=1}^{M}f_{m}(w).
\end{equation}
For a machine learning problem, we have a sample space $I = X \times Y$, where $X$ is a space of feature vectors and $Y$ is a label space. Given the hypothesis space $\mathcal{W} \subseteq \mathbb{R}^{d}$, we define a loss function $l: \mathcal{W}\times I \rightarrow \mathbb{R}$ which measures the loss of prediction on the data point $(x,y) \in I$ made with the hypothesis vector $w \in \mathcal{W}$. In such a case, $f_{m}(w)$ is a local function defined by the local dataset of worker $m$ and the hypothesis $w$. More specifically,
\begin{equation}
f_{m}(w)=\frac{1}{|D_{m}|}\sum_{(x_n,y_n)\in D_{m}}l(w;(x_n,y_n)),
\end{equation}
where $|D_{m}|$ is the size of worker $m$'s local dataset $D_{m}$. {\color{black}If the training data are distributed over the workers uniformly at random, then we would have $\mathbb{E}[f_{m}(w)]=F(w)$, where the expectation is over the training data distribution. This is the \textit{homogeneous data distribution assumption} typically made in distributed optimization \cite{mcmahan2017communication}. In many FL applications, however, the local datasets of the workers are heterogeneously distributed.}

We consider a parameter server paradigm. At each communication round $t$, each worker $m$ forms a batch of training samples, based on which it computes and transmits the stochastic gradient $\boldsymbol{g}_{m}^{(t)}$ as an estimate to the true gradient $\nabla f_{m}(w^{(t)}_{m})$. When the worker $m$ evaluates the gradient over its whole local dataset, we have $\boldsymbol{g}_{m}^{(t)} = \nabla f_{m}(w^{(t)}_{m})$. After receiving the gradients from the workers, the server performs aggregation and sends the aggregated gradient back to the workers. Finally, the workers update their local model weights using the aggregated gradient. In this sense, the classic stochastic gradient descent (SGD) algorithm \cite{robbins1951stochastic} performs iterations of the form
\begin{equation}
w^{(t+1)}_{m} = w^{(t)}_{m} - \frac{\eta}{M}\sum_{m=1}^{M}\boldsymbol{g}_{m}^{(t)}.
\end{equation}

In this case, since all the workers adopt the same update rule using the aggregated gradient, $w^{(t)}_{m}$'s are the same for all the workers. Therefore, in the following discussions, we omit the worker index $m$ for the ease of presentation. To accommodate the requirement of communication efficiency in FL, we adopt the popular idea of gradient quantization and assume that each worker $m$ quantizes the gradient with a stochastic 1-bit compressor $q(\cdot)$ and sends $q(\boldsymbol{g}_{m}^{(t)})$ instead of its actual local gradient $\boldsymbol{g}_{m}^{(t)}$. Combining with the idea of majority vote in \cite{bernstein2018signsgd1}, the corresponding algorithm is presented in Algorithm \ref{QuantizedSIGNSGD}.

\begin{algorithm}
\caption{Stochastic-Sign SGD with majority vote}
\label{QuantizedSIGNSGD}
\begin{algorithmic}
\STATE \textbf{Input}: learning rate $\eta$, current hypothesis vector $w^{(t)}$, $M$ workers each with an independent gradient $\boldsymbol{g}_{m}^{(t)}$, the 1-bit compressor $q(\cdot)$.
\STATE \textbf{on server:}
\STATE ~~\textbf{pull} $q(\boldsymbol{g}_{m}^{(t)})$ \textbf{from} worker $m$.
\STATE ~~\textbf{push} $\tilde{\boldsymbol{g}}^{(t)}= sign\big(\frac{1}{M}\sum_{m=1}^{M}q(\boldsymbol{g}_{m}^{(t)})\big)$ \textbf{to} all the workers.
\STATE \textbf{on each worker:}
\STATE ~~\textbf{update} $w^{(t+1)} = w^{(t)} - \eta\tilde{\boldsymbol{g}}^{(t)}$.
\end{algorithmic}
\end{algorithm}

Intuitively, the performance of Algorithm \ref{QuantizedSIGNSGD} is limited by the probability of wrong aggregation, which is given by
\begin{equation}\label{ConvergeneGuarantee}
sign\bigg(\frac{1}{M}\sum_{m=1}^{M}q(\boldsymbol{g}_{m}^{(t)})\bigg) \neq sign\bigg(\frac{1}{M}\sum_{m=1}^{M}\nabla f_{m}(w^{(t)})\bigg).
\end{equation}

In {\scriptsize SIGN}SGD, $q(\boldsymbol{g}_{m}^{(t)}) = sign(\boldsymbol{g}_{m}^{(t)})$ and (\ref{ConvergeneGuarantee}) holds when $\nabla f_{m}(w^{(t)}) \neq \nabla f_{j}(w^{(t)}), \forall m \neq j$ with a high probability, which prevents its convergence. In this work, we propose two compressors $sto\text{-}sign$ and $dp\text{-}sign$, which guarantee that (\ref{ConvergeneGuarantee}) holds with a probability that is strictly smaller than 0.5 and therefore the convergence of Algorithm \ref{QuantizedSIGNSGD} follows. Moreover, $dp\text{-}sign$ is differentially private, i.e., given the quantized gradient $q(\boldsymbol{g}_{m}^{(t)})$, the adversary cannot distinguish the local dataset of worker $m$ from its neighboring datasets that differ in only one data point with a high probability. The detailed definition of differential privacy can be found in Section 1 of the supplementary document.

In addition to the $M$ normal workers, it is assumed that there exist $B$ Byzantine attackers, and its set is denoted as $\mathcal{B}$. Instead of using $sto\text{-}sign$ and $dp\text{-}sign$, the Byzantine attackers can use an arbitrary compressor denoted by $byzantine\text{-}sign$. In this work, we consider the scenario that the Byzantine attackers have access to the average gradients of all the $M$ normal workers (i.e., $\boldsymbol{g}_{j}^{(t)} = \frac{1}{M}\sum_{m=1}^{M}\boldsymbol{g}_{m}^{(t)}$, $\forall j \in \mathcal{B}$) and follow the same procedure as the normal workers. Therefore, we assume that the Byzantine attacker $j$ shares the opposite signs of the true gradients, i.e., $byzantine\text{-}sign(\boldsymbol{g}_{j}^{(t)}) = -sign(\boldsymbol{g}_{j}^{(t)})$.

In order to facilitate the convergence analysis, the following commonly adopted assumptions are made.
\begin{Assumption}\label{A1}(Lower bound).
 For all $w$ and some constant $F^{*}$, we have objective value $F(w) \geq F^{*}$.
\end{Assumption}
\begin{Assumption}\label{A2}(Smoothness).
$\forall w_1,w_2$, we require for some non-negative constant $L$
\begin{equation}
F(w_1) \leq F(w_2) + <\nabla F(w_2), w_1-w_2> + \frac{L}{2}||w_1 - w_2||^{2}_2,
\end{equation}
where $<\cdot,\cdot>$ is the standard inner product.
\end{Assumption}
\begin{Assumption}\label{A4}(Variance bound).
For any worker $m$, the stochastic gradient oracle gives an independent unbiased estimate $g_{m}$ that has coordinate bounded variance:
\begin{equation}
\mathbb{E}[g_{m}] = \nabla f_{m}(w),\mathbb{E}[((g_{m})_{i}-\nabla f_{m}(w)_{i})^2] \leq \sigma^2_{i},
\end{equation}
for a vector of non-negative constants $\bar{\sigma} = [\sigma_{1},\cdots,\sigma_{d}]$; $(g_{m})_{i}$ and $\nabla f_{m}(w)_{i}$ are the $i$-th coordinate of the stochastic and the true gradient, respectively.
\end{Assumption}
\begin{Assumption}\label{A3}
The total number of workers is odd.
\end{Assumption}
We note that Assumptions \ref{A1}, \ref{A2} and \ref{A4} are standard for non-convex optimization and Assumption \ref{A3} is just to ensure that there is always a winner in the majority vote \cite{chen2019distributed}, which can be easily relaxed.

\textbf{Experimental Settings.} {\color{black}To facilitate empirical discussions on our proposed algorithms in the remaining sections, we first introduce our experimental settings here. We implement our proposed method with a two-layer fully connected neural network on the standard MNIST dataset and VGG-9 \cite{lee2020enabling} on the CIFAR-10 dataset. For MNIST, we use a fixed learning rate, which is tuned from the set $\{1,0.1,0.01,0.005,0.003,0.001,0.0001\}$. For CIFAR-10, we tune the initial learning rate from the set $\{1,0.1,0.01,0.001,0.0001\}$, which is reduced by a factor of 2, 5, 10 and 20 at iteration 1,500, 3,000, 5,000 and 7,000, respectively. We consider a scenario of $M=31$ normal workers. To simulate the heterogeneous data distribution scenario, each worker only stores exclusive data for one out of the ten categories, unless otherwise noted. Besides, for MNIST, the workers evaluate their gradients over the whole local datasets during each communication round, while for CIFAR-10, the workers train their local models with a mini-batch size of 32. More details about the implementation can be found in the supplementary document.}

\section{Algorithms and Convergence Analysis}\label{Algorithms}

{\color{black}
In this section, we first derive a sufficient condition for the convergence of sign-based gradient descent method in the presence of data heterogeneity. For the ease of presentation, we first consider a scalar case, which can be readily generalized to the vector case by applying the results independently on each coordinate.

\begin{theorem}\label{Lemma1}
Let $u_{1},u_{2},\cdots,u_{M}$ be $M$ known and fixed real numbers and consider binary random variables $\hat{u}_{m}$, $1\leq m \leq M$. Suppose that $\Bar{p} = \frac{1}{M}\sum_{m=1}^{M}\Pr\left(sign\left(\frac{1}{M}\sum_{m=1}^{M}u_{m}\right) \neq \hat{u}_{m}\right) < \frac{1}{2}$, we have
\begin{equation}\label{ProbabilityOfError}
\begin{split}
P\bigg(sign\bigg(\frac{1}{M}\sum_{m=1}^{M}\hat{u}_{m}\bigg)&\neq sign\bigg(\frac{1}{M}\sum_{m=1}^{M}u_{m}\bigg)\bigg) \leq \big[4\Bar{p}(1-\Bar{p})\big]^{\frac{M}{2}},
\end{split}
\end{equation}
\end{theorem}
% Here we provide some intuition about the proof. Given the majority vote rule, the aggregation result is wrong if more than half of the workers share the wrong signs. In addition, the number of workers that share -1 or 1 can be modeled as a Poisson binomial variable, denoted as $Z$. The key difficulty is that the correct sign $sign(\frac{1}{M}\sum_{m=1}^{M}u_{m})$ is unknown. However, thanks to the special structure of (\ref{stoprob}), the mean of the number of workers sharing either -1 or 1 depends on $\frac{1}{M}\sum_{m=1}^{M}u_{m}$ rather than on each individual $u_{m}$. That being said, we can always obtain the expectation of $Z$ as a function of $\frac{1}{M}\sum_{m=1}^{M}u_{m}$. As a result, we can invoke the Markov inequality and obtain (\ref{ProbabilityOfError}) after some algebra.

\begin{Remark}
Let $u_{m} = \nabla f_{m}(w^{(t)})_{i}$ be the $i$-th coordinate of worker $m$'s true local gradient and $\hat{u}_{m} = sign(\boldsymbol{g}_{m}^{(t)})_{i}$ the $i$-th coordinate of the 1-bit estimator, $\big[4\Bar{p}(1-\Bar{p})\big]^{\frac{M}{2}} < 1/2$ is a suficient condition that the probability of wrong aggregation on the $i$-th coordinate is less than $1/2$, where $\Bar{p} = \frac{1}{M}\sum_{m=1}^{M}\Pr(sign(\frac{1}{M}\sum_{m=1}^{M}\nabla f_{m}(w^{(t)})_{i}) \neq sign(\boldsymbol{g}_{m}^{(t)})_{i})$ characterizes the impact of data heterogeneity. Essentially, as long as $\Bar{p} < 1/2$, there exists some $M$ such that the probability of wrong aggregation is less than $1/2$, and the convergence of the sign-based gradient descent method can be established.

We note that our work is different from \cite{safaryan2021stochastic} in three aspects: (1) we do not assume the same probability of wrong signs $\Pr(sign(\nabla f_{m}(w^{(t)})_{i}) \neq sign(\boldsymbol{g}_{m}^{(t)})_{i})$ across the workers; (2) instead of $\Pr(sign(\nabla f_{m}(w^{(t)})_{i}) \neq sign(\boldsymbol{g}_{m}^{(t)})_{i}) < 1/2, \forall m$, we only require the average probability of wrong signs $\Bar{p} < 1/2$. (3) we propose a stochastic-sign based compressor to overcome the non-convergence issue of {\scriptsize SIGN}SGD when $\Bar{p} \geq 1/2$. We emphasize that such a result is crucial in the heterogeneous data distribution scenario since the probability of wrong signs can be very different in this case.
\end{Remark}
}

{\color{black}In the above discussion, we show that {\scriptsize SIGN}SGD works for a sufficiently large $M$ given that the average probability of wrong signs $\Bar{p} < 1/2$. In the scenarios with more severe data heterogeneity where $\Bar{p} \geq 1/2$, however, its convergence is not guaranteed. In the following,} we propose two compressors $sto\text{-}sign$ and $dp\text{-}sign$ for the Stochastic-Sign SGD framework, which can deal with the heterogeneous data distribution scenario. The basic ideas of the two compressors are given as follows.
\begin{itemize}[topsep=0pt,parsep=0pt,partopsep=5pt]
    \item $sto\text{-}sign$: instead of directly sharing the signs of the gradients, $sto\text{-}sign$ first performs a two-level stochastic quantization and then transmits the signs of the quantized results.
    \item $dp\text{-}sign$: it is a differentially private version of $sto\text{-}sign$. The probability of each coordinate of the gradients mapping to $\{-1,1\}$ is designed to accommodate the local differential privacy requirements.
\end{itemize}

% We first consider the scenario in which all the workers are benign. The Byzantine resilience of $sto\text{-}sign$ and $dp\text{-}sign$ will be discussed in Section \ref{ByzantineAnalysis}. In addition, we assume that each worker evaluates the gradients over its whole local dataset for simplicity (i.e., $\boldsymbol{g}_{m}^{(t)} = \nabla f_{m}(w^{(t)}), \forall 1 \leq m \leq M$). Particularly, in federated learning, the workers usually compute $\nabla f_{m}(w^{(t)})$ due to the small size of the local dataset. The discussion about stochastic gradients is presented in Section \ref{SGD}. The proofs of the theoretical results are provided in Section 2 of the supplementary document.

\subsection{The Stochastic Compressor $sto\text{-}sign$}\label{StoSIGNSGD}
Formally, the compressor $sto\text{-}sign$ is defined as follows.
\begin{Definition}
For any given gradient $\boldsymbol{g}_{m}^{(t)}$, the compressor $sto\text{-}sign$ outputs $sto\text{-}sign(\boldsymbol{g}_{m}^{(t)},\boldsymbol{b})$, where $\boldsymbol{b}$ is a vector of design parameters. The $i$-th entry of $sto\text{-}sign(\boldsymbol{g}_{m}^{(t)},\boldsymbol{b})$ is given by
\begin{equation}\label{stoprob}
sto\text{-}sign(\boldsymbol{g}_{m}^{(t)},\boldsymbol{b})_{i} =
\begin{cases}
\hfill 1, \hfill \text{with probability $\frac{b_{i}+(\boldsymbol{g}_{m}^{(t)})_{i}}{2b_{i}}$},\\
\hfill -1, \hfill \text{with probability $\frac{b_{i}-(\boldsymbol{g}_{m}^{(t)})_{i}}{2b_{i}}$},\\
\end{cases}
\end{equation}
where $(\boldsymbol{g}_{m}^{(t)})_{i}$ and $b_{i}\geq max_{m}|(\boldsymbol{g}_{m}^{(t)})_{i}|$ are the $i$-th entry of $\boldsymbol{g}_{m}^{(t)}$ and $\boldsymbol{b}$, respectively.
\end{Definition}

% Recall that the performance of Algorithm \ref{QuantizedSIGNSGD} largely depends on the probability of wrong aggregation (c.f. (\ref{ConvergeneGuarantee})). When $q(\boldsymbol{g}_{m}^{(t)}) = sign(\boldsymbol{g}_{m}^{(t)})$, whether (\ref{ConvergeneGuarantee}) holds or not is determined by the gradients $\boldsymbol{g}_{m}^{(t)}$'s, which are unknown. As a result, the convergence of {\scriptsize SIGN}SGD is not guaranteed. The key idea of $sto\text{-}sign$ is to introduce the stochasticity such that the probability of wrong aggregation can be theoretically bounded for an arbitrary realization of $\boldsymbol{g}_{m}^{(t)}$'s.

{\color{black}When $q(\boldsymbol{g}_{m}^{(t)}) = sign(\boldsymbol{g}_{m}^{(t)})$, the magnitude information of $\boldsymbol{g}_{m}^{(t)}$ is not utilized. As a result, $\Bar{p} < 1/2$ is not guaranteed. In the proposed compressor $sto\text{-}sign$, the magnitude information is encoded in the mapping probabilities in (\ref{stoprob}). By introducing the stochasticity, $sto\text{-}sign$ essentially makes use of the magnitude information (without incurring additional communication overhead) such that the probability of wrong aggregation can be theoretically bounded for an arbitrary realization of $\boldsymbol{g}_{m}^{(t)}$'s.

\begin{Corollary}\label{Corollary1}
Let $u_{1},u_{2},\cdots,u_{M}$ be $M$ known and fixed real numbers and consider binary random variables $\hat{u}_{m} = sto\text{-}sign(u_{m},b)$, $1\leq m \leq M$. We have $\Bar{p}_{sto} = \frac{1}{M}\sum_{m=1}^{M}\Pr\left(sign\left(\frac{1}{M}\sum_{m=1}^{M}u_{m}\right) \neq \hat{u}_{m}\right) = \frac{bM-|\sum_{m=1}^{M}u_{m}|}{2bM}$, and
\begin{equation}\label{ProbabilityOfError2}
\begin{split}
P\bigg(sign\bigg(\frac{1}{M}\sum_{m=1}^{M}\hat{u}_{m}\bigg)&\neq sign\bigg(\frac{1}{M}\sum_{m=1}^{M}u_{m}\bigg)\bigg) \leq \left(1-x^2\right)^{\frac{M}{2}},
\end{split}
\end{equation}
where $x = \frac{|\sum_{m=1}^{M}u_{m}|}{bM}$.
\end{Corollary}
}

\begin{Remark}\label{Remark1}(\textbf{selection of} $\boldsymbol{b}$)
{\color{black} According to Corollary \ref{Corollary1}, the average probability of wrong signs $\Bar{p}_{sto} < \frac{1}{2}$ is guaranteed when $|\sum_{m}^{M}u_{m}| > 0$, which therefore addresses the non-convergence issue of {\scriptsize SIGN}SGD. Some discussions on the choice of the vector $\boldsymbol{b}$ in (\ref{stoprob}) are in order. We take the $i$-th entry of $\boldsymbol{b}$ as an example.  In the FL application, the $i$-th entry of the gradient $\boldsymbol{g}_{m}^{(t)}$ corresponds to $u_{m}$ in Corollary \ref{Corollary1}. According to the definition of $sto\text{-}sign$, $b_{i} \geq max_{m}|(\boldsymbol{g}_{m}^{(t)})_{i}|$ and $0 \leq x = |\sum_{m=1}^{M}(\boldsymbol{g}_{m}^{(t)})_{i}|/(b_{i}M) \leq 1$. On the other hand, $\left(1-x^2\right)$ in (\ref{ProbabilityOfError2}) is a decreasing function of $x$ (and therefore an increasing function of $b_{i}$) when $0 \leq x \leq 1$. In this sense, once the requirement $b_{i} \geq max_{m}|(\boldsymbol{g}_{m}^{(t)})_{i}|$ is satisfied, the probability of wrong aggregation can be bounded by (\ref{ProbabilityOfError2}), which becomes the tightest when the equality holds. Therefore, to minimize the probability of wrong aggregation, the best strategy is to select $b_{i} = max_{m}|(\boldsymbol{g}_{m}^{(t)})_{i}|$. In practice, since $max_{m}|(\boldsymbol{g}_{m}^{(t)})_{i}|$ is unknown, the selection of an appropriate $\boldsymbol{b}$ is an interesting problem deserving further investigation.

In our experiments, we examine the performance of $sto\text{-}sign$ with a fixed vector $\boldsymbol{b}$. Since the true gradients change during the training process, it is possible that $b_{i} < max_{m}|(\boldsymbol{g}_{m}^{(t)})_{i}|$ for some $i$. In such cases, the probabilities defined in (\ref{stoprob}) may fall out of the range $[0,1]$. We round them to 1 if they are positive and 0 otherwise.}
\end{Remark}

% \begin{theorem}\label{bsufficientlylarge}
% Given the same $\{u_{m}\}_{m=1}^{M}$ and $\{\hat{u}_{m}\}_{m=1}^{M}$ as those in Theorem \ref{Lemma1}, for a sufficiently large $b$, we have $P\big(sign\big(\frac{1}{M}\sum_{m=1}^{M}\hat{u}_{m}\big)\neq sign\big(\frac{1}{M}\sum_{m=1}^{M}u_{m}\big)\big) < \frac{1}{2}$.
% \end{theorem}
%\begin{Lemma}\label{Lemma1_1}
%There always exists a constant $M_{0}$ such that $\big[\big(1-\frac{1}{x}\big)e^{\frac{1}{x}}\big]^{\frac{M}{2}} < \frac{1}{2}$ when $M > M_{0}$.
%\end{Lemma}

%Theorem \ref{Lemma1} and Lemma \ref{Lemma1_1} essentially show that the probability of wrong aggregation is strictly less than 0.5 when $sign$ is replaced by $sto\text{-}sign$ with suitable parameter $b$ and $M$. Given such a result at hand, the convergence of {\scriptsize Sto-SIGN}SGD can be proved.

For the convergence analysis, we first consider the scenario in which all the workers are benign. The Byzantine resilience of $sto\text{-}sign$ and $dp\text{-}sign$ will be discussed in Section \ref{ByzantineAnalysis}. In addition, we assume that each worker evaluates the gradients over its whole local dataset for simplicity (i.e., $\boldsymbol{g}_{m}^{(t)} = \nabla f_{m}(w^{(t)}), \forall 1 \leq m \leq M$). Particularly, in federated learning, the workers usually compute $\nabla f_{m}(w^{(t)})$ due to the small size of the local dataset. The discussion about stochastic gradients is presented in Section \ref{SGD}. The proofs of the theoretical results are provided in Section 2 of the supplementary document.

\begin{theorem}\label{convergerate}
\color{black}
Suppose Assumptions \ref{A1}, \ref{A2} and \ref{A3} are satisfied, and the learning rate is set as $\eta=\frac{1}{\sqrt{Td}}$. Then by running Algorithm \ref{QuantizedSIGNSGD} with $q(\boldsymbol{g}_{m}^{(t)}) = sto\text{-}sign(\nabla f_{m}(w^{(t)}),\boldsymbol{b})$ (termed as {\scriptsize Sto-SIGN}SGD) for $T$ iterations, we have
\begin{equation}\label{ConvergenceEquation}
\begin{split}
\frac{1}{T}\sum_{t=1}^{T}||\nabla F(w^{(t)})||_{1} &\leq
\frac{1}{c}\bigg[\frac{\mathbb{E}[F(w^{(0)}) - F(w^{(T+1)})]\sqrt{d}}{\sqrt{T}} + \frac{L\sqrt{d}}{2\sqrt{T}}+ \frac{2}{T}\sum_{t=1}^{T}\sum_{i=1}^{d}|\nabla F(w^{(t)})_i|\mathds{1}_{p_{i}^{(t)} > \frac{1-c}{2}}\bigg]\\
&\leq \frac{1}{c}\bigg[\frac{(F(w^{(0)}) - F^{*})\sqrt{d}}{\sqrt{T}} + \frac{L\sqrt{d}}{2\sqrt{T}} + 2\sum_{i=1}^{d}b_{i}\Delta(M)\bigg],
\end{split}
\end{equation}
where $0<c<1$ is some positive constant, $p_{i}^{(t)}$ is the probability that the aggregation on the $i$-coordinate of the gradient is wrong during the $t$-th communication round, and $\Delta(M)$ is the solution to $\big(1-x^2\big)^{\frac{M}{2}} = \frac{1-c}{2}$. The second inequality is due to the fact that $p_{i}^{(t)} > \frac{1-c}{2}$ only if $\frac{|\nabla F(w^{(t)})_i|}{b_{i}}\leq \Delta(M)$.
\end{theorem}

Given the results in Theorem \ref{Lemma1}, the proof of Theorem \ref{convergerate} follows the well known strategy of relating the norm of the gradient to the expected improvement of the global objective in a single iteration. Then accumulating the improvement over the iterations yields the convergence rate of the algorithm.

\begin{Remark}
\color{black}
Similar to {\scriptsize SIGN}SGD, the convergence rate of {\scriptsize Sto-SIGN}SGD depends on the $L_1$-norm of the gradient. A detailed discussion on this feature can be found in \cite{bernstein2018signsgd1}. Note that compared to the convergence rate of {\scriptsize SIGN}SGD, there are two differences: the positive coefficient $c<1$ and the gap term $2\sum_{i=1}^{d}b_{i}\Delta(M)$. It can be verified that $\Delta(M)$ is a decreasing function of $M$ and $\lim_{M \rightarrow \infty}\Delta(M) = 0$ for any $c<1$, which suggests that the convergence rate of {\scriptsize Sto-SIGN}SGD in the heterogeneous data distribution scenario approaches that of {\scriptsize SIGN}SGD with homogeneous data distribution as the number of workers increases.
\end{Remark}

{\color{black}
We note that the last term in (\ref{ConvergenceEquation}) captures the gap induced by the scenarios where the probability of wrong aggregation is larger than $\frac{1-c_{i}}{2}$, which vanishes as $M$ grows to infinity. One possible concern is that, given a finite $M$, this term may be unbounded for large $b_{i}$'s.} {In the following, we introduce two scenarios where this term can be eliminated given finite $M$.

\subsubsection{Scenario 1: Large enough $b_{i}$'s}
Essentially, the following theorem can be proved in this case.
\begin{theorem}\label{bsufficientlylarge}
Given Assumption \ref{A3} and the same $\{u_{m}\}_{m=1}^{M}$ and $\{\hat{u}_{m}\}_{m=1}^{M}$ as those in Theorem \ref{Lemma1}, we have  %$P\big(sign\big(\frac{1}{M}\sum_{m=1}^{M}\hat{u}_{m}\big)\neq sign\big(\frac{1}{M}\sum_{m=1}^{M}u_{m}\big)\big) < \frac{1}{2}$.
\begin{equation}\label{T3EQ}
\begin{split}
P\bigg(sign\bigg(\frac{1}{M}\sum_{m=1}^{M}\hat{u}_{m}\bigg)\neq sign\bigg(\frac{1}{M}\sum_{m=1}^{M}u_{m}\bigg)\bigg) = \frac{1}{2} - \frac{{M-1 \choose \frac{M-1}{2}}}{2^{M}b}\bigg|\sum_{m=1}^{M}u_{m}\bigg| + O\bigg(\frac{1}{b^2}\bigg).
\end{split}
\end{equation}
\end{theorem}

{\color{black}According to Theorem \ref{bsufficientlylarge}, the probability of wrong aggregation is strictly smaller than $\frac{1}{2}$ when $|\sum_{m=1}^{M}u_{m}|>0$ and $b$ is sufficiently large such that the second term in (\ref{T3EQ}) dominates. That being said, there always exists a positive constant $c$ such that the probability of wrong aggregation is no larger than $\frac{1-c}{2}$, and therefore the last term in (\ref{ConvergenceEquation}) can be eliminated. In particular, if we select $b_{i}=T^{1/4}d^{1/4}$, the following theorem can be proved.

\begin{theorem}\label{convergeratelargeb}
Suppose Assumptions \ref{A1}, \ref{A2} and \ref{A3} are satisfied, $|\nabla F(w^{(t)})_{i}| \leq Q, \forall 1\leq i \leq d, 1\leq t \leq T$, and the learning rate is set as $\eta=\frac{1}{\sqrt{Td}}$. Then by running Algorithm 1 with $q(\boldsymbol{g}_{m}^{(t)}) = sto\text{-}sign(\nabla f_{m}(w^{(t)}),\boldsymbol{b})$ and $b_{i}=T^{1/4}d^{1/4}, \forall i$ for $T$ iterations, we have
\begin{equation}
\begin{split}
&\frac{1}{T}\sum_{t=1}^{T}\sum_{i=1}^{d}|\nabla F(w^{(t)})_{i}|^2 \\
&\leq \frac{2^{M}}{2M{M-1 \choose \frac{M-1}{2}}}\bigg[\frac{(F(w^{(0)})-F^{*})d^{3/4}}{T^{1/4}} + \frac{Ld^{3/4}}{2T^{1/4}} + \frac{2}{T}\sum_{t=1}^{T}\sum_{i=1}^{d}|\nabla F(w^{(t)})_{i}|O\bigg(\frac{1}{T^{1/4}d^{1/4}}\bigg)\bigg]\\
&\leq \frac{\sqrt{2\pi}(M-1)^{\frac{3}{2}}}{2(M^{2}-3M)}\bigg[\frac{(F(w^{(0)})-F^{*})d^{3/4}}{T^{1/4}} + \frac{Ld^{3/4}}{2T^{1/4}} + \frac{2}{T}\sum_{t=1}^{T}\sum_{i=1}^{d}|\nabla F(w^{(t)})_{i}|O\bigg(\frac{1}{T^{1/4}d^{1/4}}\bigg)\bigg],
\end{split}
\end{equation}
which further captures the impact of $M$ (i.e., $\frac{\sqrt{2\pi}(M-1)^{\frac{3}{2}}}{2(M^{2}-3M)} \leq O(\frac{1}{\sqrt{M}})$) compared to \cite{chen2019distributed}.
\end{theorem}
}

% \begin{Remark}
% According to Theorem \ref{bsufficientlylarge}, the probability of wrong aggregation is strictly smaller than $\frac{1}{2}$ when $|\sum_{m=1}^{M}u_{m}|>0$ and $b$ is sufficiently large such that the second term in (\ref{T3EQ}) dominates. That being said, there always exists a positive constant $c$ such that the probability of wrong aggregation is no larger than $\frac{1-c}{2}$, and {\color{black}therefore the last term in (\ref{ConvergenceEquation}) can be eliminated}. However, it can be seen that when $b$ keeps increasing, the probability of wrong aggregation increases and approaches $\frac{1}{2}$ (and therefore the corresponding $c$ decreases), which negatively impact the convergence.
% \end{Remark}

\subsubsection{Scenario 2: Bounded gradient dissimilarity}
We note that Theorem \ref{convergerate} does not require any assumptions on the data heterogeneity across the clients. As a result, it is possible that $\frac{|\nabla F(w^{(t)})_i|}{b_{i}}\leq \frac{|\nabla F(w^{(t)})_i|}{max_{m}|\nabla f_{m}(w^{(t)})|} \leq \Delta(M)$ as $|\nabla F(w^{(t)})_i|$ decreases in the training process, which may lead to $p_{i}^{(t)} > \frac{1-c}{2}$ given the bound in (\ref{ProbabilityOfError}). With such consideration, we can lower bound $\frac{|\nabla F(w^{(t)})_i|}{max_{m}|\nabla f_{m}(w^{(t)})|}$ with the following coordinate-wise bounded gradient dissimilarity assumption.
\begin{Assumption}\label{assumption:gradient_similarity}
    (Bounded Gradient Dissimilarity)
    \begin{equation}
      |\nabla f_m(\boldsymbol{w})_{i}| \leq B|\nabla F(\boldsymbol{w})_{i}|, \forall\boldsymbol{w} \in \mathbb{R}^d, 1 \leq i \leq d, m \in [1,M].
    \end{equation}
\end{Assumption}
Let $b_{i} = \max_{m}|\nabla f_{m}(\boldsymbol{w})_{i}|$, we have
\begin{equation}
\begin{split}
P\bigg(sign\bigg(\frac{1}{M}\sum_{m=1}^{M}sto\text{-}sign(\boldsymbol{g}_{m}^{(t)},\boldsymbol{b})\bigg)_{i}\neq sign\bigg(\frac{1}{M}\sum_{m=1}^{M}\boldsymbol{g}_{m}^{(t)}\bigg)_{i}\bigg) \leq \bigg(1-\frac{1}{B^2}\bigg)^{\frac{M}{2}}.
\end{split}
\end{equation}
%\footnote{\color{black}However, we note that $\frac{|\nabla F(w^{(t)})_i|}{b_{i}} \leq \Delta(M)$ is not a sufficient condition for $p_{i}^{(t)} > \frac{1-c}{2}$ since the right-hand side of (\ref{ProbabilityOfError}) is an upper bound of $p_{i}^{(t)}$.}
In this case, for any large enough $M$ such that $\big(1-\frac{1}{B^2}\big)^{\frac{M}{2}} < \frac{1}{2}$, we can find the corresponding $c$ such that $\big(1-\frac{1}{B^2}\big)^{\frac{M}{2}} = \frac{1-c}{2}$ and the last term in (\ref{ConvergenceEquation}) is eliminated.
\begin{Remark}
We note that Assumption \ref{assumption:gradient_similarity} can be understood as a stronger coordinate-wise version of the local dissimilarity assumption in \cite{li2018federated} which assumes that $\frac{1}{M}\sum_{m=1}^{M}||\nabla f_{m}(\boldsymbol{w})||^2 \leq B^2||\nabla F(\boldsymbol{w})||^2, \forall \boldsymbol{w}$.
\end{Remark}

}
\begin{figure}
\centering
\begin{subfigure}{\includegraphics[width=0.45\textwidth]{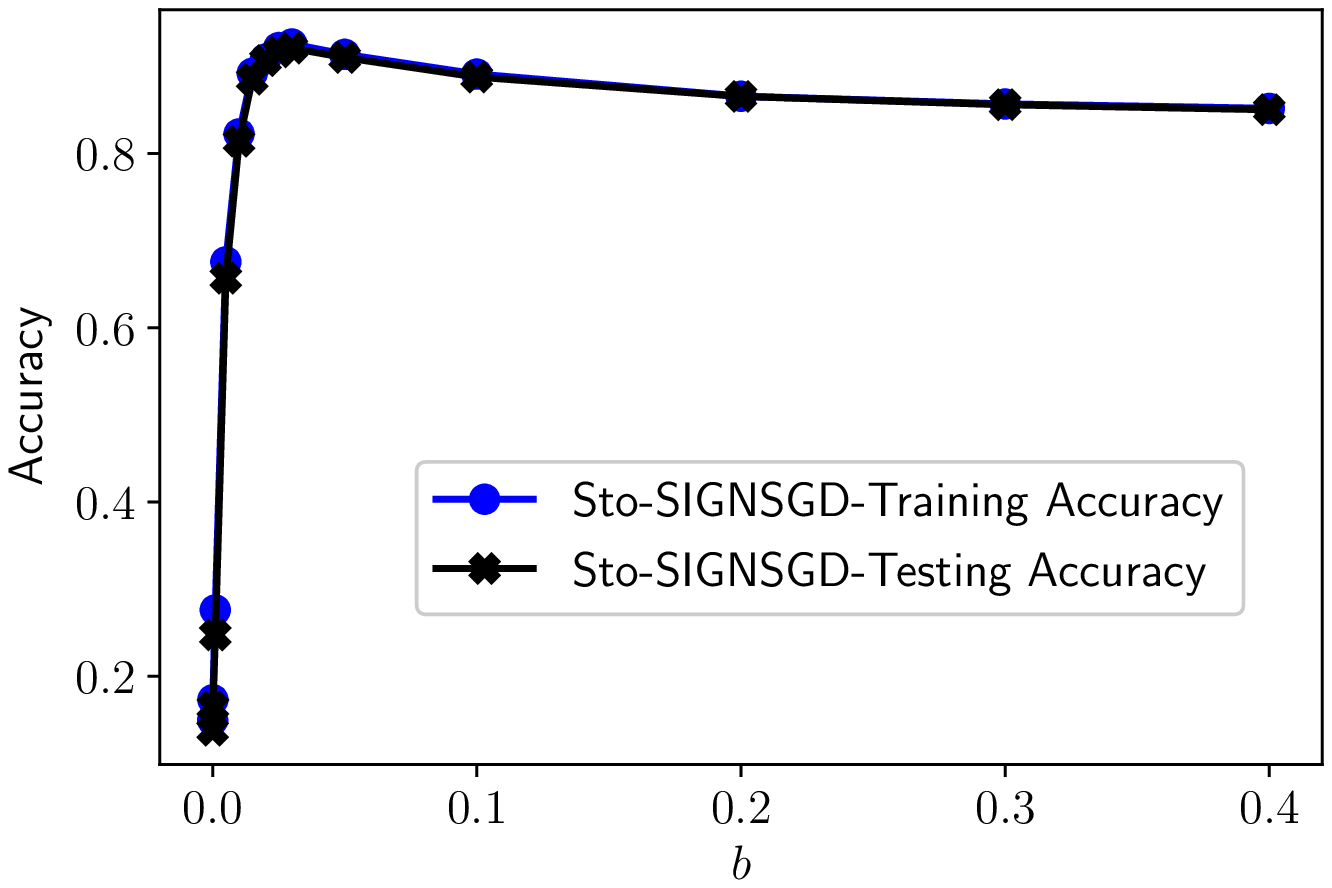}}
\end{subfigure}
\begin{subfigure}{\includegraphics[width=0.45\textwidth]{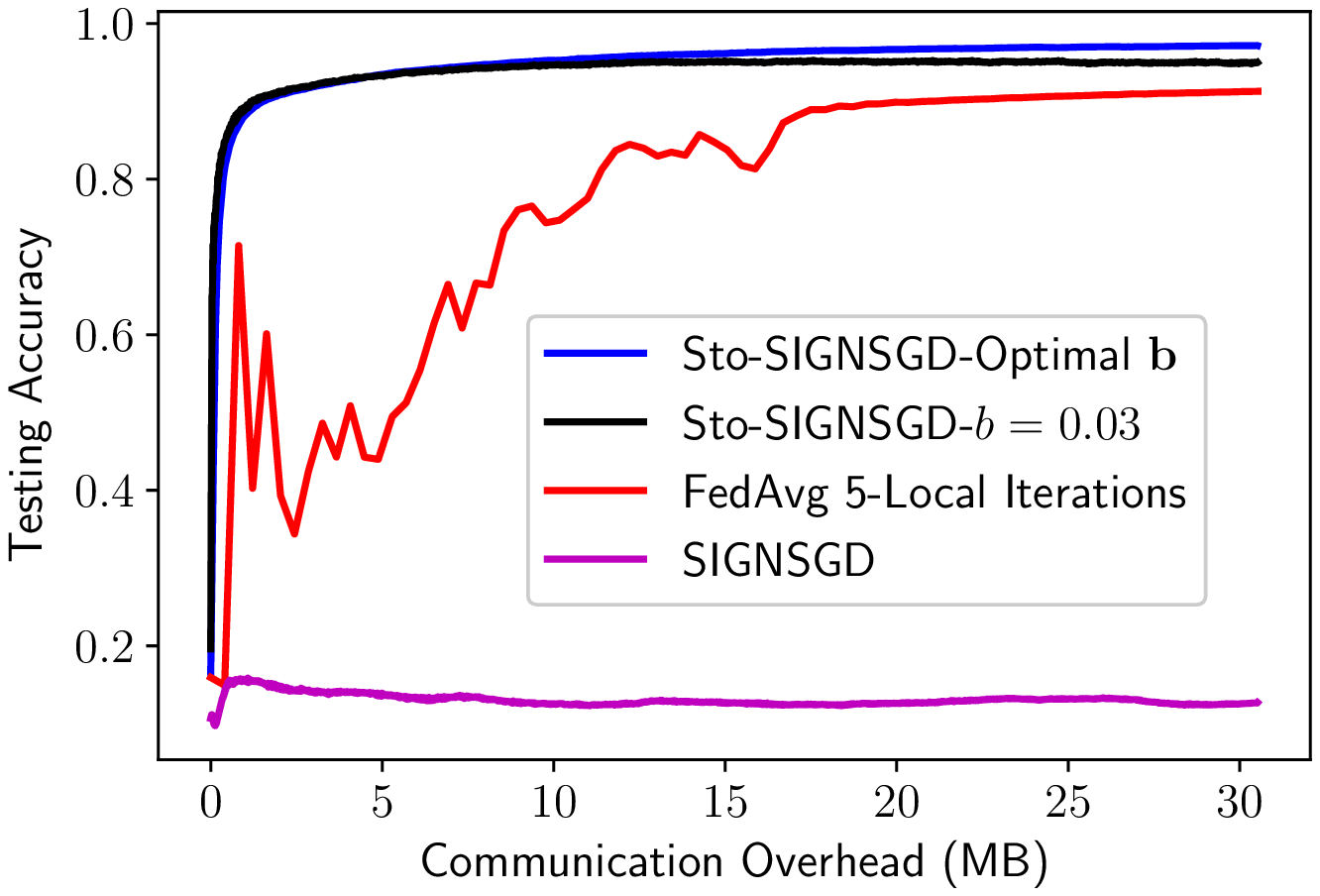}}
\end{subfigure} \\
\begin{subfigure}{\includegraphics[width=0.45\textwidth]{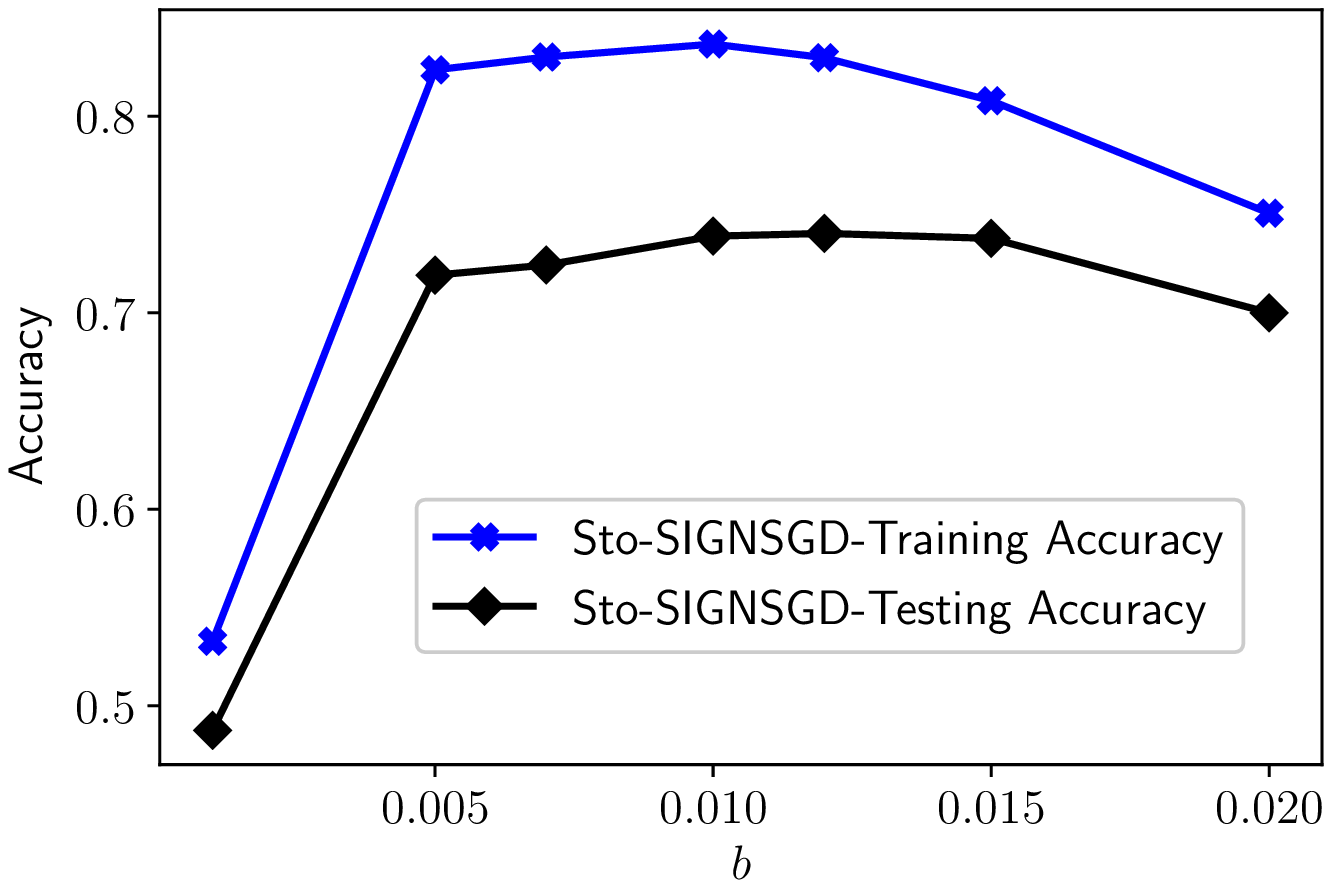}}
\end{subfigure}
\begin{subfigure}{\includegraphics[width=0.45\textwidth]{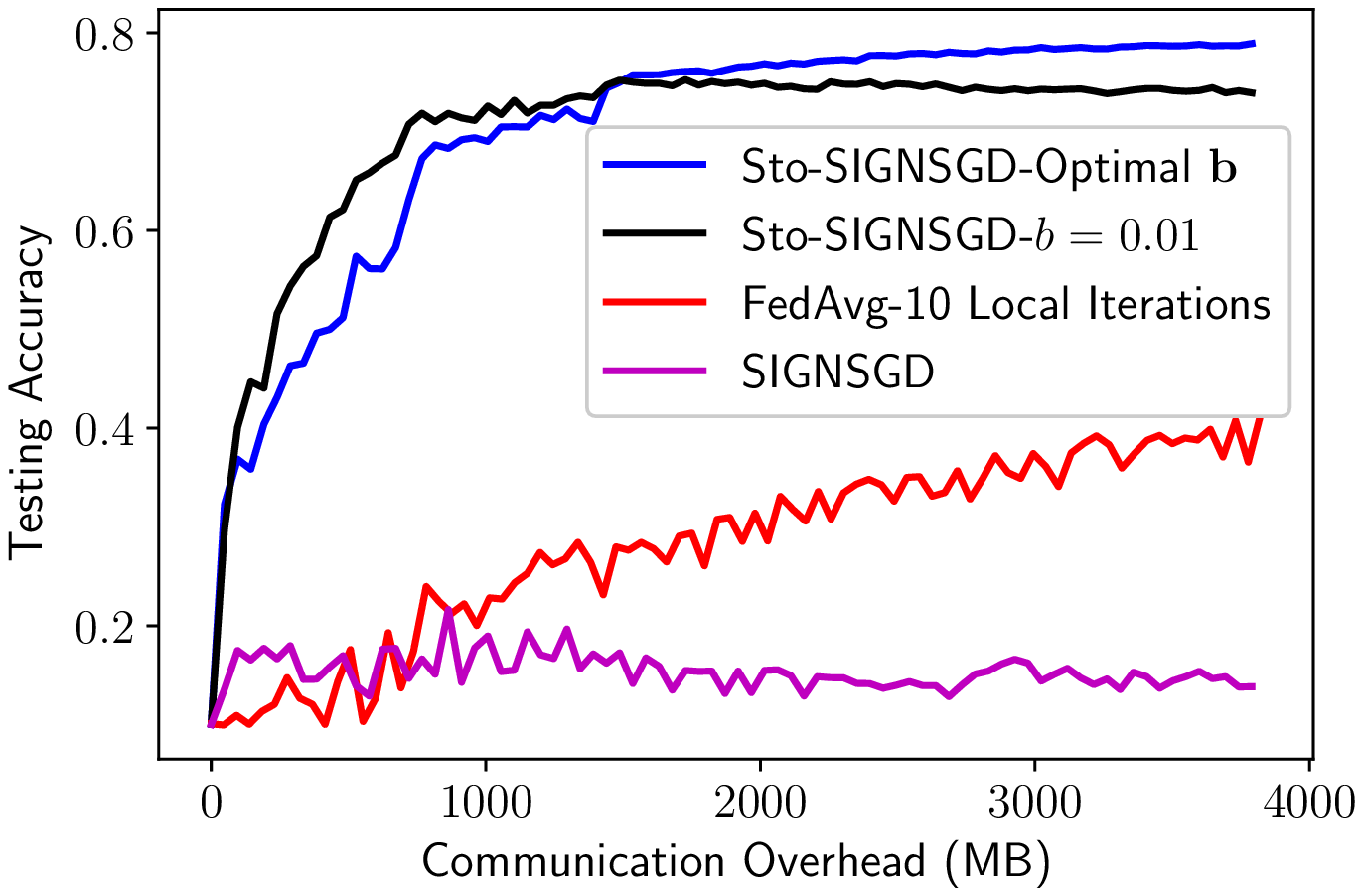}}
\end{subfigure}%
\vspace{-0.2cm}
\caption{\color{black}The two figures in the first and the second rows show the performance of {\scriptsize Sto-SIGN}SGD on MNIST and CIFAR-10, respectively. All the presented results are averaged over 5 repeats. The first column shows the training and the testing accuracy of {\scriptsize Sto-SIGN}SGD for different $\boldsymbol{b}=b\cdot\boldsymbol{1}$. We run 200 and 8,000 communication rounds for MNIST and CIFAR-10, respectively. The second column compares the testing accuracy of {\scriptsize Sto-SIGN}SGD with {\scriptsize SIGN}SGD and FedAvg \cite{mcmahan2017communication} with respect to the total communication overhead. FedAvg uses a learning rate decay of 0.99 and 0.996 per communication round for MNIST and CIFAR-10, respectively. We tune the number of local iterations from the set \{1, 5, 10, 20\} and present the results with the best final testing accuracy.}
\label{sto_impact_b}
\vspace{-0.4cm}
\end{figure}
\textbf{Experimental results.} {\color{black}We perform experiments to examine the learning performance of {\scriptsize Sto-SIGN}SGD for different selection of $\boldsymbol{b}$. Throughout our experiments, in the fixed $\boldsymbol{b}$ scenarios, we set $\boldsymbol{b} = b\cdot \boldsymbol{1}$ for some positive constant $b$. For ``Optimal $\boldsymbol{b}$", we set $b_{i} = max_{m}|(\boldsymbol{g}_{m}^{(t)})_{i}|, \forall i$. The results are shown in Figure \ref{sto_impact_b}. It can be observed that for fixed $\boldsymbol{b}$, $b$ first should be large enough to optimize the performance. Then, as $b$ keeps increasing, both the training accuracy and the testing accuracy decrease, which corroborates our analysis above. Furthermore, for a given total communication overhead, {\scriptsize Sto-SIGN}SGD with a fixed $\boldsymbol{b}$ achieves a higher testing accuracy than FedAvg (especially when the allowed communication overhead is small) and {\scriptsize SIGN}SGD and approaches that with the optimal $\boldsymbol{b}$, which demonstrates its effectiveness.}

\subsection{The Differentially Private Compressor $dp\text{-}sign$}
In this subsection, we present the differentially private version of $sto\text{-}sign$. Formally, the compressor $dp\text{-}sign$ is defined as follows.
\begin{Definition}
For any given gradient $\boldsymbol{g}_{m}^{(t)}$, the compressor $dp\text{-}sign$ outputs $dp\text{-}sign(\boldsymbol{g}_{m}^{(t)},\epsilon,\delta)$. The $i$-th entry of $dp\text{-}sign(\boldsymbol{g}_{m}^{(t)},\epsilon,\delta)$ is given by
\begin{equation}\label{dpsignsgd}
\begin{split}
&dp\text{-}sign(\boldsymbol{g}_{m}^{(t)},\epsilon,\delta)_{i} =
\begin{cases}
1, ~~~~~~~~~ \text{with probability $\Phi\big(\frac{(\boldsymbol{g}_{m}^{(t)})_{i}}{\sigma}\big)$} \\
-1,  ~~~~~~\text{with probability $1-\Phi\big(\frac{(\boldsymbol{g}_{m}^{(t)})_{i}}{\sigma}\big)$}\\
\end{cases}
\end{split}
\end{equation}
where $\Phi(\cdot)$ is the cumulative distribution function of the normalized Gaussian distribution; $\sigma = \frac{\Delta_{2}}{\epsilon}\sqrt{2\ln(\frac{1.25}{\delta})}$, where $\epsilon$ and $\delta$ are the differential privacy parameters and $\Delta_2$ is the sensitivity measure.\footnote{Please refer to Section 1 of the supplementary document for detailed information about the differential privacy parameters ($\epsilon$,$\delta$) and the sensitivity measure $\Delta_2$.}
\end{Definition}

\begin{theorem}
The proposed compressor $dp\text{-}sign(\cdot,\epsilon,\delta)$ is $(\epsilon,\delta)$-differentially private for any $\epsilon, \delta \in (0,1)$.
\end{theorem}

\begin{Remark}
Note that throughout this paper, we assume $\delta > 0$. For the $\delta = 0$ scenario, the Laplace mechanism \cite{dwork2014algorithmic} can be used by replacing the cumulative distribution function of the normalized Gaussian distribution in (\ref{dpsignsgd}) with that of the Laplace distribution. The corresponding discussion is provided in the supplementary document.
\end{Remark}

We term Algorithm \ref{QuantizedSIGNSGD} with $q(\boldsymbol{g}_{m}^{(t)}) = dp\text{-}sign(\boldsymbol{g}_{m}^{(t)},\epsilon,\delta)$ as {\scriptsize DP-SIGN}SGD. Similar to $sto\text{-}sign$, we consider the scalar case and obtain the following result for $dp\text{-}sign(\cdot,\epsilon,\delta)$.

\begin{theorem}\label{Lemma2}
Let $u_{1},u_{2},\cdots,u_{M}$ be $M$ known and fixed real numbers. Further define random variables $\hat{u}_{i}=dp\text{-}sign(u_{i},\epsilon,\delta), \forall 1\leq i \leq M$. Then there always exist a constant $\sigma_{0}$ such that when $\sigma \geq \sigma_{0}$, $P(sign(\frac{1}{M}\sum_{m=1}^{M}\hat{u}_{i})\neq sign(\frac{1}{M}\sum_{m=1}^{M}u_{i})) <\big[\big(1-x^2\big)\big]^{\frac{M}{2}}$,
where $x = \frac{|\sum_{m=1}^{M}u_{m}|}{2\sigma M}$.
\end{theorem}
%The proof of Theorem \ref{Lemma2} follows similar strategy to that of Theorem \ref{Lemma1} and Lemma \ref{Lemma1_1}. The difficulty we need to overcome is that unlike $sto\text{-}sign$, the expectation of the number of workers that share the wrong signs is not a function of $\frac{1}{M}\sum_{m=1}^{M}u_{i}$ due to the nonlinearity introduced by $\Phi(\cdot)$. However, when $\sigma$ is large enough, we show that it can be upper bounded as a function of $\frac{1}{M}\sum_{m=1}^{M}u_{i}$.

Given Theorem \ref{Lemma2}, the convergence of {\scriptsize DP-SIGN}SGD can be obtained by following a similar analysis to that of Theorem \ref{convergerate}.
%\begin{theorem}\label{convergerate2}
%Suppose Assumptions \ref{A1}-\ref{A2} are satisfied, and set the learning rate $\eta=\frac{1}{\sqrt{Td}}$. Then there exists some constant $\sigma_{0}$ such that when $\sigma \geq \sigma_{0}$, by running {\scriptsize DP-SIGN}SGD for $T$ iterations, we have
%\begin{equation}
%\begin{split}
%    \frac{1}{T}\sum_{t=1}^{T}c||\nabla F(w^{(t)})||_{1} &\leq \frac{(F(w_{0})-F^{*})\sqrt{d}}{\sqrt{T}} + \frac{L\sqrt{d}}{2\sqrt{T}}+ 2\sigma d\Delta(M),
%\end{split}
%\end{equation}
%where $c$ is some positive constant, and $\Delta(M)$ is the solution to $\big[\big(1-x\big)e^{x}\big]^{\frac{M}{2}} = \frac{1}{2}$.
%\end{theorem}

\section{Byzantine Resilience}\label{ByzantineAnalysis}
\noindent In this section, the Byzantine resilience of the sign-based gradient descent method is investigated. We note that the convergence of Algorithm \ref{QuantizedSIGNSGD} is limited by the probability of wrong aggregation (i.e., more than half of the workers share the wrong signs). In the following analysis, we assume that the Byzantine attackers evaluate their gradients over the whole training dataset, i.e., $byzantine\text{-}sign(\boldsymbol{g}_{j}^{(t)}) = -sign(\nabla F(w^{(t)}))$, $\forall j \in \mathcal{B}$, which is considered the {\textit{worst case scenario}}. The study can be easily extended to other scenarios when the Byzantine attackers are more constrained in their capabilities. Let $Z_{i}$ denote the number of normal workers that share (quantized) gradients with different signs from the true gradient $\nabla F(w^{(t)})$ on the $i$-th coordinate (i.e., $q(\boldsymbol{g}_{m}^{(t)})_{i}\neq sign(\nabla F(w^{(t)})_{i})$). Then, $Z_{i}$ is a Poisson binomial variable. In order to tolerate $k_{i}$ Byzantine workers that always share the wrong signs on the $i$-th coordinate of the gradient, we need to have $P(Z_{i} \geq \frac{M-k_{i}}{2}) \leq \frac{1-c}{2}$ for some positive constant $c$, where $M$ is the number of benign workers. Therefore, we can prove the following theorem.
\begin{theorem}\label{ByzantineResienceDP}
\color{black}
During $t$-th communication round, let $\frac{1}{M}\sum_{m=1}^{M}\Pr\left(sign\left(\nabla F(w^{(t)})\right)_{i} \neq q(\boldsymbol{g}_{m}^{(t)})_{i}\right) = \Bar{p}_{i}^{(t)}$, then Algorithm \ref{QuantizedSIGNSGD} can at least tolerate $k_{i}$ Byzantine attackers on the $i$-th coordinate of the gradient and $k_{i}$ satisfies
\begin{enumerate}[topsep=0pt,parsep=0pt,partopsep=0pt]
    \item $\Bar{p}_{i}^{(t)} \leq \frac{M-k_{i}}{2M}$.
    \item There exists some positive constant $c$ such that \begin{equation}\label{ByzantineInequa}
    \begin{split}
        &\left[\frac{(M-k_{i})(1-\Bar{p}_{i}^{(t)})}{(M+k_{i})\Bar{p}_{i}^{(t)}}\right]^{\frac{k}{2}} \left(\sqrt{\frac{M-k_{i}}{M+k_{i}}}+\sqrt{\frac{M+k_{i}}{M-k_{i}}}\right)^{M}\times\left[\Bar{p}_{i}^{(t)}(1-\Bar{p}_{i}^{(t)})\right]^{\frac{M}{2}} \leq \frac{1-c}{2}.
    \end{split}
    \end{equation}
\end{enumerate}
% \begin{equation}\label{ByzantineInequa}
% \begin{split}
% k_{i} < \frac{|\sum_{m=1}^{M}\nabla f_{m}(w^{(t)})_{i}|}{b_{i}},~~~~\bigg[\bigg(1-\frac{1}{x}\bigg)e^{\frac{1}{x}}\bigg]^{\frac{M-k_{i}}{2}} < \frac{1}{2},
% \end{split}
% \end{equation}
% where $x = \frac{(M-k_{i})b_{i}}{|\sum_{m=1}^{M}\nabla f_{m}(w^{(t)})_{i}|-b_{i}k_{i}}$.
Overall, the number of Byzantine workers that the algorithms can tolerate is given by $min_{1\leq i \leq d}k_{i}$.
\end{theorem}

\begin{Remark}
\color{black}
When $q(\boldsymbol{g}_{m}^{(t)}) = sto\text{-}sign(\nabla f_{m}(w^{(t)}),\boldsymbol{b})$, we have $\Bar{p}_{i}^{(t)} = \frac{b_{i}M-|\sum_{m=1}^{M}\nabla f_{m}(w^{(t)})_{i}|}{2b_{i}M}$ and the first condition in Theorem \ref{ByzantineResienceDP} is reduced to $k_{i} \leq \frac{|\sum_{m=1}^{M}\nabla f_{m}(w^{(t)})_{i}|}{b_{i}}$. In this sense, if we set $b_{i}=max_{m}|\nabla f_{m}(w^{(t)})_i|$ as in Section \ref{Algorithms}, the first condition in Theorem \ref{ByzantineResienceDP} gives $k_i < \frac{|\sum_{m=1}^{M}\nabla f_{m}(w^{(t)})_i|}{max_{m}|\nabla f_{m}(w^{(t)})_i|}$, which means that the Byzantine resilience depends on the heterogeneity of the local datasets. In an ideal scenario where the workers have the same local datasets,\footnote{We note that this can be relaxed for weaker attackers that do not have access to the whole training dataset.} i.e., $\nabla f_{m}(w^{(t)})_i = \nabla f_{n}(w^{(t)})_i, \forall m,n$, Theorem \ref{ByzantineResienceDP} gives $\Bar{p}_{i}^{(t)}=0$ and $k_{i} < M$. Therefore, it can tolerate $M-1$ Byzantine workers.
\end{Remark}

\begin{Remark}
Our analysis of the convergence and the Byzantine resilience is based on each individual coordinate of the gradients, which corresponds to the dimensional Byzantine resilience \cite{xie2018generalized}. %Furthermore, it also indicates that the parameter $\sigma$ in $dp\text{-}sign$ can be different across coordinates and iterations, which allows one to select suitable parameters for different coordinates and iterations to improve the privacy performance of the algorithm. A similar idea has been explored in \cite{xiang2019differentially} without considering quantization.
%Finally, since the gradients tend to decrease during the training process, it is preferred that $\sigma$ decreases over iterations. Such insights may help the allocation of differential privacy budget over iterations.
{\color{black}Furthermore, for a fixed $\boldsymbol{b}$, the gradients (and therefore the Byzantine resilience) tend to decrease during the training process. With such consideration, we propose a reputation based weighted vote mechanism to improve the Byzantine resilience, which can be found in Section \ref{ImprovedByzantine}. More specifically, the server can identify the normal workers in the beginning of the training process with higher probabilities and assign higher weights to them in majority vote.}
\end{Remark}

\begin{figure}
\centering
\begin{subfigure}
  \centering
  \includegraphics[width=0.45\textwidth]{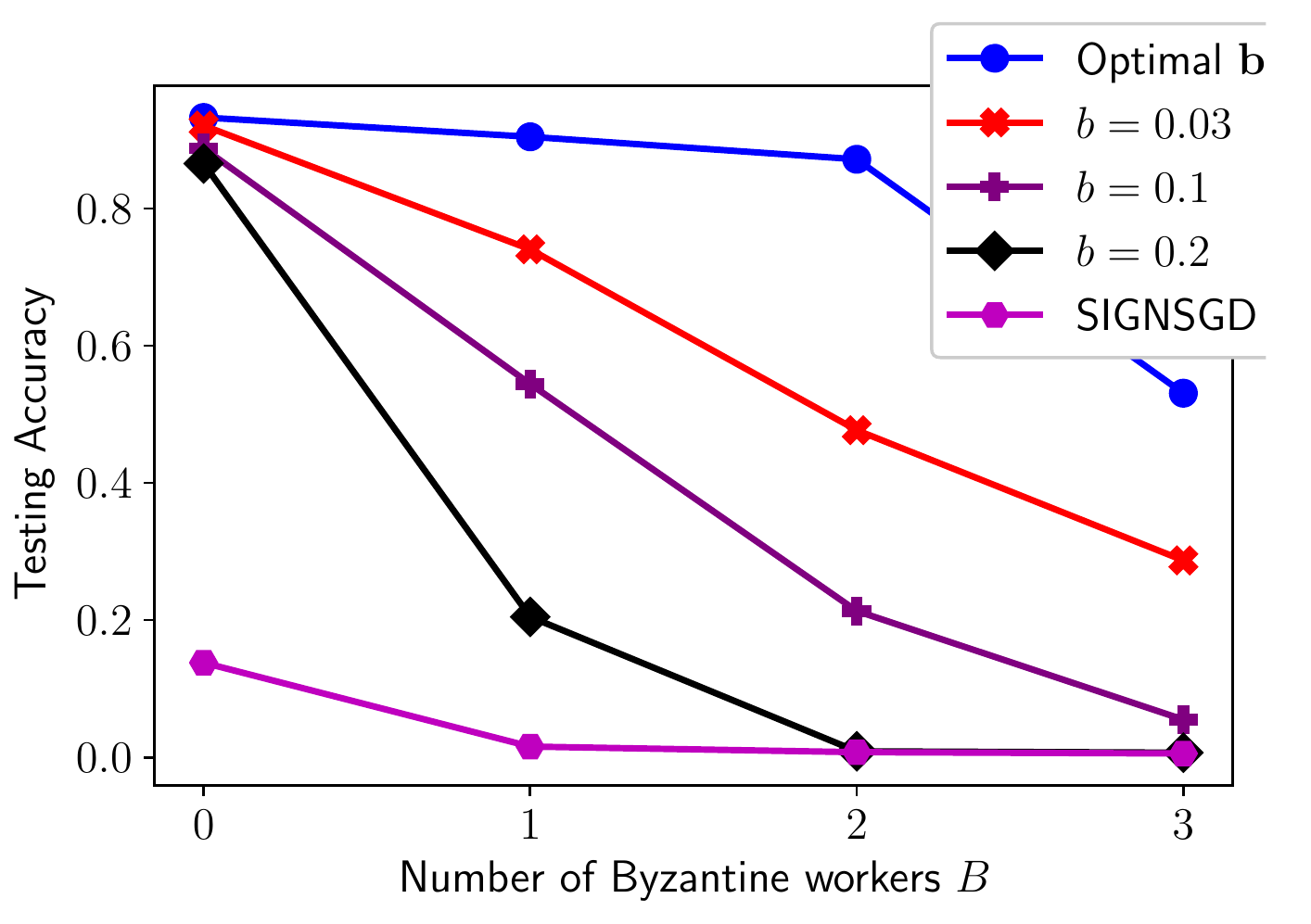}
\end{subfigure}%
\begin{subfigure}
  \centering
  \includegraphics[width=0.45\textwidth]{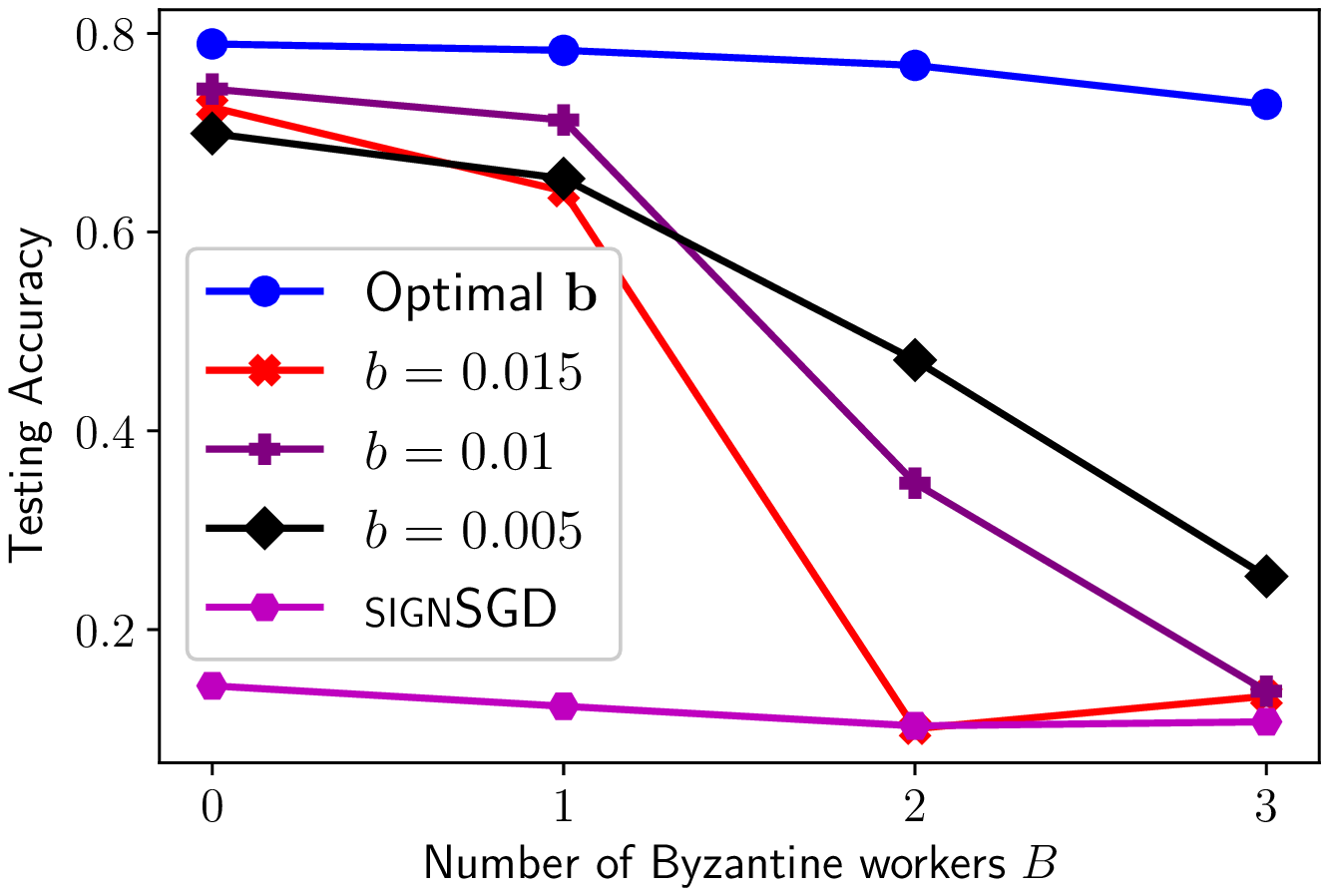}
\end{subfigure}
\caption{{\color{black}The left and right figures show the testing accuracy of {\scriptsize Sto-SIGN}SGD for different number of Byzantine workers and different $\boldsymbol{b}$ on MNIST and CIFAR-10, respectively. For MNIST, the Byzantine workers evaluate their gradients over the whole training dataset. For CIFAR-10, each Byzantine worker has 2,000 training examples that are sampled from the training dataset uniformly at random. The mini-batch sizes of all the workers and the Byzantine attackers are set to 32.}}
\label{sto_impact_Byzantine}
\end{figure}
\textbf{Experimental results.} {\color{black}Fig. \ref{sto_impact_Byzantine} shows the performance of {\scriptsize Sto-SIGN}SGD for different selection of $\boldsymbol{b}=b\cdot \boldsymbol{1}$ and different number of Byzantine workers $B$. It can be seen that as the number of Byzantine workers increases, both the training and the testing accuracy of {\scriptsize Sto-SIGN}SGD with a larger $b$ drop much faster than that with a smaller $b$, which conforms to our analysis above that a lager $b$ results in worse Byzantine resilience. It is also observed that {\scriptsize SIGN}SGD essentially fails in this extremely heterogeneous data distribution setting (where each worker holds exclusive data), even without attackers.

Furthermore, to examine the impact of data heterogeneity, we vary the number of labels of each worker's local training dataset in Table \ref{sto_sign_table}. It can be observed that the testing accuracy of {\scriptsize SIGN}SGD improves when the training data become more homogeneously distributed across workers. Furthermore, both {\scriptsize SIGN}SGD and {\scriptsize Sto-SIGN}SGD obtain better Byzantine resilience as the number of labels increases. Finally, {\scriptsize Sto-SIGN}SGD with optimal $\boldsymbol{b}$ still outperforms {\scriptsize SIGN}SGD, which indicates that introducing the stochasticity is still beneficial in the more homogeneous data distribution scenarios.}

% \begin{table}[t]
% \caption{Testing Accuracy of {\scriptsize Sto\text{-}SIGN}SGD}
% \label{sto_sign_table}
% \vskip 0.15in
% \begin{center}
% \begin{small}
% \begin{sc}
% \begin{tabular}{lcccr}
% \toprule
% $B$ & 1 Label & 2 Labels & 4 Labels\\
% \midrule
% 0 & 91.90\%& 94.57\%& 91.79\%\\
% 1 & 85.39\%& 94.21\%& 92.42\%\\
% 2 & 70.04\%& 83.95\%& 93.76\%\\
% 3 & 10.31\%& 80.61\%& 92.35\%\\
% 4 & 0.25\%& 57.52\%& 85.92\%\\
% \bottomrule
% \end{tabular}
% \end{sc}
% \end{small}
% \end{center}
% \vskip -0.2in
% \end{table}

\begin{table}[t]
\caption{Testing Accuracy of {\scriptsize Sto\text{-}SIGN}SGD on MNIST}
\label{sto_sign_table}
\vskip 0.15in
\begin{center}
\begin{small}
\begin{sc}
\begin{tabular}{lcccc}
\toprule
$B$ & \makecell{{\scriptsize SIGN}SGD 2 Labels} & \makecell{optimal $\boldsymbol{b}$ 2 Labels} & \makecell{{\scriptsize SIGN}SGD 4 Labels} &\makecell{optimal $\boldsymbol{b}$ 4 Labels}\\
\midrule
0 & 70.03\%& 92.34\% &90.53\%&93.12\%\\
1 & 66.31\%& 93.14\% &88.21\%&93.38\%\\
2 & 60.19\%& 92.71\% &87.34\%&93.39\%\\
3 & 56.23\%& 91.13\% &82.49\%&92.19\%\\
4 & 47.44\%& 84.49\% &81.51\%&92.31\%\\
\bottomrule
\end{tabular}
\end{sc}
\end{small}
\end{center}
\vskip -0.2in
\end{table}

{\color{black}
\section{Improved Resilience with Weighted Vote}\label{ImprovedByzantine}
According to Theorem \ref{ByzantineResienceDP}, it is shown that the number of Byzantine attackers that {\scriptsize Sto-SIGN}SGD can tolerate decreases as the gradient $|\sum_{m=1}^{M}\nabla f_{m}(w^{(t)})|$ decreases. As a result, {\scriptsize Sto-SIGN}SGD is more robust against Byzantine attackers in the beginning of the training process. With such consideration, a reputation based weighted vote mechanism is proposed and the corresponding algorithm is presented in Algorithm \ref{WeightedSIGNSGD}. In the proposed mechanism, the server stores a credit $r_{m}^{(t)}$ for each worker. During each iteration, given the shared signs of the gradients from the workers, the server first performs a weighted vote given by (\ref{weightedvote}). Then, the aggregated result $\tilde{\boldsymbol{g}}^{(t)}$ is used as the ``ground truth", based on which a credit is assigned to each worker. More specifically, in (\ref{creditupdate}), $\sum_{i=1}^{d}\mathds{1}_{\tilde{\boldsymbol{g}}^{(t)}_{i} = q(\boldsymbol{g}_{m}^{(t)})_{i}}/d$ and $\sum_{i=1}^{d}\mathds{1}_{\tilde{\boldsymbol{g}}^{(t)}_{i} \neq q(\boldsymbol{g}_{m}^{(t)})_{i}}/d$ measures the number of coordinates that worker $m$ shares the same and different signs compared to the aggregated result, respectively. Considering that in the beginning of the training process, the probability of correct aggregation is high, and the attackers that deliberately share wrong signs are expected to receive lower credits and therefore play a smaller role in the future iterations. As a result, the impact of the attackers is reduced.%\footnote{We leave the study on more intelligent attackers that can adjust their attacking strategies during the training process for future works.}

\begin{algorithm}[tb]
\caption{Stochastic-Sign SGD with weighted vote}
\label{WeightedSIGNSGD}
\begin{algorithmic}
\STATE \textbf{Input}: learning rate $\eta$, current hypothesis vector $w^{(t)}$, $M$ workers each with an independent gradient $\boldsymbol{g}_{m}^{(t)}$, the 1-bit compressor $q(\cdot)$, initialized credit $r_{m}^{(0)} = 1, \forall m$.
\STATE \textbf{on server:}
\STATE ~~~~\textbf{pull} $q(\boldsymbol{g}_{m}^{(t)})$ \textbf{from} worker $m$ and compute the weighted vote given by
\begin{equation}\label{weightedvote}
\tilde{\boldsymbol{g}}^{(t)}= sign\bigg(\frac{1}{M}\sum_{m=1}^{M}\max\{r_{m}^{(t)},0\}q(\boldsymbol{g}_{m}^{(t)})\bigg)
\end{equation}
\STATE ~~~~\textbf{push} $\tilde{\boldsymbol{g}}^{(t)}$ \textbf{to} all the workers and update the credits
\begin{equation}\label{creditupdate}
    r_{m}^{(t+1)} = r_{m}^{(t)}+\frac{\sum_{i=1}^{d}[\mathds{1}_{\tilde{\boldsymbol{g}}^{(t)}_{i} = q(\boldsymbol{g}_{m}^{(t)})_{i}}-\mathds{1}_{\tilde{\boldsymbol{g}}^{(t)}_{i} \neq q(\boldsymbol{g}_{m}^{(t)})_{i}}]}{d},
\end{equation}
where $\tilde{\boldsymbol{g}}^{(t)}_{i}$ and $q(\boldsymbol{g}_{m}^{(t)})_{i}$ are the $i$-th entry of the aggregated result $\tilde{\boldsymbol{g}}^{(t)}$ and the vector $q(\boldsymbol{g}_{m}^{(t)})$ shared by worker $m$, respectively.
\STATE \textbf{on each worker:}
\STATE ~~~~\textbf{update} $w^{(t+1)} = w^{(t)} - \eta\tilde{\boldsymbol{g}}^{(t)}$.
\end{algorithmic}
\end{algorithm}

\begin{figure}
\centering
\begin{subfigure}
  \centering
  \includegraphics[width=0.45\textwidth]{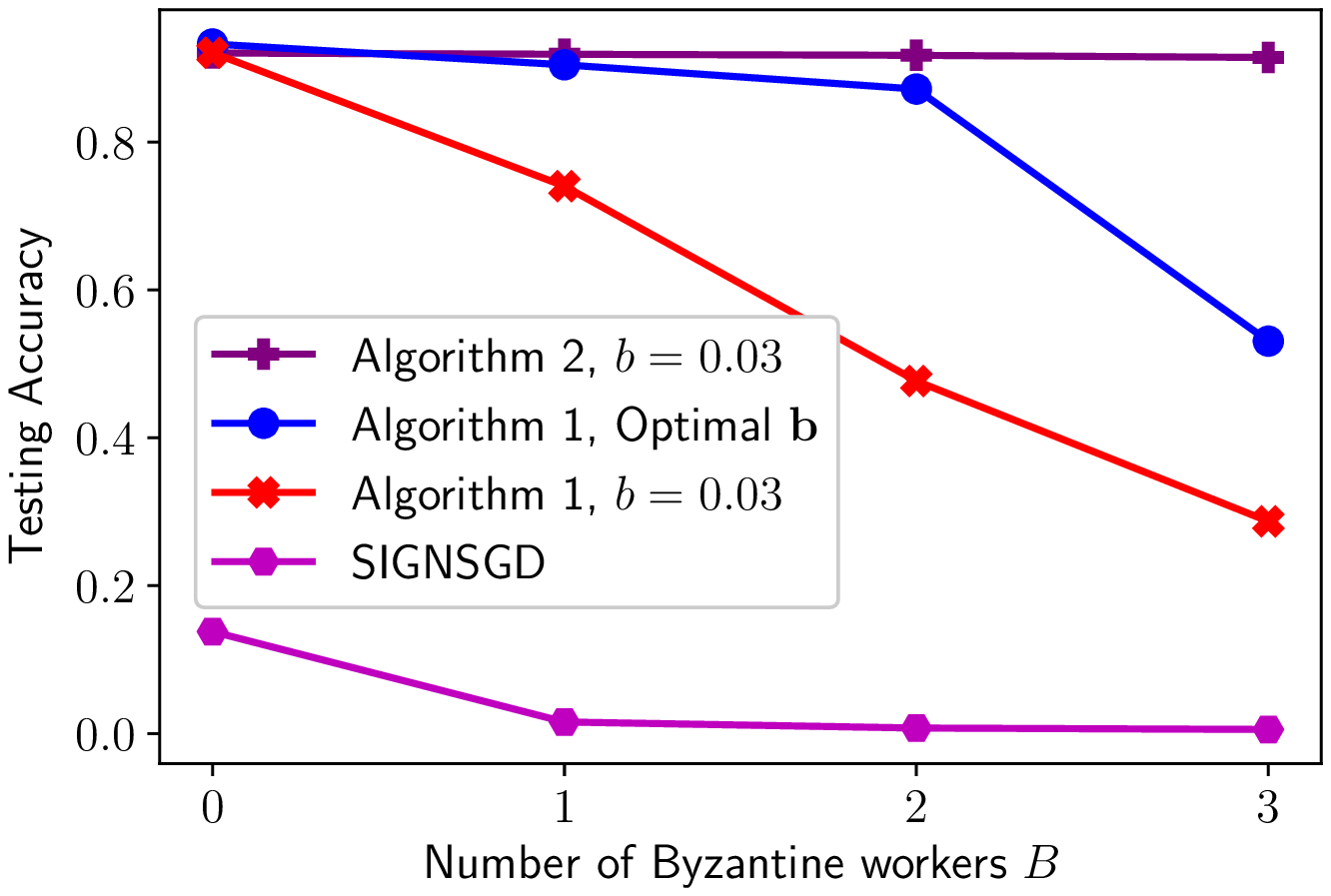}
\end{subfigure}%
\begin{subfigure}
  \centering
  \includegraphics[width=0.45\textwidth]{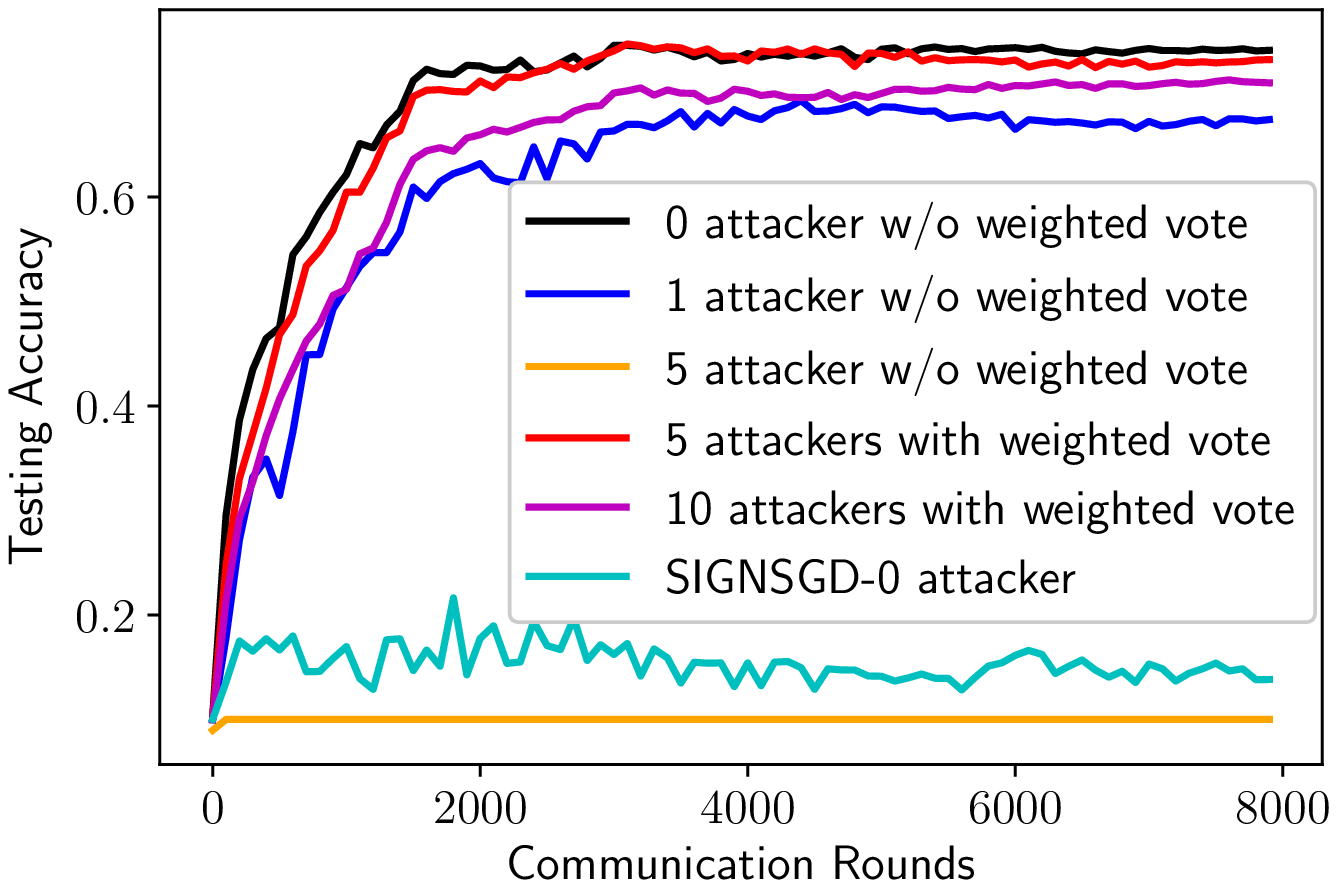}
\end{subfigure}
\caption{The left figure shows the testing accuracy of {\scriptsize Sto-SIGN}SGD with weighted vote on MNIST with $\boldsymbol{b}=0.03\cdot\boldsymbol{1}$ and different number of Byzantine workers that evaluate their gradients over the whole training dataset. The right figure shows the testing accuracy of {\scriptsize Sto-SIGN}SGD with weighted vote on CIFAR-10, with $\boldsymbol{b}=0.012\cdot\boldsymbol{1}$. Each Byzantine worker has 2,000 training examples that are sampled from the training dataset uniformly at random. The mini-batch sizes of all the workers and the Byzantine attackers are set to 32.}
\label{cifar_Byzantine}
\end{figure}

\textbf{Experimental results.} Fig. \ref{cifar_Byzantine} shows the performance of {\scriptsize Sto-SIGN}SGD with weighted vote. It can be observed that for MNIST, {\scriptsize Sto-SIGN}SGD with weighted vote obtains a comparable testing accuracy in the presence of up to 3 Byzantine workers compared to {\scriptsize Sto-SIGN}SGD without Byzantine workers. For CIFAR-10, it can be observed that {\scriptsize Sto-SIGN}SGD diverges in the presence of 5 Byzantine attackers. In the meantime, {\scriptsize Sto-SIGN}SGD with weighted vote achieves almost the same testing accuracy against 5 Byzantine workers as {\scriptsize Sto-SIGN}SGD without Byzantine workers. Furthermore, {\scriptsize Sto-SIGN}SGD with weighted vote achieves a higher testing accuracy against 10 Byzantine workers compared to {\scriptsize Sto-SIGN}SGD against 1 Byzantine worker, which validates its effectiveness.
}

{\color{black}
\section{Improved Differential Privacy with Sparsification}
Intuitively, according to Theorem \ref{Lemma2}, as the absolute value of the true gradient $|\nabla F(w^{(t)})_{i}|$ decreases, the corresponding probability of wrong aggregation increases. In this sense, discarding the coordinates with higher probabilities of wrong aggregation may help improve the learning performance. Therefore, we propose to improve the performance of {\scriptsize DP-SIGN}SGD by incorporating the Top-K sparsification scheme \cite{lin2018deep,aji2017sparse}. The corresponding algorithm is termed as {\scriptsize DP-TopSIGN}SGD and presented in Algorithm \ref{TopKDPSIGNSGD}.

\begin{algorithm}
\caption{{\scriptsize DP-SIGN}SGD with sparsification}
\label{TopKDPSIGNSGD}
\begin{algorithmic}
\STATE \textbf{Input}: learning rate $\eta$, current hypothesis vector $w^{(t)}$, $M$ workers each with an independent gradient $\boldsymbol{g}_{m}^{(t)}$, the 1-bit compressor $q(\cdot)$, the Top-$k$ sparsifier $top_{k}(\cdot)$.
\STATE \textbf{on server:}
\STATE ~~\textbf{pull} $dp\text{-}sign(top_{k}(\boldsymbol{g}_{m}^{(t)}))$ \textbf{from} worker $m$.
\STATE ~~\textbf{push} $\tilde{\boldsymbol{g}}^{(t)}= sign\big(\frac{1}{M}\sum_{m=1}^{M}dp\text{-}sign(top_{k}(\boldsymbol{g}_{m}^{(t)}))\big)$ \textbf{to} all the workers.
\STATE \textbf{on each worker:}
\STATE ~~\textbf{update} $w^{(t+1)} = w^{(t)} - \eta\tilde{\boldsymbol{g}}^{(t)}$.
\end{algorithmic}
\end{algorithm}

\textbf{Experimental results.} We compare our proposed differentially private scheme with {\scriptsize DP-Fed}SGD \cite{brendan2018learning}, a direct extension of {\scriptsize DP-}SGD \cite{abadi2016deep} to the distributed setting, where differential privacy is provided through additive Gaussian noise.
% In {\scriptsize DP-Fed}SGD, the workers first clip each gradient in $l_{2}$ norm, i.e., the gradient vector $\boldsymbol{g}$ is replaced by $\boldsymbol{g}/\max(1,||\boldsymbol{g}||_{2}/C)$ for a clipping threshold $C$, which ensures that the sensitivity measure $\Delta_{2} \leq C$. Then, the workers perturb their gradients by adding a zero-mean Gaussian noise with a variance of $C^{2}\sigma^{2}$ and send the perturbed gradients to the parameter server for aggregation. Similarly, in {\scriptsize DP-SIGN}SGD, instead of perturbing the gradients, the workers compress them with the one-bit compressor $dp\text{-}sign$.
Moreover, we further incorporate the Top-K sparsification scheme \cite{lin2018deep,aji2017sparse} into {\scriptsize DP-SIGN}SGD. In this case, each client only sends 10\% of the coordinates of the gradients with the largest magnitudes through the $dp\text{-}sign$ compressor, and the corresponding algorithm is termed as {\scriptsize DP-TopSIGN}SGD. The corresponding results are presented in Fig. \ref{ResilienceToNoise}, where we run the algorithms for 200 communication rounds and utilize the notion of Gaussian differential privacy for privacy composition \cite{dong2019gaussian}. We set $\sigma \in \{10,20,30,50,80\}$, which provide $\mu$-Gaussian differential privacy guarantees of $\mu \in \{5.66,2.83,1.89,1.13,0.71\}$; in terms of common notion of differential privacy, this corresponds to $\epsilon \in\{4.05,1.48,0.83,0.38,0.18\}$ for the commonly considered $\delta = 10^{-5}$.\footnote{Essentially, a mechanism is $\mu$-Gaussian differential privacy if and only if it is $(\epsilon,\delta(\epsilon))$-differential privacy with $\delta(\epsilon) = \Phi(-\frac{\epsilon}{\mu}+\frac{\mu}{2}) - e^{\epsilon}\Phi(-\frac{\epsilon}{\mu}-\frac{\mu}{2})$.} It can be seen that {\scriptsize DP-SIGN}SGD performs similarly with {\scriptsize DP-Fed}SGD for the same level of privacy protection, while enjoying an improvement of $32\times$ in communication efficiency. {\scriptsize DP-TopSIGN}SGD outperforms {\scriptsize DP-Fed}SGD and the improvement increases as $\sigma$ increases (which indicates more stringent requirement for privacy), while further reduces the communication overhead (compared to {\scriptsize DP-SIGN}SGD).

\begin{figure}
\centering
\includegraphics[width=0.4\textwidth]{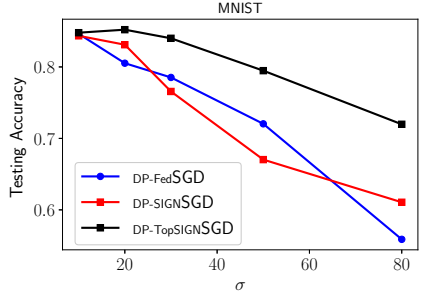}
\caption{Performance of {\scriptsize DP-TopSIGN}SGD, {\scriptsize DP-SIGN}SGD and {\scriptsize DP-Fed}SGD on MNIST. We follow the idea of gradient clipping in \cite{abadi2016deep} to bound the sensitivity $\Delta_{2}$. After computing the gradient for each individual training sample in the local dataset, each worker clips it in its $L_{2}$ norm for a clipping threshold $C$ to ensure that $\Delta_{2}\leq C$. We set $C=4$ in the experiments.}
\label{ResilienceToNoise}
\end{figure}
}

\section{Extending to SGD}\label{SGD}
\noindent Up until this point in the paper, the discussions are based on the assumption that each worker can evaluate its local true gradient $\nabla f_{m}(w^{(t)})$ for the ease of presentation. In the SGD scenario, we have to further account for the sampling noise. Particularly, the following theorem for {\scriptsize Sto-SIGN}SGD can be proved. The corresponding result for {\scriptsize DP-SIGN}SGD can be obtained following a similar strategy.

\begin{theorem}\label{Theorem7}
Suppose Assumptions \ref{A1}-\ref{A3} are satisfied, and set the learning rate $\eta=\frac{1}{\sqrt{Td}}$. Then, when $\boldsymbol{b}=b\cdot\boldsymbol{1}$ and $b$ is sufficiently large, {\scriptsize Sto-SIGN}SGD converges to the (local) optimum if either of the following two conditions is satisfied.
\begin{itemize}
  \item $P\big(sign(\frac{1}{M}\sum_{m=1}^{M}(\boldsymbol{g}_{m}^{t})_i)\neq sign(\nabla F(w^{t})_{i}\big)<0.5, \forall 1\leq i\leq d$.
  \item The mini-batch size of stochastic gradient at each iteration is at least $T$.
\end{itemize}
\end{theorem}

\begin{Remark}
Note that the first condition is not hard to satisfy. One sufficient condition is that the sampling noise of each worker is symmetric with zero mean. This assumption is also used in \cite{bernstein2018signsgd2}, which shows that the sampling noise is approximately not only symmetric, but also unimodal.
\end{Remark}

% \begin{Remark}
% We note that by replacing the compressor $sign$ in {\scriptsize SIGN}SGD with $sto\text{-}sign$ or $dp\text{-}sign$, we can obtain the improved rate (a factor of $\frac{1}{\sqrt{M}}$ in the variance term) without assuming unimodal and symmetric stochastic gradient sampling noise as in \cite{bernstein2018signsgd1}. More details can be found in Section 2 of the supplementary document.
% \end{Remark}

\begin{Remark}
We note that the above discussion assumes that $b$ is sufficiently large, which guarantees that the probability of wrong aggregation is less than 0.5. For an arbitrary $b$ that satisfies the condition in the definition of $sto\text{-}sign$, we believe that it is possible to prove that the algorithm converges to the neighborhood of the (local) optimum. In particular, similar to the proof of Theorem \ref{convergerate}, there will be an additional term $\sum_{i=1}^{d}|\nabla F(w^{t})_{i}|\mathds{1}_{|\frac{1}{M}\sum_{m=1}^{M}(\boldsymbol{g}_{m}^{t})_{i}| \leq b\Delta(M)}$. It is possible to upper bound this additional term given the fact that $\mathbb{E}[\frac{1}{M}\sum_{m=1}^{M}(\boldsymbol{g}_{m}^{t})_{i}] = \nabla F(w^{t})_{i}$, despite that more efforts are required to make the analysis rigorous.
%In addition, in Theorem \ref{SPLemma1}, it can be observed from (\ref{meanPoisson}) that $Z$ follows a Poisson binomial distribution with mean $\mu < \frac{M}{2}$. If we can utilize the Lyapunov Central Limit Theorem (CLT) \cite{billingsley2013convergence} to approximate the distribution of $Z$ by normal distribution, the probability of wrong aggregation (i.e., $P(Z \geq \frac{M}{2})$) is always less than 0.5. It is commonly assumed that CLT can be applied when $M$ is in the order of tens, which is typically satisfied in federated learning.
\end{Remark}

{\color{black}
\section{Extending to Error-feedback Variant}\label{EFVariants}
To further improved the performance of Algorithm \ref{QuantizedSIGNSGD}, we incorporate the error-feedback technique and propose its error-feedback variant (i.e., Algorithm \ref{Error-feedback noisy SignSGD DP}), where the server utilizes an $\alpha$-approximate compressor $\mathcal{C}(\cdot)$ (i.e., $||\mathcal{C}(\boldsymbol{x})-\boldsymbol{x}||_{2}^{2} \leq (1-\alpha)||\boldsymbol{x}||_{2}^{2}, \forall \boldsymbol{x}$ \cite{karimireddy2019error}) and keeps track of the corresponding compression error.

\begin{algorithm}
\caption{Error-Feedback Stochastic-Sign SGD with majority vote}
\label{Error-feedback noisy SignSGD DP}
\begin{algorithmic}
\STATE \textbf{Input}: learning rate $\eta$, current hypothesis vector $w^{(t)}$, current residual error vector $\tilde{\boldsymbol{e}}^{(t)}$, $M$ workers each with an independent gradient $\boldsymbol{g}_{m}^{(t)}=\nabla f_{m}(w^{(t)})$, the 1-bit compressor $q(\cdot)$.
\STATE \textbf{on server:}
\STATE ~~~~\textbf{pull} $q(\boldsymbol{g}_{m}^{(t)})$ \textbf{from} worker $m$.
\STATE ~~~~\textbf{push} $\tilde{\boldsymbol{g}}^{(t)} = \mathcal{C}\big(\frac{1}{M}\sum_{m=1}^{M}q(\boldsymbol{g}_{m}^{(t)})+\tilde{\boldsymbol{e}}^{(t)}\big)$ \textbf{to} all the workers,
\STATE ~~~~\textbf{update residual error:}
\begin{equation}\label{residualupdate}
\tilde{\boldsymbol{e}}^{(t+1)} = \frac{1}{M}\sum_{m=1}^{M}q(\boldsymbol{g}_{m}^{(t)}) + \tilde{\boldsymbol{e}}^{(t)} - \tilde{\boldsymbol{g}}^{(t)}.
\end{equation}
\STATE \textbf{on each worker:}
\STATE ~~~~\textbf{update} $w^{(t+1)} = w^{(t)} - \eta\tilde{\boldsymbol{g}}^{(t)}$.
\end{algorithmic}
\end{algorithm}

\begin{Remark}
Note that in Algorithm \ref{Error-feedback noisy SignSGD DP}, only the server adopts the error-feedback method. When $dp\text{-}sign$ is used, implementing error-feedback on the worker's side may increase the privacy leakage. Accounting for the additional privacy leakage caused by error-feedback is left as future work.
\end{Remark}

% \begin{Remark}
% In (\ref{residualupdate}), by adding the coefficient $\frac{1}{M}$ to $\tilde{\boldsymbol{g}}^{(t)}$, the server keeps the magnitude information about the aggregation results and enables more effective error-feedback performance. More discussion about the parameter $\frac{1}{M}$ is provided in the supplementary document.
% \end{Remark}

Both $sto\text{-}sign$ and $dp\text{-}sign$ can be used in Algorithm \ref{Error-feedback noisy SignSGD DP} and the corresponding algorithms are termed as {\scriptsize EF\text{-}Sto\text{-}SIGN}SGD and {\scriptsize EF\text{-}DP\text{-}SIGN}SGD, respectively. In the following, we show the convergence and Byzantine resilience of Algorithm \ref{Error-feedback noisy SignSGD DP} when $sto\text{-}sign$ is used. The results can be easily adapted for $dp\text{-}sign$. Particularly, the following theorems can be proved.

\begin{theorem}\label{EFDPSIGNConvergence2}
When Assumptions \ref{A1}, \ref{A2} and \ref{A3} are satisfied, by running Algorithm \ref{Error-feedback noisy SignSGD DP} with $\eta = \frac{1}{\sqrt{Td}}$, $q(\boldsymbol{g}_{m}^{(t)}) = sto\text{-}sign(\nabla f_{m}(w^{(t)}),\boldsymbol{b})$ and $\boldsymbol{b} = b\cdot\boldsymbol{1}$, we have
\begin{equation}
\begin{split}
\frac{1}{T}\sum_{t=0}^{T-1}\frac{||\nabla F(w^{(t)})||^2_{2}}{b} &\leq \frac{(F(w_{0})-F^{*})\sqrt{d}}{\sqrt{T}} + \frac{(1+L+L^2\beta)\sqrt{d}}{\sqrt{T}},
\end{split}
\end{equation}
where $\beta$ is some positive constant.
\end{theorem}

\begin{Remark}
Theorem \ref{EFDPSIGNConvergence2} shows that if $b_{i}$'s are upper bounded (i.e., when $|\nabla f_{m}(w^{(t)})_{i}| \leq Q, \forall m,t,i$ as in \cite{chen2019distributed}), {\scriptsize EF\text{-}Sto\text{-}SIGN}SGD can converge to the (local) optimum while {\scriptsize Sto-SIGN}SGD only converges to the neighborhood of optimum (c.f. Theorem \ref{convergerate}).

In our experiments, the server adopts the compressor $\mathcal{C}(\boldsymbol{x}) = \frac{1}{M}sign(\boldsymbol{x})$. In this case, the communication overhead of {\scriptsize EF\text{-}Sto\text{-}SIGN}SGD is essentially the same as {\scriptsize Sto-SIGN}SGD. Utilizing the fact that the output of the compressor $q(\cdot) \in \{-1,1\}$, it can be readily shown that $||\frac{1}{M}sign(\boldsymbol{x})-\boldsymbol{x}||_{2}^{2} < ||\boldsymbol{x}||_{2}^{2}$ with $\boldsymbol{x}=\frac{1}{M}\sum_{m=1}^{M}q(\boldsymbol{g}_{m}^{(t)})+\tilde{\boldsymbol{e}}^{(t)}$, which suggests that there exists some $\alpha$ such that $\mathcal{C}(\cdot)$ is $\alpha$-approximate compressor. More details can be found in Section \ref{dpsignextend} of the supplementary document.
\end{Remark}

Besides the fact that error-feedback is only used on the server's side, another difference between Algorithm \ref{Error-feedback noisy SignSGD DP} and those in \cite{karimireddy2019error,zheng2019communication} is that it does not require the workers to share the magnitude information about the gradients. On the one hand, it saves communication overhead. On the other hand, it keeps the resilience against the re-scaling attacks. By following a similar strategy to the proofs of Theorem \ref{EFDPSIGNConvergence2} and considering the impact of Byzantine attackers, we obtain the Byzantine resilience of Algorithm \ref{Error-feedback noisy SignSGD DP} as follows.

\begin{theorem}\label{T9}
At each iteration $t$, Algorithm \ref{Error-feedback noisy SignSGD DP} can at least tolerate $k_{i} = |\sum_{m=1}^{M}\nabla f_{m}(w^{(t)})_i|/b$ Byzantine attackers on the $i$-th coordinate of the gradient. Overall, the number of Byzantine workers that Algorithm \ref{Error-feedback noisy SignSGD DP} can tolerate is given by $min_{1\leq i \leq d}k_{i}$.
\end{theorem}

% \begin{figure*}
% \begin{minipage}[t]{0.45\linewidth}
% \centering
% \includegraphics[width=0.95\textwidth]{dp_impact_Byzantine_testing.eps}
% \end{minipage}
% \begin{minipage}[t]{0.45\linewidth}
% \centering
% \includegraphics[width=0.95\textwidth]{ComparisonEFB.eps}
% \end{minipage}
% \caption{\color{black}The first figure shows the performance of {\scriptsize DP-SIGN}SGD and {\scriptsize EF-DP-SIGN}SGD on MNIST for different $\epsilon$ when $\delta = 10^{-5}$, without Byzantine attackers. The $\epsilon$'s measure the per epoch privacy guarantee of the algorithms. The second figure compares {\scriptsize EF\text{-}DP\text{-}SIGN}SGD with {\scriptsize DP\text{-}SIGN}SGD on MNIST when $\epsilon=1$.}
% \label{dp_impact_epsilon}
% \end{figure*}

\begin{figure}
\centering
\begin{subfigure}
  \centering
  \includegraphics[width=0.32\textwidth]{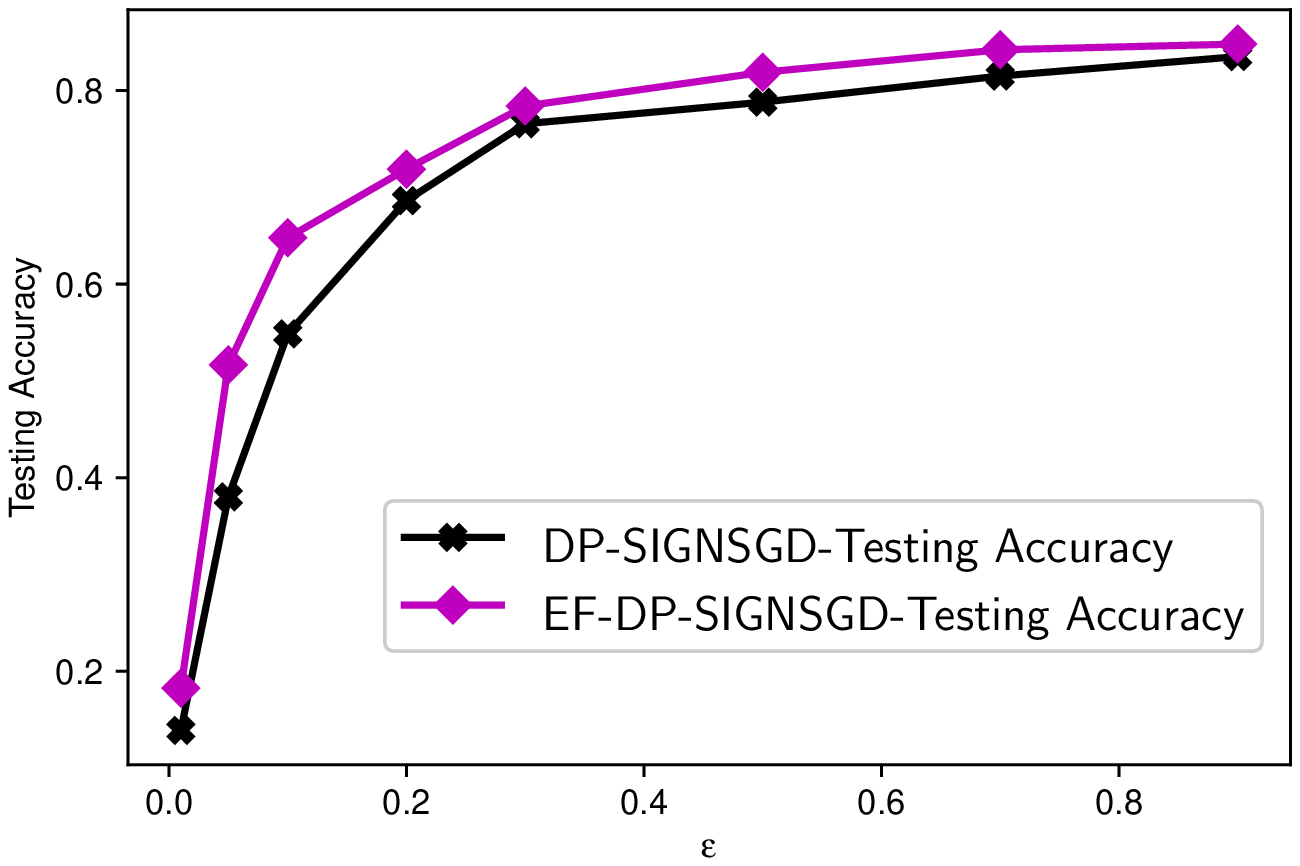}
\end{subfigure}%
\begin{subfigure}
  \centering
  \includegraphics[width=0.32\textwidth]{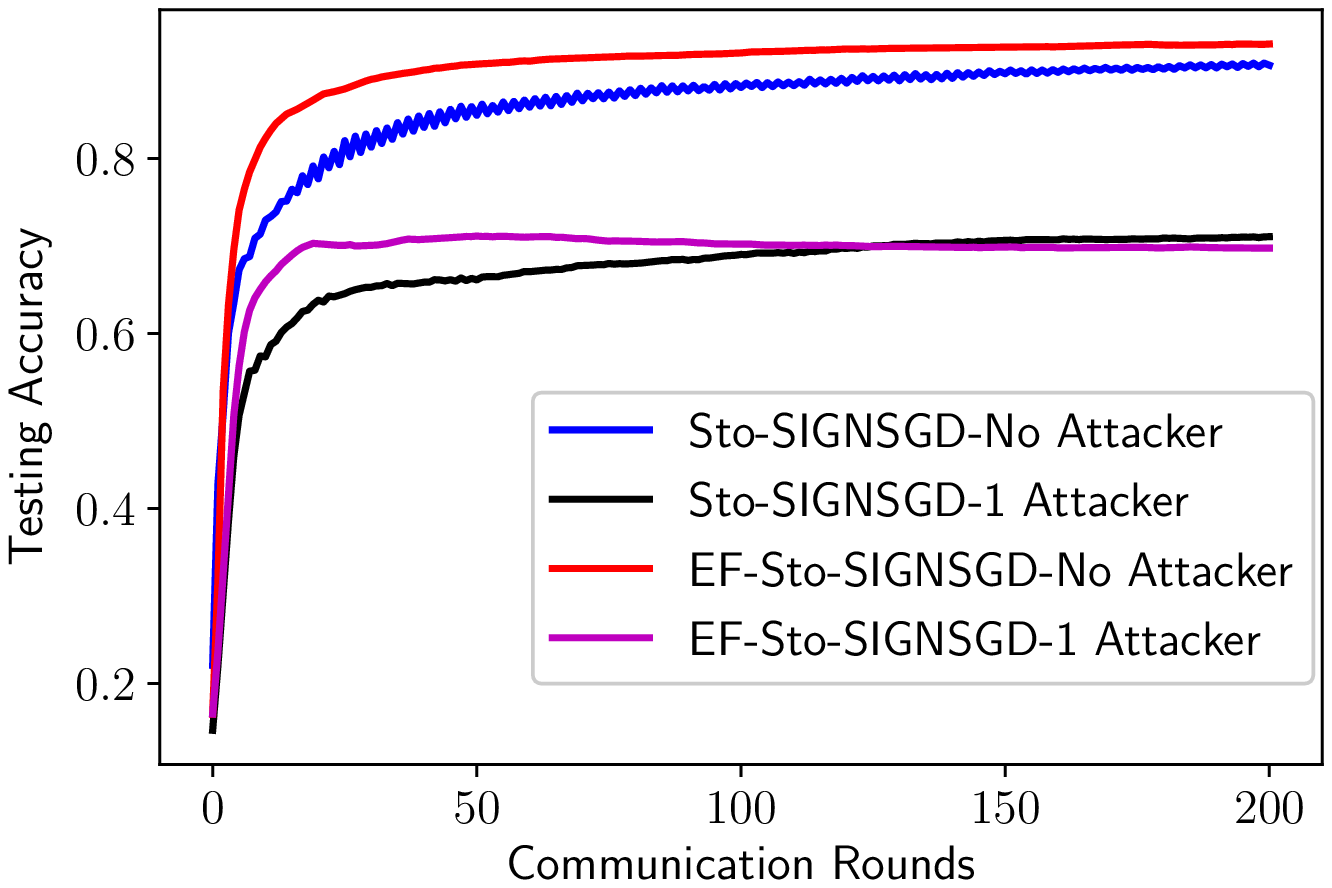}
\end{subfigure}%
\begin{subfigure}
  \centering
  \includegraphics[width=0.32\textwidth]{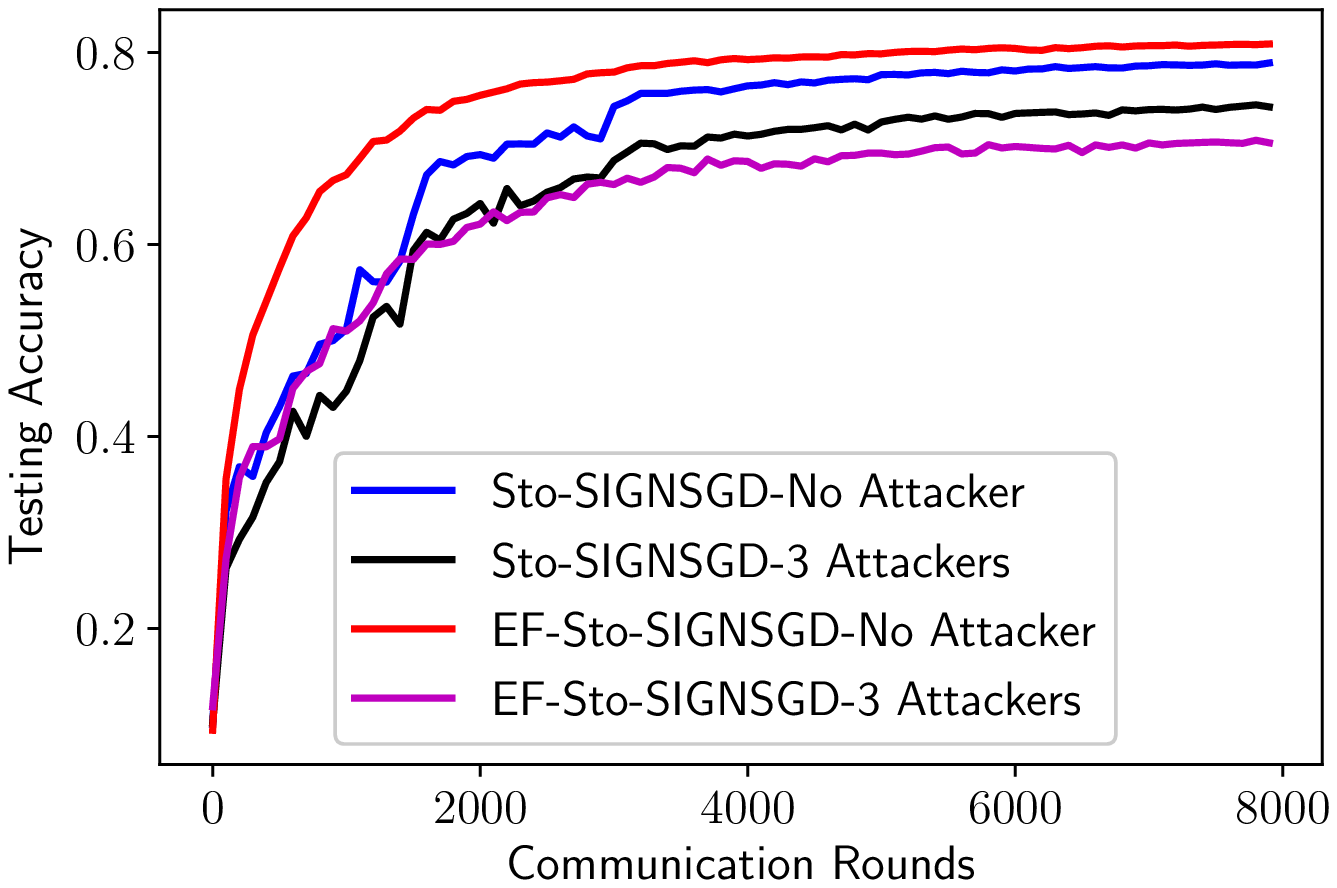}
\end{subfigure}
\caption{\color{black}The first figure shows the performance of {\scriptsize DP-SIGN}SGD and {\scriptsize EF-DP-SIGN}SGD on MNIST for different $\epsilon$ when $\delta = 10^{-5}$, without Byzantine attackers. The $\epsilon$'s measure the per epoch privacy guarantee of the algorithms. The second figure compares {\scriptsize Sto\text{-}SIGN}SGD with {\scriptsize EF\text{-}Sto\text{-}SIGN}SGD on MNIST with $\boldsymbol{b}=0.02\cdot\boldsymbol{1}$. The last figure compares {\scriptsize Sto\text{-}SIGN}SGD with {\scriptsize EF\text{-}Sto\text{-}SIGN}SGD on CIFAR-10 with optimal $\boldsymbol{b}$.}
\label{dp_impact_epsilon2}
\vspace{-0.3cm}
\end{figure}

\textbf{Experimental results.} For {\scriptsize DP\text{-}SIGN}SGD and {\scriptsize EF\text{-}DP\text{-}SIGN}SGD, we follow the idea of gradient clipping in \cite{abadi2016deep} to bound the sensitivity $\Delta_{2}$. After computing the gradient for each individual training sample in the local dataset, each worker clips it in its $L_{2}$ norm for a clipping threshold $C$ to ensure that $\Delta_{2}\leq C$. We set $C=4$ in the experiments and the results are shown in the first figure of Fig. \ref{dp_impact_epsilon2}. It can be observed that when there is no Byzantine attackers, {\scriptsize EF-DP-SIGN}SGD outperforms {\scriptsize DP\text{-}SIGN}SGD for all the examined $\epsilon$'s, which demonstrates its effectiveness. %In addition, both {\scriptsize DP\text{-}SIGN}SGD and {\scriptsize EF\text{-}DP\text{-}SIGN}SGD outperform {\scriptsize SIGN}SGD, while providing privacy guarantees.

Another observation is that the error-feedback variants do not necessarily perform better. For instance, in the last figure of Fig. \ref{dp_impact_epsilon2}, when there are 3 Byzantine attackers, the testing accuracy of {\scriptsize EF\text{-}Sto\text{-}SIGN}SGD is worse than that of {\scriptsize Sto\text{-}SIGN}SGD. In the beginning of the training process, $k_{i}$'s in Theorem \ref{T9} are large enough such that the algorithm can tolerate the Byzantine attacker. As the gradients decrease, the probability of wrong aggregation increases. In this case, the error-feedback mechanism may carry the wrong aggregations to the future iterations and have a negative impact on the learning process.
}
% Another observation is that the error-feedback variants do not necessarily perform better. For instance, in the second figure of Fig. \ref{dp_impact_epsilon}, when there is one Byzantine attacker and $\epsilon = 1$, the testing accuracy of {\scriptsize EF\text{-}DP\text{-}SIGN}SGD is worse than that of {\scriptsize DP\text{-}SIGN}SGD. In the beginning of the training process, $k_{i}$'s in Theorem \ref{T9} are large enough such that the algorithm can tolerate the Byzantine attacker. As the gradients decrease, the probability of wrong aggregation increases. In this case, the error-feedback mechanism may carry the wrong aggregations to the future iterations and have a negative impact on the learning process. {\color{black}Similar results are obtained for {\scriptsize Sto\text{-}SIGN}SGD, which are shown in Fig. \ref{dp_impact_epsilon2}.}

\section{Conclusion}
We propose a Stochastic-Sign SGD framework which utilizes two novel gradient compressors and can deal with heterogeneous data distribution. The proposed algorithms are proved to converge in the heterogeneous data distribution scenario with the same rate as {\scriptsize SIGN}SGD in the homogeneous data distribution case. In particular, the proposed differentially private compressor $dp\text{-}sign$ improves the privacy and the accuracy simultaneously without sacrificing any communication efficiency. Then, we further improve the learning performance of the proposed method by incorporating the error-feedback scheme. In addition, the Byzantine resilience of the proposed algorithms is shown analytically. It is expected that the proposed algorithms can find wide applications in the design of communication efficient, differentially private and Byzantine resilient FL algorithms.
\newpage

\begin{center}
{\Large \textbf{Supplementary Material}}
\end{center}

\setlength{\abovedisplayskip}{3pt}
\setlength{\belowdisplayskip}{3pt}
The supplementary material is organized as follows. In Section \ref{DPDefinition}, we formally provide the definition of local differential privacy \cite{dwork2014algorithmic}. In Section \ref{Proofs}, we provide the proofs of the theoretical results presented in the main document. Discussions about the extended differentially private compressor $dp\text{-}sign$ when $\delta=0$ are provided in Section \ref{DPLAPLACE}. Discussions about the server's compressor $\mathcal{C}(\cdot)$ in Algorithm \ref{Error-feedback noisy SignSGD DP} are provided in Section \ref{dpsignextend}. The details about the implementation of our experiments and some additional experimental results are presented in Section \ref{DetailsImplementation}.
\setcounter{section}{0}
\setcounter{theorem}{0}
\setcounter{Lemma}{0}
\setcounter{Corollary}{0}
\section{Definition of Local Differential Privacy}\label{DPDefinition}
In this work, we study the privacy guarantee of the proposed algorithms from the lens of local differential privacy \cite{dwork2014algorithmic}, which provides a strong notion of individual privacy in data analysis. The definition of local differential privacy is formally given as follows.
\begin{Definition}
Given a set of local datasets $\mathcal{D}$ provided with a notion of neighboring local datasets $\mathcal{N}_{\mathcal{D}}\subset \mathcal{D}\times \mathcal{D}$ that differ in only one data point. For a query function $f: \mathcal{D}\rightarrow \mathcal{X}$, a mechanism $\mathcal{M}:\mathcal{X}\rightarrow \mathcal{O}$ to release the answer of the query is defined to be $(\epsilon,\delta)$-locally differentially private if for any measurable subset $\mathcal{S}\subseteq\mathcal{O}$ and two neighboring local datasets $(D_1,D_2)\in \mathcal{N}_{\mathcal{D}}$,
\begin{equation}
P(\mathcal{M}(f(D_1))\in \mathcal{S}) \leq e^{\epsilon}P(\mathcal{M}(f(D_2))\in \mathcal{S}) + \delta.
\end{equation}
\end{Definition}

A key quantity in characterizing local differential privacy for many mechanisms is the sensitivity of the query $f$ in a given norm $l_{r}$, which is defined as
\begin{equation}\label{sensitivity}
\Delta_{r} = \max_{(D_1,D_2)\in\mathcal{N}_{\mathcal{D}}}||f(D_1)-f(D_2)||_r.
\end{equation}

For more details about the concept of differential privacy, the reader is referred to \cite{dwork2014algorithmic} for a survey.

\section{Proofs}\label{Proofs}
\subsection{Proof of Theorem 1}
\begin{theorem}\label{SPLemma1}
\color{black}
Let $u_{1},u_{2},\cdots,u_{M}$ be $M$ known and fixed real numbers and consider binary random variables $\hat{u}_{m}$, $1\leq m \leq M$. Suppose that $\Bar{p} = \frac{1}{M}\sum_{m=1}^{M}\Pr\left(sign\left(\frac{1}{M}\sum_{m=1}^{M}u_{m}\right) \neq \hat{u}_{m}\right) < \frac{1}{2}$, we have
\begin{equation}\label{SPProbabilityOfError}
\begin{split}
P\bigg(sign\bigg(\frac{1}{M}\sum_{m=1}^{M}\hat{u}_{m}\bigg)&\neq sign\bigg(\frac{1}{M}\sum_{m=1}^{M}u_{m}\bigg)\bigg) \leq \big[4\Bar{p}(1-\Bar{p})\big]^{\frac{M}{2}},
\end{split}
\end{equation}
\end{theorem}

% We first provide some intuition about the proof. Given the majority vote rule, the aggregation result is wrong if more than half of the workers share the wrong signs. In addition, based on the definition of $sto\text{-}sign$, we can obtain the probability of each worker sharing 1 or -1. Therefore, the number of workers that share the wrong signs can be modeled as a Poisson binomial variable, denoted as $Z$. The key difficulty is that the correct sign $sign(\frac{1}{M}\sum_{m=1}^{M}u_{m})$ is unknown. However, it can be shown that the mean of the number of workers sharing either -1 or 1 depends on $\frac{1}{M}\sum_{m=1}^{M}u_{m}$ rather than on each individual $u_{m}$. That being said, we can always obtain the expectation of $Z$ as a function of $\frac{1}{M}\sum_{m=1}^{M}u_{m}$. As a result, we can invoke the Markov inequality and obtain (\ref{SPProbabilityOfError}) after some algebra.

\begin{proof}
\color{black}
Define a series of random variables $\{X_m\}_{m=1}^{M}$ given by
\begin{equation}
X_{m} =
\begin{cases}
\hfill 1, \hfill &\text{if $\hat{u}_{m} \neq sign\bigg(\frac{1}{M}\sum_{m=1}^{M}u_{m}\bigg)$},\\
\hfill -1, \hfill &\text{if $\hat{u}_{m} = sign\bigg(\frac{1}{M}\sum_{m=1}^{M}u_{m}\bigg)$.}
\end{cases}
\end{equation}

In particular, $X_{m}$ can be considered as the outcome of one Bernoulli trial with successful probability $P(X_{m} = 1)$, and we have
\begin{equation}\label{eq1}
P\bigg(sign\bigg(\frac{1}{M}\sum_{m=1}^{M}\hat{u}_{m}\bigg)\neq sign\bigg(\frac{1}{M}\sum_{m=1}^{M}u_{m}\bigg)\bigg) = P\left(\sum_{m=1}^{M}X_{m} \geq 0\right).
\end{equation}

For any variable $a>0$, we have
\begin{equation}
\begin{split}
P\left(\sum_{m=1}^{M}X_m \geq 0\right) = P\left(e^{a\sum_{m=1}^{M}X_m} \geq e^{0}\right) &\leq \frac{\mathbb{E}[e^{a\sum_{m=1}^{M}X_m}]}{e^{0}} = \mathbb{E}[e^{a\sum_{m=1}^{M}X_m}],
\end{split}
\end{equation}
which is due to Markov's inequality, given the fact that $e^{a\sum_{m=1}^{M}X_m}$ is non-negative. For the ease of presentation, let $P(X_{m} = 1) = p_{m}$, we have,
\begin{equation}
\begin{split}
 \mathbb{E}[e^{a\sum_{m=1}^{M}X_m}] &= e^{\ln(\mathbb{E}[e^{a\sum_{m=1}^{M}X_m}])}= e^{\ln(\prod_{m=1}^{M}\mathbb{E}[e^{aX_m}])} = e^{\sum_{m=1}^{M}\ln(\mathbb{E}[e^{aX_m}])} \\
 &=e^{\sum_{m=1}^{M}\ln(e^{a}p_{m}+e^{-a}(1-p_{m}))} \\
 &=e^{M\left(\frac{1}{M}\sum_{m=1}^{M}\ln(e^{a}p_{m}+e^{-a}(1-p_{m}))\right)}\\
 &\leq e^{M\ln(e^{a}\Bar{p}+e^{-a}(1-\Bar{p}))},
\end{split}
\end{equation}
where $\Bar{p} = \frac{1}{M}\sum_{m=1}^{M}p_{m}$ and the inequality is due to Jensen's inequality.
Optimizing $a$ yields $a=\ln\left(\sqrt{\frac{1-\Bar{p}}{\Bar{p}}}\right) > 0$ and
\begin{equation}
e^{M\ln(e^{a}\Bar{p}+e^{-a}(1-\Bar{p}))} = [4\Bar{p}(1-\Bar{p})]^{\frac{M}{2}},
\end{equation}
which completes the proof.
\end{proof}

\subsection{Proof of Corollary 1}
\begin{Corollary}\label{SPCorollary1}
\color{black}
Let $u_{1},u_{2},\cdots,u_{M}$ be $M$ known and fixed real numbers and consider binary random variables $\hat{u}_{m} = sto\text{-}sign(u_{m},b)$, $1\leq m \leq M$. We have $\Bar{p} = \frac{1}{M}\sum_{m=1}^{M}\Pr\left(sign\left(\frac{1}{M}\sum_{m=1}^{M}u_{m}\right) \neq \hat{u}_{m}\right) = \frac{bM-|\sum_{m=1}^{M}u_{m}|}{2bM}$, and
\begin{equation}\label{SPProbabilityOfError2}
\begin{split}
P\bigg(sign\bigg(\frac{1}{M}\sum_{m=1}^{M}\hat{u}_{m}\bigg)&\neq sign\bigg(\frac{1}{M}\sum_{m=1}^{M}u_{m}\bigg)\bigg) \leq \left(1-x^2\right)^{\frac{M}{2}},
\end{split}
\end{equation}
where $x = \frac{|\sum_{m=1}^{M}u_{m}|}{bM}$.
\end{Corollary}
\begin{proof}
\color{black}
It can be easily shown that $\Bar{p} = \frac{1}{M}\sum_{m=1}^{M}\Pr\left(sign\left(\frac{1}{M}\sum_{m=1}^{M}u_{m}\right) \neq \hat{u}_{m}\right) = \frac{bM-|\sum_{m=1}^{M}u_{m}|}{2bM}$ when $\hat{u}_{m} = sto\text{-}sign(u_{m},b)$. Plugging it into (\ref{SPProbabilityOfError}) completes the proof.
\end{proof}

\subsection{Proof of Theorem 2}
\begin{theorem}\label{SPconvergerate}
Suppose Assumptions \ref{A1}, \ref{A2} and \ref{A3} are satisfied, and the learning rate is set as $\eta=\frac{1}{\sqrt{Td}}$. Then by running {\scriptsize Sto-SIGN}SGD for $T$ iterations, we have
\begin{equation}
\color{black}
\begin{split}
&\frac{1}{T}\sum_{t=1}^{T}||\nabla F(w^{(t)})||_{1} \\
&\leq
\frac{1}{c}\bigg[\frac{\mathbb{E}[F(w^{(0)}) - F(w^{(T+1)})]\sqrt{d}}{\sqrt{T}} + \frac{L\sqrt{d}}{2\sqrt{T}} + \frac{2\eta}{T}\sum_{t=1}^{T}\sum_{i=1}^{d}|\nabla F(w^{(t)})_i|\mathds{1}_{p_{i}^{(t)} > \frac{1-c}{2}}\bigg]\\
&\leq \frac{1}{c}\bigg[\frac{(F(w^{(0)}) - F^{*})\sqrt{d}}{\sqrt{T}} + \frac{L\sqrt{d}}{2\sqrt{T}} + 2\sum_{i=1}^{d}b_{i}\Delta(M)\bigg],
\end{split}
\end{equation}
where $0<c<1$ is some positive constant, $p_{i}^{(t)}$ is the probability that the aggregation on the $i$-coordinate of the gradient is wrong during the $t$-th communication round, and $\Delta(M)$ is the solution to $\big(1-x^2\big)^{\frac{M}{2}} = \frac{1-c}{2}$. The second inequality is due to the fact that $p_{i}^{(t)} > \frac{1-c}{2}$ only if $\frac{|\nabla F(w^{(t)})_i|}{b_{i}}\leq \Delta(M)$.
\end{theorem}

The proof of Theorem \ref{convergerate} follows the well known strategy of relating the norm of the gradient to the expected improvement of the global objective in a single iteration. Then accumulating the improvement over the iterations yields the convergence rate of the algorithm.

\begin{proof}
According to Assumption 2, we have
\begin{equation}
\begin{split}
&F(w^{(t+1)}) - F(w^{(t)}) \\
&\leq <\nabla F(w^{(t)}), w^{(t+1)}-w^{(t)}> + \frac{L}{2}||w^{(t+1)}-w^{(t)}||^2 \\
& =-\eta <\nabla F(w^{(t)}), sign\bigg(\frac{1}{M}\sum_{m=1}^{M}sto\text{-}sign(\boldsymbol{g}_{m}^{(t)})\bigg)> + \frac{L}{2}\bigg|\bigg|\eta sign\bigg(\frac{1}{M}\sum_{m=1}^{M}sto\text{-}sign(\boldsymbol{g}_{m}^{(t)})\bigg)\bigg|\bigg|^2 \\
& = -\eta <\nabla F(w^{(t)}), sign\bigg(\frac{1}{M}\sum_{m=1}^{M}sto\text{-}sign(\boldsymbol{g}_{m}^{(t)})\bigg)> + \frac{L\eta^2d}{2} \\
& = -\eta ||\nabla F(w^{(t)})||_{1} + \frac{L\eta^2d}{2} + 2\eta\sum_{i=1}^{d}|\nabla F(w^{(t)})_{i}|\mathds{1}_{sign(\frac{1}{M}\sum_{m=1}^{M}sto\text{-}sign(\boldsymbol{g}_{m}^{(t)})_{i})\neq sign(\nabla F(w^{(t)})_{i})},
\end{split}
\end{equation}
where $\nabla F(w^{(t)})_{i}$ is the $i$-th entry of the vector $\nabla F(w^{(t)})$ and $\eta$ is the learning rate. Taking expectation on both sides yeilds

\begin{equation}\label{convergencee1}
\color{black}
\begin{split}
&\mathbb{E}[F(w^{(t+1)}) - F(w^{(t)})] \leq -\eta ||\nabla F(w^{(t)})||_{1} + \frac{L\eta^2d}{2} \\
&+2\eta\sum_{i=1}^{d}|\nabla F(w^{(t)})_{i}|P\bigg(sign\bigg(\frac{1}{M}\sum_{m=1}^{M}sto\text{-}sign(\boldsymbol{g}_{m}^{(t)})_{i}\bigg)\neq sign(\nabla F(w^{(t)})_{i})\bigg)\\
% &\leq -\eta ||\nabla F(w^{(t)})||_{1} + \frac{L\eta^2d}{2} +2\eta\sum_{i=1}^{d}|\nabla F(w^{(t)})_{i}|\bigg[\bigg(1-\frac{|\nabla F(w^{(t)})_i|}{b_{i}}\bigg)e^{\frac{|\nabla F(w^{(t)})_i|}{b_{i}}}\bigg]^{\frac{M}{2}}
\end{split}
\end{equation}

Let  $\Delta(M)$ denote the solution to $\big[\big(1-x\big)e^{x}\big]^{\frac{M}{2}} = \frac{1-c}{2}$. Since $\big(1-x\big)e^{x}$ is a decreasing function of $x$ for $0 < x < 1$, it can be verified that $\big[\big(1-x\big)e^{x}\big]^{\frac{M}{2}} < \frac{1-c}{2}$ when $x > \Delta(M)$ and $\big[\big(1-x\big)e^{x}\big]^{\frac{M}{2}} \geq \frac{1-c}{2}$ otherwise. According to Theorem \ref{Lemma1}, we have two possible scenarios as follows.
\begin{equation}\label{Constant}
P\bigg(sign\bigg(\frac{1}{M}\sum_{m=1}^{M}sto\text{-}sign(\boldsymbol{g}_{m}^{(t)})_{i}\bigg)\neq sign(\nabla F(w^{(t)})_{i})\bigg)
\begin{cases}
\leq \frac{1-c}{2}, \text{if}~ \frac{|\nabla F(w^{(t)})_i|}{b_{i}} > \Delta(M), \\
\leq 1, \text{if}~ \frac{|\nabla F(w^{(t)})_i|}{b_{i}} \leq \Delta(M).
\end{cases}
\end{equation}

Plugging (\ref{Constant}) into (\ref{convergencee1}), we can obtain
\begin{equation}
\color{black}
\begin{split}
&\mathbb{E}[F(w^{(t+1)}) - F(w^{(t)})] \\
&\leq -\eta ||\nabla F(w^{(t)})||_{1} + \frac{L\eta^2d}{2} \\
& + \eta \bigg[(1-c)\sum_{i=1}^{d}|\nabla F(w^{(t)})_i|\mathds{1}_{p_{i}^{(t)} \leq \frac{1-c}{2}} + (1-c)\sum_{i=1}^{d}|\nabla F(w^{(t)})_i|\mathds{1}_{p_{i}^{(t)} > \frac{1-c}{2}}\bigg]\\
&+ 2\eta\sum_{i=1}^{d}|\nabla F(w^{(t)})_i|\mathds{1}_{p_{i}^{(t)} > \frac{1-c}{2}}\\
&\leq -\eta ||\nabla F(w^{(t)})||_{1} + \frac{L\eta^2d}{2} + \eta (1-c)||\nabla F(w^{(t)})||_{1}+2\eta\sum_{i=1}^{d}|\nabla F(w^{(t)})_i|\mathds{1}_{p_{i}^{(t)} > \frac{1-c}{2}} \\
&=-\eta c||\nabla F(w^{(t)})||_{1} + \frac{L\eta^2d}{2} + 2\eta\sum_{i=1}^{d}|\nabla F(w^{(t)})_i|\mathds{1}_{p_{i}^{(t)} > \frac{1-c}{2}},
\end{split}
\end{equation}
{\color{black}where $p_{i}^{(t)}=P\big(sign\big(\frac{1}{M}\sum_{m=1}^{M}sto\text{-}sign(\boldsymbol{g}_{m}^{(t)})_{i}\big)\neq sign(\nabla F(w^{(t)})_{i})\big)$ is the probability of wrong aggregation.} Adjusting the above inequality and averaging both sides over $t=1,2,\cdots,T$, we can obtain
\begin{equation}
\color{black}
\begin{split}
\frac{1}{T}\sum_{t=1}^{T}\eta c||\nabla F(w^{(t)})||_{1} &\leq \frac{\mathbb{E}[F(w^{(0)}) - F(w^{(T+1)})]}{T} + \frac{L\eta^2d}{2} \\
&+ \frac{2\eta}{T}\sum_{t=1}^{T}\sum_{i=1}^{d}|\nabla F(w^{(t)})_i|\mathds{1}_{p_{i}^{(t)} > \frac{1-c}{2}}\\
&\leq \frac{\mathbb{E}[F(w^{(0)}) - F(w^{(T+1)})]}{T} + \frac{L\eta^2d}{2} + 2\eta \sum_{i=1}^{d}b_{i}\Delta(M),
\end{split}
\end{equation}
{\color{black}where the last inequality is due to the fact that $p_{i}^{(t)} > \frac{1-c}{2}$ only when $\frac{|\nabla F(w^{(t)})_i|}{b_{i}} \leq \Delta(M)$.}
Letting $\eta=\frac{1}{\sqrt{dT}}$ and dividing both sides by $c\eta$ gives
\begin{equation}
\color{black}
\begin{split}
&\frac{1}{T}\sum_{t=1}^{T}||\nabla F(w^{(t)})||_{1} \\
&\leq
\frac{1}{c}\bigg[\frac{\mathbb{E}[F(w^{(0)}) - F(w^{(T+1)})]\sqrt{d}}{\sqrt{T}} + \frac{L\sqrt{d}}{2\sqrt{T}} + \frac{2\eta}{T}\sum_{t=1}^{T}\sum_{i=1}^{d}|\nabla F(w^{(t)})_i|\mathds{1}_{p_{i}^{(t)} > \frac{1-c}{2}}\bigg]\\
&\leq \frac{1}{c}\bigg[\frac{(F(w^{(0)}) - F^{*})\sqrt{d}}{\sqrt{T}} + \frac{L\sqrt{d}}{2\sqrt{T}} + 2\sum_{i=1}^{d}b_{i}\Delta(M)\bigg],
\end{split}
\end{equation}
which completes the proof.
\end{proof}

\subsection{Proof of Theorem 3}
\begin{theorem}\label{SPbsufficientlylarge}
Given Assumption 3 and the same $\{u_{m}\}_{m=1}^{M}$ and $\{\hat{u}_{m}\}_{m=1}^{M}$ as those in Theorem \ref{Lemma1}, we have  %$P\big(sign\big(\frac{1}{M}\sum_{m=1}^{M}\hat{u}_{m}\big)\neq sign\big(\frac{1}{M}\sum_{m=1}^{M}u_{m}\big)\big) < \frac{1}{2}$.
\begin{equation}\label{SPT3EQ}
\begin{split}
P\bigg(sign\bigg(\frac{1}{M}\sum_{m=1}^{M}\hat{u}_{m}\bigg)\neq sign\bigg(\frac{1}{M}\sum_{m=1}^{M}u_{m}\bigg)\bigg) = \frac{1}{2} - \frac{{M-1 \choose \frac{M-1}{2}}}{2^{M}b}\bigg|\sum_{m=1}^{M}u_{m}\bigg| + O\bigg(\frac{1}{b^2}\bigg)
\end{split}
\end{equation}
\end{theorem}
\begin{proof}
Without loss of generality, assume $u_{1} \leq u_{2} \leq \cdots \leq u_{K} < 0 \leq u_{K+1} \leq \cdots \leq u_{M}$. According to the definition of $sto\text{-}sign$, we have

\begin{equation}
\hat{u}_{m} = sto\text{-}sign(u_m,b) =
\begin{cases}
\hfill 1, \hfill \text{with probability $\frac{b+u_m}{2b}$},\\
\hfill -1, \hfill \text{with probability $\frac{b-u_m}{2b}$},\\
\end{cases}
\end{equation}

Further define a series of random variables $\{\hat{X}_m\}_{m=1}^{M}$ given by
\begin{equation}
\hat{X}_{m} =
\begin{cases}
\hfill 1, \hfill &\text{if $\hat{u}_{m} =1$},\\
\hfill 0, \hfill &\text{if $\hat{u}_{m} = -1$.}
\end{cases}
\end{equation}

In particular, $\hat{X}_{m}$ can be considered as the outcome of one Bernoulli trial with successful probability $P(\hat{X}_{m} = 1)$. Let $\hat{Z} = \sum_{m=1}^{M}\hat{X}_m$, then
\begin{equation}
P\bigg(sign\bigg(\frac{1}{M}\sum_{m=1}^{M}\hat{u}_m\bigg)=1\bigg) = P\bigg(\hat{Z} \geq \frac{M}{2}\bigg) = \sum_{H = \frac{M+1}{2}}^{M}P(\hat{Z} = H).
\end{equation}

In addition,
\begin{equation}\label{Chapter4-SPhatZ}
P(\hat{Z}=H) = \frac{\sum_{A \in F_H}\prod_{i \in A}(b+u_{i})\prod_{j \in A^{c}}(b-u_{j})}{(2b)^{M}} = \frac{a_{M,H}b^{M} + a_{M-1,H}b^{M-1} + \cdots + a_{0,H}b^{0}}{(2b)^{M}},
\end{equation}
in which $F_H$ is the set of all subsets of $H$ integers that can be selected from $\{1,2,3,...,M\}$; $a_{m,H}, \forall 0\leq m \leq M$ is some constant. It can be easily verified that $a_{M,H} = {M \choose H}$.

When $b$ is sufficiently large, $P(\hat{Z}=H)$ is dominated by the first two terms in (\ref{Chapter4-SPhatZ}). As a result, we have
\begin{equation}
    P(\hat{Z}=H) = \frac{a_{M,H}b^{M} + a_{M-1,H}b^{M-1}}{(2b)^{M}} + O\bigg(\frac{1}{b^2}\bigg).
\end{equation}

In particular, $\forall m$, we have
\begin{equation}
\begin{split}
\sum_{A \in F_H}\prod_{i \in A}(b+u_{i})\prod_{j \in A^{c}}(b-u_{j}) &= (b+u_{m})\sum_{A \in F_H, m\in A}\prod_{i \in A/\{m\}}(b+u_{i})\prod_{j \in A^{c}}(b-u_{j}) \\&+ (b-u_{m})\sum_{A \in F_H, m\notin A}\prod_{i \in A}(b+u_{i})\prod_{j \in A^{c}/\{m\}}(b-u_{j}).
\end{split}
\end{equation}
As a result, when $\frac{M+1}{2} \leq H \leq M-1$, the $u_{m}$ related term in $a_{M-1,H}$ is given by
\begin{equation}
\bigg[{M-1 \choose H-1} - {M-1 \choose H}\bigg]u_{m}.
\end{equation}
When $H = M$, the $u_{m}$ related term in $a_{M-1,H}$ is given by
\begin{equation}
\bigg[{M-1 \choose H-1}\bigg]u_{m}.
\end{equation}
By summing over $m$, we have
\begin{equation}
a_{M-1,H} = \bigg[{M-1 \choose H-1} - {M-1 \choose H}\bigg]\sum_{m=1}^{M}u_{m}, ~~~~~\text{if}~~~ \frac{M+1}{2} \leq H \leq M-1,
\end{equation}
and
\begin{equation}
a_{M-1,H} = \bigg[{M-1 \choose H-1}\bigg]\sum_{m=1}^{M}u_{m}, ~~~~~\text{if}~~~H = M.
\end{equation}

By summing over $H$, we have
\begin{equation}
\sum_{H = \frac{M+1}{2}}^{M}a_{M,H} = \sum_{H = \frac{M+1}{2}}^{M}{M \choose H} = 2^{M-1},
\end{equation}
\begin{equation}
\sum_{H = \frac{M+1}{2}}^{M}a_{M-1,H} = {M-1 \choose \frac{M-1}{2}}\sum_{m=1}^{M}u_{m}.
\end{equation}
As a result,
\begin{equation}
\begin{split}
P\bigg(sign\bigg(\frac{1}{M}\sum_{m=1}^{M}\hat{u}_m\bigg)=1\bigg) &= P\bigg(\hat{Z} \geq \frac{M}{2}\bigg) \\
&= \sum_{H = \frac{M+1}{2}}^{M}P(\hat{Z} = H) \\
&= \frac{2^{M-1}b^{M} + {M-1 \choose \frac{M-1}{2}}\sum_{m=1}^{M}u_{m}b^{M-1}}{(2b)^{M}} + O\bigg(\frac{1}{b^2}\bigg)\\
& = \frac{1}{2} + \frac{{M-1 \choose \frac{M-1}{2}}}{2^{M}b}\sum_{m=1}^{M}u_{m} + O\bigg(\frac{1}{b^2}\bigg).
\end{split}
\end{equation}

Therefore,
\begin{equation}
\begin{split}
P\bigg(sign\bigg(\frac{1}{M}\sum_{m=1}^{M}\hat{u}_{m}\bigg)\neq sign\bigg(\frac{1}{M}\sum_{m=1}^{M}u_{m}\bigg)\bigg) = \frac{1}{2} - \frac{{M-1 \choose \frac{M-1}{2}}}{2^{M}b}\bigg|\sum_{m=1}^{M}u_{m}\bigg| + O\bigg(\frac{1}{b^2}\bigg).
\end{split}
\end{equation}
\end{proof}

% \begin{Remark}
% We note that if we plug (\ref{SPT3EQ}) into (\ref{convergencee1}) and follow a similar strategy as that in the proof of Theorem \ref{convergerate} with $\boldsymbol{b}=b\cdot \boldsymbol{1}$ and $\frac{c}{2} = \frac{M{M-1 \choose \frac{M-1}{2}}}{2^{M}b}|\nabla F(w^{(t)})_{i}| - O(\frac{1}{b^2})$, we can obtain the following result.
% \begin{equation}
% \begin{split}
% \frac{1}{T}\sum_{t=1}^{T}\sum_{i=1}^{d}\frac{|\nabla F(w^{(t)})_{i}|^2}{b} &\leq \frac{2^{M}}{2M{M-1 \choose \frac{M-1}{2}}}\bigg[\frac{(F(w^{(0)})-F^{*})\sqrt{d}}{\sqrt{T}} + \frac{L\sqrt{d}}{2\sqrt{T}} + \frac{1}{T}\sum_{t=1}^{T}\sum_{i=1}^{d}|\nabla F(w^{(t)})_{i}|O\bigg(\frac{1}{b^2}\bigg)\bigg].
% \end{split}
% \end{equation}
% It is known that $\frac{n^{n+1}e^{-n}\sqrt{2\pi}}{\sqrt{n}} \leq n! < \frac{n^{n+1}e^{-n}\sqrt{2\pi}}{\sqrt{n-1}}$ \cite{batir2008sharp}. With some algebra, we can show that $\frac{2^{M}}{2M{M-1 \choose \frac{M-1}{2}}} < \frac{\sqrt{2\pi}(M-1)^{\frac{3}{2}}}{2(M^{2}-3M)}\leq O(\frac{1}{\sqrt{M}})$. Therefore, if we select $b=T^{1/4}d^{1/4}$ as \cite{chen2019distributed}, we obtain a convergence rate of $O(\frac{d^{3/4}}{T^{1/4}M^{1/2}})$, which improves the rate in \cite{chen2019distributed} by a factor of $O(\frac{1}{\sqrt{M}})$.
% \end{Remark}

\subsection{Proof of Theorem 4}
\begin{theorem}\label{SPconvergeratelargeb}
Suppose Assumptions \ref{A1}, \ref{A2} and \ref{A3} are satisfied, $|\nabla F(w^{(t)})_{i}| \leq Q, \forall 1\leq i \leq d, 1\leq t \leq T$, and the learning rate is set as $\eta=\frac{1}{\sqrt{Td}}$. Then by running Algorithm 1 with $q(\boldsymbol{g}_{m}^{(t)}) = sto\text{-}sign(\nabla f_{m}(w^{(t)}),\boldsymbol{b})$ and $b_{i}=T^{1/4}d^{1/4}, \forall i$ for $T$ iterations, we have
\begin{equation}
\begin{split}
&\frac{1}{T}\sum_{t=1}^{T}\sum_{i=1}^{d}|\nabla F(w^{(t)})_{i}|^2 \\
&\leq \frac{2^{M}}{2M{M-1 \choose \frac{M-1}{2}}}\bigg[\frac{(F(w^{(0)})-F^{*})d^{3/4}}{T^{1/4}} + \frac{Ld^{3/4}}{2T^{1/4}} + \frac{2}{T}\sum_{t=1}^{T}\sum_{i=1}^{d}|\nabla F(w^{(t)})_{i}|O\bigg(\frac{1}{T^{1/4}d^{1/4}}\bigg)\bigg]\\
&\leq \frac{\sqrt{2\pi}(M-1)^{\frac{3}{2}}}{2(M^{2}-3M)}\bigg[\frac{(F(w^{(0)})-F^{*})d^{3/4}}{T^{1/4}} + \frac{Ld^{3/4}}{2T^{1/4}} + \frac{2}{T}\sum_{t=1}^{T}\sum_{i=1}^{d}|\nabla F(w^{(t)})_{i}|O\bigg(\frac{1}{T^{1/4}d^{1/4}}\bigg)\bigg],
\end{split}
\end{equation}
which further captures the impact of $M$ (i.e., $\frac{\sqrt{2\pi}(M-1)^{\frac{3}{2}}}{2(M^{2}-3M)} \leq O(\frac{1}{\sqrt{M}})$) compared to \cite{chen2019distributed}.
\end{theorem}
\begin{proof}
According to (\ref{convergencee1}), we have

\begin{equation}\label{convergencee1_1}
\begin{split}
&\mathbb{E}[F(w^{(t+1)}) - F(w^{(t)})] \leq -\eta ||\nabla F(w^{(t)})||_{1} + \frac{L\eta^2d}{2} \\
&+2\eta\sum_{i=1}^{d}|\nabla F(w^{(t)})_{i}|P\bigg(sign\bigg(\frac{1}{M}\sum_{m=1}^{M}sto\text{-}sign(\boldsymbol{g}_{m}^{(t)})_{i}\bigg)\neq sign(\nabla F(w^{(t)})_{i})\bigg)\\
% &\leq -\eta ||\nabla F(w^{(t)})||_{1} + \frac{L\eta^2d}{2} +2\eta\sum_{i=1}^{d}|\nabla F(w^{(t)})_{i}|\bigg[\bigg(1-\frac{|\nabla F(w^{(t)})_i|}{b_{i}}\bigg)e^{\frac{|\nabla F(w^{(t)})_i|}{b_{i}}}\bigg]^{\frac{M}{2}}
\end{split}
\end{equation}

According to Lemma \ref{bsufficientlylarge}

\begin{equation}
\begin{split}
P\bigg(sign\bigg(\frac{1}{M}\sum_{m=1}^{M}sto\text{-}sign(\boldsymbol{g}_{m}^{(t)})_{i}\bigg)\neq sign(\nabla F(w^{(t)})_{i})\bigg) = \frac{1}{2} - \frac{M{M-1 \choose \frac{M-1}{2}}}{2^{M}b_{i}}\big|\nabla F(w^{(t)})_{i}\big| + O\bigg(\frac{1}{b_{i}^2}\bigg)
\end{split}
\end{equation}

Plugging (\ref{Constant}) into (\ref{convergencee1_1}), we can obtain
\begin{equation}\label{T6E1}
\begin{split}
&\mathbb{E}[F(w^{(t+1)}) - F(w^{(t)})] \\
&\leq -\eta ||\nabla F(w^{(t)})||_{1} + \frac{L\eta^2d}{2}+2\eta\sum_{i=1}^{d}|\nabla F(w^{(t)})_{i}|\bigg[\frac{1}{2} - \frac{M{M-1 \choose \frac{M-1}{2}}}{2^{M}b_{i}}\big|\nabla F(w^{(t)})_{i}\big| + O\bigg(\frac{1}{b_{i}^2}\bigg)\bigg]\\
&=\frac{L\eta^2d}{2}-\eta\frac{2M{M-1 \choose \frac{M-1}{2}}}{2^{M}}\sum_{i=1}^{d}\frac{|\nabla F(w^{(t)})_{i}|^2}{b_i} + 2\eta\sum_{i=1}^{d}|\nabla F(w^{(t)})_{i}|O\bigg(\frac{1}{b_{i}^2}\bigg)
\end{split}
\end{equation}

Rearranging (\ref{T6E1}) gives
\begin{equation}
\begin{split}
\eta\frac{2M{M-1 \choose \frac{M-1}{2}}}{2^{M}}\sum_{i=1}^{d}\frac{|\nabla F(w^{(t)})_{i}|^2}{b_i} \leq \mathbb{E}[F(w^{(t)}) - F(w^{(t+1)})] + \frac{L\eta^2d}{2} + 2\eta\sum_{i=1}^{d}|\nabla F(w^{(t)})_{i}|O\bigg(\frac{1}{b_{i}^2}\bigg)
\end{split}
\end{equation}

Adjusting the above inequality and averaging both sides over $t=1,2,\cdots,T$ yields
\begin{equation}
\begin{split}
&\frac{1}{T}\sum_{t=1}^{T}\sum_{i=1}^{d}\frac{|\nabla F(w^{(t)})_{i}|^2}{b_i} \leq \frac{2^{M}}{2M{M-1 \choose \frac{M-1}{2}}}\bigg[\frac{(F(w^{(0)})-F^{*})}{\eta} + \frac{Ld\eta}{2} + \frac{2}{T}\sum_{t=1}^{T}\sum_{i=1}^{d}|\nabla F(w^{(t)})_{i}|O\bigg(\frac{1}{b_{i}^2}\bigg)\bigg].
\end{split}
\end{equation}

Setting $\eta=\frac{1}{\sqrt{Td}}$ and $b_{i}=T^{1/4}d^{1/4}$ gives

\begin{equation}
\begin{split}
\frac{1}{T}\sum_{t=1}^{T}\sum_{i=1}^{d}|\nabla F(w^{(t)})_{i}|^2 &\leq \frac{2^{M}}{2M{M-1 \choose \frac{M-1}{2}}}\bigg[\frac{(F(w^{(0)})-F^{*})d^{3/4}}{T^{1/4}} + \frac{Ld^{3/4}}{2T^{1/4}} \\
&+ \frac{2}{T}\sum_{t=1}^{T}\sum_{i=1}^{d}|\nabla F(w^{(t)})_{i}|O\bigg(\frac{1}{T^{1/4}d^{1/4}}\bigg)\bigg]
\end{split}
\end{equation}

It is known that $\frac{n^{n+1}e^{-n}\sqrt{2\pi}}{\sqrt{n}} \leq n! < \frac{n^{n+1}e^{-n}\sqrt{2\pi}}{\sqrt{n-1}}$ \cite{batir2008sharp}. With some algebra, we can show that $\frac{2^{M}}{2M{M-1 \choose \frac{M-1}{2}}} < \frac{\sqrt{2\pi}(M-1)^{\frac{3}{2}}}{2(M^{2}-3M)}\leq O(\frac{1}{\sqrt{M}})$. Therefore, we have

\begin{equation}
\begin{split}
\frac{1}{T}\sum_{t=1}^{T}\sum_{i=1}^{d}|\nabla F(w^{(t)})_{i}|^2 &\leq \frac{\sqrt{2\pi}(M-1)^{\frac{3}{2}}}{2(M^{2}-3M)}\bigg[\frac{(F(w^{(0)})-F^{*})d^{3/4}}{T^{1/4}} + \frac{Ld^{3/4}}{2T^{1/4}} +\\
&\frac{2}{T}\sum_{t=1}^{T}\sum_{i=1}^{d}|\nabla F(w^{(t)})_{i}|O\bigg(\frac{1}{T^{1/4}d^{1/4}}\bigg)\bigg]
\end{split}
\end{equation}
which completes the proof.
\end{proof}

\subsection{Proof of Theorem 5}
\begin{theorem}
The proposed compressor $dp\text{-}sign(\cdot,\epsilon,\delta)$ is $(\epsilon,\delta)$-differentially private for any $\epsilon, \delta \in (0,1)$.
\end{theorem}
\begin{proof}
We start from the one-dimension scenario and consider any $a,b$ that satisfy $||a - b||_{2} \leq \Delta_2$. Without loss of generality, assume that $dp\text{-}sign(a,\epsilon,\delta)=dp\text{-}sign(b,\epsilon,\delta)=1$. Then we have
\begin{equation}
\begin{split}
P(dp\text{-}sign(a,\epsilon,\delta)=1) = \Phi\bigg(\frac{a}{\sigma}\bigg) = \int_{-\infty}^{a}\frac{1}{\sqrt{2\pi}\sigma}e^{-\frac{x^2}{2\sigma^2}}dx,\\
P(dp\text{-}sign(b,\epsilon,\delta)=1) = \Phi\bigg(\frac{b}{\sigma}\bigg) = \int_{-\infty}^{b}\frac{1}{\sqrt{2\pi}\sigma}e^{-\frac{x^2}{2\sigma^2}}dx.
\end{split}
\end{equation}

Furthermore,
\begin{equation}
\begin{split}
\frac{P(dp\text{-}sign(a,\epsilon,\delta)=1)}{P(dp\text{-}sign(b,\epsilon,\delta)=1)} = \frac{\int_{-\infty}^{a}e^{-\frac{x^2}{2\sigma^2}}dx}{\int_{-\infty}^{b}e^{-\frac{x^2}{2\sigma^2}}dx} = \frac{\int_{0}^{\infty}e^{-\frac{(x-a)^2}{2\sigma^2}}dx}{\int_{0}^{\infty}e^{-\frac{(x-b)^2}{2\sigma^2}}dx}.
\end{split}
\end{equation}
According to Theorem A.1 in \cite{dwork2014algorithmic}, given the parameters $\epsilon, \delta$ and $\sigma$, it can be verified that $e^{-\epsilon} \leq \big|\frac{P(dp\text{-}sign(a,\epsilon,\delta)=1)}{P(dp\text{-}sign(b,\epsilon,\delta)=1)}\big| \leq e^{\epsilon}$ with probability at least $1-\delta$.

For the multi-dimension scenario, consider any vector $\boldsymbol{a}$ and $\boldsymbol{b}$ such that $||\boldsymbol{a} - \boldsymbol{b}||_{2} \leq \Delta_2$ and $\boldsymbol{v} \in \{-1,1\}^{d}$, we have
\begin{equation}
\begin{split}
\frac{P(dp\text{-}sign(\boldsymbol{a},\epsilon,\delta)=\boldsymbol{v})}{P(dp\text{-}sign(\boldsymbol{b},\epsilon,\delta)=\boldsymbol{v})} = \frac{\int_{D}e^{-\frac{||\boldsymbol{x}-\boldsymbol{a}||_{2}^2}{2\sigma^2}}d\boldsymbol{x}}{\int_{D}e^{-\frac{||\boldsymbol{x}-\boldsymbol{b}||_{2}^2}{2\sigma^2}}d\boldsymbol{x}},
\end{split}
\end{equation}
where $D$ is some integral area depending on $\boldsymbol{v}$. Similarly, it can be shown that $e^{-\epsilon} \leq \big|\frac{P(dp\text{-}sign(\boldsymbol{a},\epsilon,\delta)=\boldsymbol{v})}{P(dp\text{-}sign(\boldsymbol{b},\epsilon,\delta)=\boldsymbol{v})}\big| \leq e^{\epsilon}$ with probability at least $1-\delta$.
\end{proof}

\subsection{Proof of Theorem 6}
\begin{theorem}\label{SPLemma2}
Let $u_{1},u_{2},\cdots,u_{M}$ be $M$ known and fixed real numbers. Further define random variables $\hat{u}_{i}=dp\text{-}sign(u_{i},\epsilon,\delta), \forall 1\leq i \leq M$. Then there always exist a constant $\sigma_{0}$ such that when $\sigma \geq \sigma_{0}$, $P(sign(\frac{1}{M}\sum_{m=1}^{M}\hat{u}_{i})\neq sign(\frac{1}{M}\sum_{m=1}^{M}u_{i})) <\big(1-x^2\big)^{\frac{M}{2}}$,
where $x = \frac{|\sum_{m=1}^{M}u_{m}|}{2\sigma M}$.
\end{theorem}
The proof of Theorem \ref{SPLemma2} follows a similar strategy to that of Theorem \ref{SPLemma1}. The difficulty we need to overcome is that unlike $sto\text{-}sign$, the expectation of the number of workers that share the wrong signs is not a function of $\frac{1}{M}\sum_{m=1}^{M}u_{i}$ due to the nonlinearity introduced by $\Phi(\cdot)$. However, when $\sigma$ is large enough, we show that it can be upper bounded as a function of $\frac{1}{M}\sum_{m=1}^{M}u_{i}$.

\begin{proof}
Without loss of generality, assume $u_{1} \leq u_{2} \leq \cdots \leq u_{K} < 0 \leq u_{K+1} \leq \cdots \leq u_{M}$ and $\frac{1}{M}\sum_{i=1}^{M}u_{i} < 0$. Note that similar analysis can be done when $\frac{1}{M}\sum_{i=1}^{M}u_{i} > 0$.

We are interested in obtaining $\Bar{p}_{dp} = \frac{1}{M}\sum_{m=1}^{M}\Pr\left(sign\left(\frac{1}{M}\sum_{m=1}^{M}u_{m}\right) \neq \hat{u}_{m}\right)$. In particular,
\begin{equation}
\Pr\left(sign\left(\frac{1}{M}\sum_{m=1}^{M}u_{m}\right) \neq \hat{u}_{m}\right) = \Phi\bigg(\frac{u_{m}}{\sigma}\bigg),
\end{equation}

\begin{equation}
\begin{split}
&\Bar{p}_{dp} = \frac{1}{M}\sum_{m=1}^{M}\Pr\left(sign\left(\frac{1}{M}\sum_{m=1}^{M}u_{m}\right) \neq \hat{u}_{m}\right) = \frac{1}{M}\sum_{m=1}^{M}\Phi\bigg(\frac{u_{m}}{\sigma}\bigg).
\end{split}
\end{equation}

Let $n$ denote a zero-mean Gaussian noise with variance $\sigma$, according to the assumption that $u_{1} \leq u_{2} \leq \cdots \leq u_{K} < 0 \leq u_{K+1} \leq \cdots \leq u_{M}$, we have
\begin{equation}
\begin{split}
&\Phi\bigg(\frac{u_{m}}{\sigma}\bigg) = \frac{1}{2} - P(u_{m} < n < 0),~~\forall 1 \leq m \leq K,\\
&\Phi\bigg(\frac{u_{m}}{\sigma}\bigg) = \frac{1}{2} + P(0 < n < u_{m}),~~\forall K+1 \leq m \leq M.
\end{split}
\end{equation}

Therefore,
\begin{equation}
\begin{split}
\Bar{p}_{dp} = \frac{1}{M}\sum_{m=1}^{M}\Phi\bigg(\frac{u_{m}}{\sigma}\bigg) = \frac{1}{2} - \frac{1}{M}\bigg[\sum_{m=1}^{K}P(u_{m} < n < 0)-\sum_{m=K+1}^{M}P(0 < n < u_{m})\bigg].
\end{split}
\end{equation}

Note that for any Gaussian noise, $P(a_1 < n < 0) + P(a_2 < n < 0) \geq P(a_1 + a_2 < n < 0)$ for any $a_{1} < 0, a_{2} < 0$. Therefore, we consider the worst case scenario such that $\sum_{m=1}^{K}P(u_{m} < n < 0)-\sum_{m=K+1}^{M}P(0 < n < u_{m})$ is minimized, i.e., $K = 1$. In this case,
\begin{equation}
\begin{split}
&\sum_{m=1}^{K}P(u_{m} < n < 0)-\sum_{m=K+1}^{M}P(0 < n < u_{m}) \\
& = P\bigg(u_{1} < n \leq -\sum_{m=2}^{M}u_{m}\bigg) + P\bigg(-\sum_{m=2}^{M}u_{m} < n < 0\bigg) - \sum_{m=2}^{M}P(0 < n < u_{m})\\
& > \bigg|\sum_{m=1}^{M}u_{m}\bigg|\bigg[\frac{1}{\sqrt{2\pi}\sigma}e^{-\frac{u_{1}^2}{2\sigma^2}}\bigg] + P\bigg(-\sum_{m=2}^{M}u_{m} < n < 0\bigg) - \sum_{m=2}^{M}P(0 < n < u_{m}) \\
& > \bigg|\sum_{m=1}^{M}u_{m}\bigg|\bigg[\frac{1}{\sqrt{2\pi}\sigma}e^{-\frac{u_{1}^2}{2\sigma^2}}\bigg] - \bigg|\sum_{m=2}^{M}u_{m}\bigg|\frac{1}{\sqrt{2\pi}\sigma}\bigg[1-e^{-\frac{(\sum_{m=2}^{M}u_{m})^2}{2\sigma^2}}\bigg] \\
&= \frac{1}{\sqrt{2\pi}\sigma}\bigg[\bigg|\sum_{m=1}^{M}u_{m}\bigg|e^{-\frac{u_{1}^2}{2\sigma^2}}+\bigg|\sum_{m=2}^{M}u_{m}\bigg|\bigg[e^{-\frac{(\sum_{m=2}^{M}u_{m})^2}{2\sigma^2}}-1\bigg]\bigg],
\end{split}
\end{equation}
where the first inequality is due to $f(a) > f(u_{1})$ for $a \in (u_{1}, \sum_{m=1}^{M}u_{m}]$ and the second inequality is due to $f(a) < \frac{1}{\sqrt{2\pi}\sigma}$ for any $a > 0$, where $f(\cdot)$ is the probability density function of the normal distribution.

In particular, as $\sigma \rightarrow \infty$, $|\sum_{m=1}^{M}u_{m}|e^{-\frac{u_{1}^2}{2\sigma^2}}$ increases and converges to $|\sum_{m=1}^{M}u_{m}|$ while $|\sum_{m=2}^{M}u_{m}|\bigg[e^{-\frac{(\sum_{m=2}^{M}u_{m})^2}{2\sigma^2}}-1\bigg]$ increases and converges to 0. Therefore, we have
\begin{equation}
\frac{1}{\sqrt{2\pi}\sigma}\bigg[\bigg|\sum_{m=1}^{M}u_{m}\bigg|e^{-\frac{u_{1}^2}{2\sigma^2}}+\bigg|\sum_{m=2}^{M}u_{m}\bigg|\bigg[e^{-\frac{(\sum_{m=2}^{M}u_{m})^2}{2\sigma^2}}-1\bigg]\bigg] \xrightarrow{\sigma \rightarrow \infty} -\frac{\sum_{m=1}^{M}u_{m}}{\sqrt{2\pi}\sigma}.
\end{equation}

As a result, there exists a $\sigma_{0}$ such that when $\sigma \geq \sigma_{0}$, we have
\begin{equation}\label{meanlabel}
\Bar{p}_{dp} = \frac{1}{M} \sum_{m=1}^{M}\Phi\bigg(\frac{u_{m}}{\sigma}\bigg) \leq \frac{1}{2} + \frac{\sum_{m=1}^{M}u_{m}}{4M\sigma}.
\end{equation}

Following the same analysis as that in the proof of Corollary \ref{SPLemma1}, we can show that
\begin{equation}
P\bigg(sign\bigg(\frac{1}{M}\sum_{m=1}^{M}\hat{u}_{i}\bigg)\neq sign\bigg(\frac{1}{M}\sum_{m=1}^{M}u_{i}\bigg)\bigg) <\big(1-x^2\big)^{\frac{M}{2}},
\end{equation}
where $x = \frac{|\sum_{m=1}^{M}u_{m}|}{2\sigma M}$.
\end{proof}

\subsection{Proof of Theorem 7}
\begin{theorem}\label{SPByzantineResienceDP}
\color{black}
During $t$-th communication round, let $\frac{1}{M}\sum_{m=1}^{M}\Pr\left(sign\left(\nabla F(w^{(t)})\right)_{i} \neq q(\boldsymbol{g}_{m}^{(t)})_{i}\right) = \Bar{p}_{i}^{(t)}$, then Algorithm 1 can at least tolerate $k_{i}$ Byzantine attackers on the $i$-th coordinate of the gradient and $k_{i}$ satisfies
\begin{enumerate}
    \item $\Bar{p}_{i}^{(t)} \leq \frac{M-k_{i}}{2M}$.
    \item There exists some positive constant $c$ such that \begin{equation}\label{SPByzantineInequa}
    \begin{split}
        &\left[\frac{(M-k_{i})(1-\Bar{p}_{i}^{(t)})}{(M+k_{i})\Bar{p}_{i}^{(t)}}\right]^{\frac{k}{2}} \left(\sqrt{\frac{M-k_{i}}{M+k_{i}}}+\sqrt{\frac{M+k_{i}}{M-k_{i}}}\right)^{M}\left[\Bar{p}_{i}^{(t)}(1-\Bar{p}_{i}^{(t)})\right]^{\frac{M}{2}} \leq \frac{1-c}{2}.
    \end{split}
    \end{equation}
\end{enumerate}
% \begin{equation}\label{ByzantineInequa}
% \begin{split}
% k_{i} < \frac{|\sum_{m=1}^{M}\nabla f_{m}(w^{(t)})_{i}|}{b_{i}},~~~~\bigg[\bigg(1-\frac{1}{x}\bigg)e^{\frac{1}{x}}\bigg]^{\frac{M-k_{i}}{2}} < \frac{1}{2},
% \end{split}
% \end{equation}
% where $x = \frac{(M-k_{i})b_{i}}{|\sum_{m=1}^{M}\nabla f_{m}(w^{(t)})_{i}|-b_{i}k_{i}}$.
Overall, the number of Byzantine workers that the algorithms can tolerate is given by $min_{1\leq i \leq d}k_{i}$.
\end{theorem}

We first provide some intuition about the proof. It has been shown in the proof of Theorem \ref{SPconvergerate} that the convergence of Algorithm 1 is guaranteed if there exists some positive constant $c$ such that the probability of more than half of the workers sharing wrong signs is no larger than $\frac{1-c}{2}$. On the $i$-th coordinate of the gradient, if there are $k_{i}$ Byzantine workers that always share the wrong signs, then at most $\frac{M-k_{i}}{2}$ normal workers can share wrong signs such that the aggregated result is still correct.

\begin{proof}
\color{black}
We first consider the same setting as in Theorem \ref{Lemma1} and define a series of random variables $\{X_m\}_{m=1}^{M}$ given by
\begin{equation}
X_{m} =
\begin{cases}
\hfill 1, \hfill &\text{if $\hat{u}_{m} \neq sign\bigg(\frac{1}{M}\sum_{m=1}^{M}u_{m}\bigg)$},\\
\hfill -1, \hfill &\text{if $\hat{u}_{m} = sign\bigg(\frac{1}{M}\sum_{m=1}^{M}u_{m}\bigg)$.}
\end{cases}
\end{equation}

In addition, let $\{\hat{u}_{j}\}_{j=M+1}^{M+k_i}$ denote the binary variables shared by the Byzantine attackers. Then, $X_{m}$ can be considered as the outcome of one Bernoulli trial with successful probability $P(X_{m} = 1)$, and we have
\begin{equation}
P\left(sign\left(\frac{1}{M+k_i}\left[\sum_{m=1}^{M}\hat{u}_{m}+\sum_{j=M+1}^{k_i}\hat{u}_{j}\right]\right)\neq sign\left(\frac{1}{M}\sum_{m=1}^{M}u_{m}\right)\right) = P\left(\sum_{m=1}^{M}X_{m}+k_{i} \geq 0\right).
\end{equation}

For any variable $a>0$, we have
\begin{equation}
\begin{split}
P\left(\sum_{m=1}^{M}X_m+k_{i} \geq 0\right) = P\left(e^{a(\sum_{m=1}^{M}X_m+k_{i})} \geq e^{0}\right) &\leq \frac{\mathbb{E}[e^{a(\sum_{m=1}^{M}X_m+k_{i})}]}{e^{0}} = \mathbb{E}[e^{a(\sum_{m=1}^{M}X_m+k_{i})}],
\end{split}
\end{equation}
which is due to Markov's inequality, given the fact that $e^{a(\sum_{m=1}^{M}X_m+k_{i})}$ is non-negative. For the ease of presentation, let $P(X_{m} = 1) = p_{m}$, we have,
\begin{equation}
\begin{split}
 \mathbb{E}[e^{a(\sum_{m=1}^{M}X_m+k_{i})}] &= e^{\ln(\mathbb{E}[e^{a(\sum_{m=1}^{M}X_m+k_{i})}])}= e^{\ln(\prod_{m=1}^{M}\mathbb{E}[e^{aX_m}])+ak_{i}} = e^{\sum_{m=1}^{M}\ln(\mathbb{E}[e^{aX_m}])+ak_{i}} \\
 &=e^{ak_{i}+\sum_{m=1}^{M}\ln(e^{a}p_{m}+e^{-a}(1-p_{m}))} \\
 &=e^{ak_{i}+M\left(\frac{1}{M}\sum_{m=1}^{M}\ln(e^{a}p_{m}+e^{-a}(1-p_{m}))\right)}\\
 &\leq e^{ak_{i}+M\ln(e^{a}\Bar{p}+e^{-a}(1-\Bar{p}))},
\end{split}
\end{equation}
where $\Bar{p} = \frac{1}{M}\sum_{m=1}^{M}p_{m}$ and the inequality is due to Jensen's inequality.
Optimizing $a$ yields $a=\ln\left(\sqrt{\frac{(M-k_i)(1-\Bar{p})}{(M+k_{i})\Bar{p}}}\right) > 0$ and
\begin{equation}
e^{ak_{i}+M\ln(e^{a}\Bar{p}+e^{-a}(1-\Bar{p}))} =\left[\frac{(M-k_{i})(1-\Bar{p})}{(M+k_{i})\Bar{p}}\right]^{\frac{k}{2}} \left(\sqrt{\frac{M-k_{i}}{M+k_{i}}}+\sqrt{\frac{M+k_{i}}{M-k_{i}}}\right)^{M}\left[\Bar{p}(1-\Bar{p})\right]^{\frac{M}{2}}.
\end{equation}

In addition, $a=\ln\left(\sqrt{\frac{(M-k_i)(1-\Bar{p})}{(M+k_{i})\Bar{p}}}\right) > 0$ indicates $\Bar{p} < \frac{M-k_{i}}{2M}$. Setting $sign(\nabla F(w^{(t)}))_{i} = sign(\frac{1}{M}\sum_{m=1}^{M}u_{M})$ and $\hat{u}_{m} = q(\boldsymbol{g}_{m}^{(t)})_{i}$ completes the proof.
\end{proof}

\subsection{Proof of Theorem 8}
\begin{theorem}
Suppose Assumptions \ref{A1}-\ref{A3} are satisfied, and set the learning rate $\eta=\frac{1}{\sqrt{Td}}$. Then, when $\boldsymbol{b}=b\cdot\boldsymbol{1}$ and $b$ is sufficiently large, {\scriptsize Sto-SIGN}SGD converges to the (local) optimum if either of the following two conditions is satisfied.
\begin{itemize}
  \item $P\big(sign(\frac{1}{M}\sum_{m=1}^{M}(\boldsymbol{g}_{m}^{t})_i)\neq sign(\nabla F(w^{t})_{i}\big)<0.5, \forall 1\leq i\leq d$.
  \item The mini-batch size of stochastic gradient at each iteration is at least $T$.
\end{itemize}
\end{theorem}

\begin{proof}
Note that in the proof of Theorem \ref{SPconvergerate} and Lemma \ref{bsufficientlylarge}, we obtain
\begin{equation}\label{SPT7C2E1}
\begin{split}
&\mathbb{E}[F(w^{(t+1)}) - F(w^{(t)})] \leq -\eta ||\nabla F(w^{(t)})||_{1} + \frac{L\eta^2d}{2} \\
&+ 2\eta\sum_{i=1}^{d}|\nabla F(w^{(t)})_{i}|P\bigg(sign\bigg(\frac{1}{M}\sum_{m=1}^{M}q(\boldsymbol{g}_{m}^{(t)})_{i}\bigg)\neq sign\bigg(\frac{1}{M}\sum_{m=1}^{M}\nabla f_{m}(w^{(t)})_{i}\bigg)\bigg),
\end{split}
\end{equation}
and
\begin{equation}
P\bigg(sign\bigg(\frac{1}{M}\sum_{m=1}^{M}q(\boldsymbol{g}_{m}^{(t)})_{i}\bigg)\neq sign\bigg(\frac{1}{M}\sum_{m=1}^{M}(\boldsymbol{g}_{m}^{(t)})_{i}\bigg)\bigg) < \frac{1}{2} - \frac{M{M-1 \choose \frac{M-1}{2}}}{2^{M}b_{i}}\big|\frac{1}{M}\sum_{m=1}^{M}(\boldsymbol{g}_{m}^{(t)})_{i}\big| + O\bigg(\frac{1}{b_{i}^2}\bigg),
\end{equation}
where $q(\boldsymbol{g}_{m}^{(t)}) = sto\text{-}sign(\boldsymbol{g}_{m}^{(t)})$ in {\scriptsize Sto-SIGN}SGD.

We first prove the convergence under the first condition. For the ease of notation, let
\begin{equation}\label{sgdextension1}
\begin{split}
&p_{i,1} = P\bigg(sign\bigg(\frac{1}{M}\sum_{m=1}^{M}q(\boldsymbol{g}_{m}^{(t)})_{i}\bigg)\neq sign\bigg(\frac{1}{M}\sum_{m=1}^{M}(\boldsymbol{g}_{m}^{(t)})_{i}\bigg)\bigg) \\
&= \frac{1}{2} - \frac{M{M-1 \choose \frac{M-1}{2}}}{2^{M}b_{i}}\big|\frac{1}{M}\sum_{m=1}^{M}(\boldsymbol{g}_{m}^{(t)})_{i}\big| + O\bigg(\frac{1}{b_{i}^2}\bigg), \\
&p_{i,2} = P\bigg(sign\bigg(\frac{1}{M}\sum_{m=1}^{M}\nabla f_{m}(w^{(t)})_{i}\bigg)\neq sign\bigg(\frac{1}{M}\sum_{m=1}^{M}(\boldsymbol{g}_{m}^{(t)})_{i}\bigg)\bigg) < \frac{1}{2}, \\
&p_i = P\bigg(sign\bigg(\frac{1}{M}\sum_{m=1}^{M}q(\boldsymbol{g}_{m}^{(t)})_{i}\bigg)\neq sign\bigg(\frac{1}{M}\sum_{m=1}^{M}\nabla f_{m}(w^{(t)})_{i}\bigg)\bigg).
\end{split}
\end{equation}

Then
\begin{equation}\label{sgdextension2}
\begin{split}
p_{i} &= p_{i,1}(1-p_{i,2}) + p_{i,2}(1-p_{i,1}) = p_{i,1} + p_{i,2} - 2p_{i,1}p_{i,2} = p_{i,2} + (1-2p_{i,2})p_{i,1}\\
&= p_{i,2} + (1-2p_{i,2})\left[\frac{1}{2} - \frac{M{M-1 \choose \frac{M-1}{2}}}{2^{M}b_{i}}\big|\frac{1}{M}\sum_{m=1}^{M}(\boldsymbol{g}_{m}^{(t)})_{i}\big| + O\bigg(\frac{1}{b_{i}^2}\bigg)\right]\\
&=\frac{1}{2} - (1-2p_{i,2})\left[\frac{M{M-1 \choose \frac{M-1}{2}}}{2^{M}b_{i}}\big|\frac{1}{M}\sum_{m=1}^{M}(\boldsymbol{g}_{m}^{(t)})_{i}\big| - O\bigg(\frac{1}{b_{i}^2}\bigg)\right]
\end{split}
\end{equation}

Following the same strategy as that in the proof of Theorem \ref{SPconvergeratelargeb}, the convergence can be established as follows.

\begin{equation}
\begin{split}
\frac{1}{T}\sum_{t=1}^{T}\sum_{i=1}^{d}|\nabla F(w^{(t)})_{i}|^2 &\leq \frac{\sqrt{2\pi}(M-1)^{\frac{3}{2}}}{2(M^{2}-3M)}\bigg[\frac{(F(w^{(0)})-F^{*})d^{3/4}}{T^{1/4}(1-2p_{i,2})} + \frac{Ld^{3/4}}{2T^{1/4}(1-2p_{i,2})} \\
&+ \frac{2}{T}\sum_{t=1}^{T}\sum_{i=1}^{d}|\nabla F(w^{(t)})_{i}|O\bigg(\frac{1}{T^{1/4}d^{1/4}}\bigg)\bigg]
\end{split}
\end{equation}

Then, we prove the convergence under the second condition. According to (\ref{sgdextension2}), it is obvious that $p_{i} \leq p_{i,1}+p_{i,2}$. Therefore, we have
\begin{equation}\label{sgdextension3}
\begin{split}
\sum_{i=1}^{d}|\nabla F(w^{(t)})_{i}|p_{i} \leq \sum_{i=1}^{d}|\nabla F(w^{(t)})_{i}|p_{i,1} + \sum_{i=1}^{d}|\nabla F(w^{(t)})_{i}|p_{i,2}.
\end{split}
\end{equation}

In particular,
\begin{equation}\label{SPpi2}
\begin{split}
&p_{i,2} = P\bigg(sign\bigg(\frac{1}{M}\sum_{m=1}^{M}\nabla f_{m}(w^{(t)})_{i}\bigg)\neq sign\bigg(\frac{1}{M}\sum_{m=1}^{M}(\boldsymbol{g}_{m}^{(t)})_{i}\bigg)\bigg) \\
&\leq P\bigg(\bigg|\frac{1}{M}\sum_{m=1}^{M}\nabla f_{m}(w^{(t)})_{i}-\frac{1}{M}\sum_{m=1}^{M}(\boldsymbol{g}_{m}^{(t)})_{i}\bigg| \geq \bigg|\frac{1}{M}\sum_{m=1}^{M}\nabla f_{m}(w^{(t)})_{i}\bigg|\bigg)\\
&\leq \frac{\mathbb{E}[|\frac{1}{M}\sum_{m=1}^{M}\nabla f_{m}(w^{(t)})_{i}-\frac{1}{M}\sum_{m=1}^{M}(\boldsymbol{g}_{m}^{(t)})_{i}|]}{|\frac{1}{M}\sum_{m=1}^{M}\nabla f_{m}(w^{(t)})_{i}|} \\
&\leq \frac{\sqrt{\mathbb{E}[(\frac{1}{M}\sum_{m=1}^{M}\nabla f_{m}(w^{(t)})_{i}-\frac{1}{M}\sum_{m=1}^{M}(\boldsymbol{g}_{m}^{(t)})_{i})^2]}}{|\frac{1}{M}\sum_{m=1}^{M}\nabla f_{m}(w^{(t)})_{i}|} \\
&\leq \frac{\sigma_{i}}{\sqrt{MT}|\nabla F(w^{(t)})_{i}|}.
\end{split}
\end{equation}

Plugging (\ref{SPpi2}) into (\ref{SPT7C2E1}) yields
\begin{equation}
\begin{split}
&\mathbb{E}[F(w^{(t+1)}) - F(w^{(t)})] \leq -\eta ||\nabla F(w^{(t)})||_{1} + \frac{L\eta^2d}{2} \\
&+ 2\eta\sum_{i=1}^{d}|\nabla F(w^{(t)})_{i}|P\bigg(sign\bigg(\frac{1}{M}\sum_{m=1}^{M}q(\boldsymbol{g}_{m}^{(t)})_{i}\bigg)\neq sign\bigg(\frac{1}{M}\sum_{m=1}^{M}(\boldsymbol{g}_{m}^{(t)})_{i}\bigg)\bigg) + 2\eta\sum_{i=1}^{d}\frac{\sigma_{i}}{\sqrt{MT}},
\end{split}
\end{equation}

Following the same strategy as that in the proof of Theorem \ref{SPconvergeratelargeb}, the convergence can be established as follows.
\begin{equation}
\begin{split}
\frac{1}{T}\sum_{t=1}^{T}\sum_{i=1}^{d}|\nabla F(w^{(t)})_{i}|^2 &\leq \frac{\sqrt{2\pi}(M-1)^{\frac{3}{2}}}{2(M^{2}-3M)}\bigg[\frac{(F(w^{(0)})-F^{*})d^{3/4}}{T^{1/4}} + \frac{Ld^{3/4}}{2T^{1/4}} \\
&+ \frac{2}{T}\sum_{t=1}^{T}\sum_{i=1}^{d}|\nabla F(w^{(t)})_{i}|O\bigg(\frac{1}{T^{1/4}d^{1/4}}\bigg)+\frac{2||\sigma||_1d^{1/4}}{M^{1/2}T^{1/4}}\bigg]
\end{split}
\end{equation}
\end{proof}

\subsection{Proof of Theorem 9}
\begin{theorem}\label{SPEFDPSIGNConvergence2}
When Assumptions \ref{A1}, \ref{A2} and \ref{A3} are satisfied, by running Algorithm \ref{Error-feedback noisy SignSGD DP} with $\eta = \frac{1}{\sqrt{Td}}$, $q(\boldsymbol{g}_{m}^{(t)}) = sto\text{-}sign(\nabla f_{m}(w^{(t)}),\boldsymbol{b})$ and $\boldsymbol{b} = b\cdot\boldsymbol{1}$, we have
\begin{equation}
\begin{split}
\frac{1}{T}\sum_{t=0}^{T-1}\frac{||\nabla F(w^{(t)})||^2_{2}}{b} &\leq \frac{(F(w_{0})-F^{*})\sqrt{d}}{\sqrt{T}} + \frac{(1+L+L^2\beta)\sqrt{d}}{\sqrt{T}},
\end{split}
\end{equation}
where $\beta$ is some positive constant.
\end{theorem}

The proof of Theorem \ref{EFDPSIGNConvergence2} follows the strategy of taking $y^{(t)}=w^{(t)}-\eta \tilde{\boldsymbol{e}}^{(t)}$ such that $y^{(t)}$ is updated in the same way as $w^{(t)}$ in the non error-feedback scenario. Therefore, before proving Theorem \ref{EFDPSIGNConvergence2}, we first prove the following lemmas.

% The proof of Theorem \ref{EFDPSIGNConvergence2} follows the strategy of taking $y^{(t)}=w^{(t)}-\eta \tilde{\boldsymbol{e}}^{(t)}$ such that $y^{(t)}$ is updated in the same way as $w^{(t)}$ in the non error-feedback scenario. A key technical challenge is to bound the norm of the residual error $||\tilde{\boldsymbol{e}}^{(t)}||_{2}^{2}$. Utilizing the fact that the output of the compressor $q(\cdot) \in \{-1,1\}$, we upper bound it by first proving that in this case, the server's compressor is an $\alpha$-approximate compressor \cite{karimireddy2019error} for some $\alpha < 1$. Therefore, before proving Theorem \ref{EFDPSIGNConvergence2}, we first prove the following lemmas.
\begin{Lemma}\label{Recurrence2}
Let $y^{(t)} = w^{(t)} - \eta \tilde{\boldsymbol{e}}^{(t)}$, we have
\begin{equation}
y^{(t+1)} = y^{(t)} - \eta\frac{1}{M}\sum_{m=1}^{M}sto\text{-}sign(\boldsymbol{g}_{m}^{(t)},\boldsymbol{b}).
\end{equation}
\end{Lemma}

\begin{proof}
\begin{equation}
\begin{split}
   y^{(t+1)} &=  w^{(t+1)} - \eta \tilde{\boldsymbol{e}}^{(t+1)}\\
     &= w^{(t)} - \eta\tilde{\boldsymbol{g}}^{(t)} - \eta \tilde{\boldsymbol{e}}^{(t+1)} \\
     &= w^{(t)} - \eta\bigg(\frac{1}{M}\sum_{m=1}^{M}sto\text{-}sign(\boldsymbol{g}_{m}^{(t)},\boldsymbol{b}) + \tilde{\boldsymbol{e}}^{(t)} - \tilde{\boldsymbol{e}}^{(t+1)}\bigg) - \eta \tilde{\boldsymbol{e}}^{(t+1)} \\
     &= w^{(t)} - \eta\frac{1}{M}\sum_{m=1}^{M}sto\text{-}sign(\boldsymbol{g}_{m}^{(t)},\boldsymbol{b}) - \eta \tilde{\boldsymbol{e}}^{(t)} \\
     & = y^{(t)} - \eta\frac{1}{M}\sum_{m=1}^{M}sto\text{-}sign(\boldsymbol{g}_{m}^{(t)},\boldsymbol{b}).
\end{split}
\end{equation}
\end{proof}

\begin{Lemma}\label{bound2}
There exists a positive constant $\beta > 0$ such that
$\mathbb{E}[||\tilde{\boldsymbol{e}}^{(t)}||^2_{2}] \leq \beta d, \forall t$.
\end{Lemma}

\begin{proof}
Since $\mathcal{C}(\cdot)$ is an $\alpha$-approximate compressor, it can be shown that
\begin{equation}
\begin{split}
\mathbb{E}||\tilde{\boldsymbol{e}}^{(t+1)}||_{2}^{2} &\leq (1-\alpha)\bigg|\bigg|\frac{1}{M}\sum_{m=1}^{M}q(\boldsymbol{g}_{m}^{(t)})+\tilde{\boldsymbol{e}}^{(t)}\bigg|\bigg|_{2}^{2} \\
&\leq (1-\alpha)(1+\rho)\mathbb{E}||\tilde{\boldsymbol{e}}^{(t)}||_2^2 + (1-\alpha)\bigg(1+\frac{1}{\rho}\bigg)\mathbb{E}\bigg|\bigg|\frac{1}{M}\sum_{m=1}^{M}q(\boldsymbol{g}_{m}^{(t)})\bigg|\bigg|_2^2 \\
&\leq \sum_{j=0}^{t}[(1-\alpha)(1+\rho)]^{t-j}(1-\alpha)\bigg(1+\frac{1}{\rho}\bigg)\mathbb{E}\bigg|\bigg|\frac{1}{M}\sum_{m=1}^{M}q(\boldsymbol{g}_{m}^{(t)})\bigg|\bigg|_2^2 \\
&\leq \frac{(1-\alpha)\bigg(1+\frac{1}{\rho}\bigg)d}{1-(1-\alpha)(1+\rho)},
\end{split}
\end{equation}
where we invoke Young's inequality recurrently and $\rho$ can be any positive constant. Therefore, there exists some constant $\beta>0$ such that $\mathbb{E}[||\tilde{\boldsymbol{e}}^{(t)}||^2_{2}] \leq \beta d, \forall t$.
\end{proof}

Now, we are ready to prove Theorem \ref{EFDPSIGNConvergence2}.
\begin{proof}
Let $y^{(t)} = w^{(t)} - \eta \tilde{\boldsymbol{e}}^{(t)}$, according to Lemma \ref{Recurrence2}, we have
\begin{equation}\label{Convergence3}
\begin{split}
&\mathbb{E}[F(y^{(t+1)}) - F(y^{(t)})] \\
&\leq -\eta\mathbb{E}\bigg[<\nabla F(y^{(t)}), \frac{1}{M}\sum_{m=1}^{M}sto\text{-}sign(\boldsymbol{g}_{m}^{(t)},\boldsymbol{b})>\bigg] + \frac{L}{2}\mathbb{E}\bigg[\bigg|\bigg|\eta\frac{1}{M}\sum_{m=1}^{M}sto\text{-}sign(\boldsymbol{g}_{m}^{(t)},\boldsymbol{b})\bigg|\bigg|^2_{2}\bigg] \\
&= \eta\mathbb{E}\bigg[<\nabla F(w^{(t)}) - \nabla F(y^{(t)}), \frac{1}{M}\sum_{m=1}^{M}sto\text{-}sign(\boldsymbol{g}_{m}^{(t)},\boldsymbol{b})>\bigg] \\
&+ \frac{L\eta^2}{2}\mathbb{E}\bigg[\bigg|\bigg|\frac{1}{M}\sum_{m=1}^{M}sto\text{-}sign(\boldsymbol{g}_{m}^{(t)},\boldsymbol{b})\bigg|\bigg|^2_{2}\bigg] \\
& -  \eta\mathbb{E}\bigg[<\nabla F(w^{(t)}), \frac{1}{M}\sum_{m=1}^{M}sto\text{-}sign(\boldsymbol{g}_{m}^{(t)},\boldsymbol{b})>\bigg].
\end{split}
\end{equation}

We first bound the first term, in particular, we have
\begin{equation}\label{Convergence4}
\begin{split}
&<\nabla F(w^{(t)})- \nabla F(y^{(t)}), \frac{1}{M}\sum_{m=1}^{M}sto\text{-}sign(\boldsymbol{g}_{m}^{(t)},\boldsymbol{b})> \\
& \leq \frac{\eta}{2}||\frac{1}{M}\sum_{m=1}^{M}sto\text{-}sign(\boldsymbol{g}_{m}^{(t)},\boldsymbol{b})||^2_{2} + \frac{1}{2\eta}||\nabla F(w^{(t)})-\nabla F(y^{(t)})||^2_{2} \\
&\leq \frac{\eta}{2}||\frac{1}{M}\sum_{m=1}^{M}sto\text{-}sign(\boldsymbol{g}_{m}^{(t)},\boldsymbol{b})||^2_{2} + \frac{L^2}{2\eta}||y^{(t)} - w^{(t)}||^2_{2} \\
&= \frac{\eta}{2}||\frac{1}{M}\sum_{m=1}^{M}sto\text{-}sign(\boldsymbol{g}_{m}^{(t)},\boldsymbol{b})||^2_{2} + \frac{L^2\eta}{2}||\tilde{\boldsymbol{e}}^{(t)}||^2_{2} \\
&\leq \frac{\eta}{2}||\frac{1}{M}\sum_{m=1}^{M}sto\text{-}sign(\boldsymbol{g}_{m}^{(t)},\boldsymbol{b})||^2_{2} + \frac{L^2\eta \beta d}{2},
\end{split}
\end{equation}
where the second inequality is due to the $L$-smoothness of $F$.

Then, we can bound the last term as follows.
\begin{equation}\label{Convergence5}
\begin{split}
&-\mathbb{E}\bigg[<\nabla F(w^{(t)}), \frac{1}{M}\sum_{m=1}^{M}sto\text{-}sign(\boldsymbol{g}_{m}^{(t)},\boldsymbol{b})>\bigg] \\
&=-\mathbb{E}\bigg[\sum_{i=1}^{d}\nabla F(w^{(t)})_{i}\frac{1}{M}\sum_{m=1}^{M}sto\text{-}sign(\boldsymbol{g}_{m}^{(t)},\boldsymbol{b})_{i}\bigg]\\
&=-\sum_{i=1}^{d}\nabla F(w^{(t)})_{i}\frac{1}{M}\sum_{m=1}^{M}\frac{\nabla f_{m}(w^{(t)})_{i}}{b} \\
&= -\frac{||\nabla F(w^{(t)})||^2_{2}}{b},
\end{split}
\end{equation}

Plugging (\ref{Convergence4}) and (\ref{Convergence5}) into (\ref{Convergence3}) yields
\begin{equation}\label{Convergence6}
\begin{split}
\mathbb{E}[F(y^{(t+1)}) - F(y^{(t)})] &\leq \frac{\eta^2+L\eta^2}{2}\mathbb{E}\bigg[\bigg|\bigg|\frac{1}{M}\sum_{m=1}^{M}sto\text{-}sign(\boldsymbol{g}_{m}^{(t)},\boldsymbol{b})\bigg|\bigg|^2_{2}\bigg] + \frac{L^2\eta^2 \beta d}{2} - \frac{\eta||\nabla F(w^{(t)})||^2_{2}}{b}\\
&\leq \frac{(\eta^2+L\eta^2+L^2\eta^2\beta)d}{2} - \frac{\eta||\nabla F(w^{(t)})||^2_{2}}{b}.
\end{split}
\end{equation}

Rewriting (\ref{Convergence6}) and taking average over $t=0,1,2,\cdots,T-1$ on both sides yields
\begin{equation}
\begin{split}
&\frac{1}{T}\sum_{t=0}^{T-1}\frac{||\nabla F(w^{(t)})||^2_{2}}{b} \leq \sum_{t=0}^{T-1}\frac{\mathbb{E}[F(y^{(t)}) - F(y^{(t+1)})]}{\eta T} + \frac{(\eta+L\eta+L^2\eta\beta)d}{2}.\\
\end{split}
\end{equation}

Taking $\eta = \frac{1}{\sqrt{Td}}$ and $w^{(0)}=y^{(0)}$ yields
\begin{equation}
\begin{split}
\frac{1}{T}\sum_{t=0}^{T-1}\frac{||\nabla F(w^{(t)})||^2_2}{b} \leq \frac{(F(w^{(0)})-F^{*})\sqrt{d}}{\sqrt{T}} + \frac{(1+L+L^2\beta)\sqrt{d}}{\sqrt{T}}.
\end{split}
\end{equation}
\end{proof}
\subsection{Proof of Theorem 10}
\begin{theorem}
At each iteration $t$, Algorithm \ref{Error-feedback noisy SignSGD DP} can at least tolerate $k_{i} = |\sum_{m=1}^{M}\nabla f_{m}(w^{(t)})_i|/b$ Byzantine attackers on the $i$-th coordinate of the gradient. Overall, the number of Byzantine workers that Algorithm 2 can tolerate is given by $min_{1\leq i \leq d}k_{i}$.
\end{theorem}

By following a similar strategy to the proof of Theorem \ref{SPEFDPSIGNConvergence2} and taking the impact of Byzantine attackers into consideration, the convergence of Algorithm 2 in the presence of Byzantine attackers can be established.

\begin{proof}
Without loss of generality, assume that the first $M$ workers are normal and the last $B$ are Byzantine. Following a similar procedure to the proof of Theorem \ref{EFDPSIGNConvergence2}, we can show that
\begin{equation}
\begin{split}
&\mathbb{E}[F(y^{(t+1)}) - F(y^{(t)})] \\
&\leq -\eta\mathbb{E}\bigg[<\nabla F(y^{(t)}), \frac{1}{M+B}\bigg[\sum_{m=1}^{M}sto\text{-}sign(\boldsymbol{g}_{m}^{(t)},\boldsymbol{b})+\sum_{j=1}^{B}byzantine\text{-}sign(\boldsymbol{g}_{j}^{(t)})\bigg]>\bigg] \\
&+ \frac{L}{2}\mathbb{E}\bigg[\bigg|\bigg|\eta\frac{1}{M+B}\bigg[\sum_{m=1}^{M}sto\text{-}sign(\boldsymbol{g}_{m}^{(t)},\boldsymbol{b})+\sum_{j=1}^{B}byzantine\text{-}sign(\boldsymbol{g}_{j}^{(t)})\bigg]\bigg|\bigg|^2_{2}\bigg] \\
&= \eta\mathbb{E}\bigg[<\nabla F(w^{(t)}) - \nabla F(y^{(t)}), \frac{1}{M+B}\bigg[\sum_{m=1}^{M}sto\text{-}sign(\boldsymbol{g}_{m}^{(t)},\boldsymbol{b})+\sum_{j=1}^{B}byzantine\text{-}sign(\boldsymbol{g}_{j}^{(t)})\bigg]>\bigg] \\
&+ \frac{L\eta^2}{2}\mathbb{E}\bigg[\bigg|\bigg| \frac{1}{M+B}\bigg[\sum_{m=1}^{M}sto\text{-}sign(\boldsymbol{g}_{m}^{(t)},\boldsymbol{b})+\sum_{j=1}^{B}byzantine\text{-}sign(\boldsymbol{g}_{j}^{(t)})\bigg]\bigg|\bigg|^2_{2}\bigg] \\
& -  \eta\mathbb{E}\bigg[<\nabla F(w^{(t)}), \frac{1}{M+B}\bigg[\sum_{m=1}^{M}sto\text{-}sign(\boldsymbol{g}_{m}^{(t)},\boldsymbol{b})+\sum_{j=1}^{B}byzantine\text{-}sign(\boldsymbol{g}_{j}^{(t)})\bigg]>\bigg].\\
\end{split}
\end{equation}

For the first term, we have
\begin{equation}
\begin{split}
&<\nabla F(w^{(t)}) - \nabla F(y^{(t)}), \frac{1}{M+B}\bigg[\sum_{m=1}^{M}sto\text{-}sign(\boldsymbol{g}_{m}^{(t)},\boldsymbol{b})+\sum_{j=1}^{B}byzantine\text{-}sign(\boldsymbol{g}_{j}^{(t)})\bigg]> \\
& \leq \frac{\eta}{2}\bigg|\bigg|\frac{1}{M+B}\bigg[\sum_{m=1}^{M}sto\text{-}sign(\boldsymbol{g}_{m}^{(t)},\boldsymbol{b})+\sum_{j=1}^{B}byzantine\text{-}sign(\boldsymbol{g}_{j}^{(t)})\bigg]\bigg|\bigg|^2_{2} + \frac{1}{2\eta}||\nabla F(w^{(t)})-\nabla F(y^{(t)})||^2 \\
&\leq \frac{\eta d}{2} + \frac{L^2}{2\eta}||y^{(t)} - w^{(t)}||^2 \\
&= \frac{\eta d}{2}+ \frac{L^2\eta}{2}||\tilde{\boldsymbol{e}}^{(t)}||^2 \\
&\leq \frac{\eta d}{2} + \frac{L^2\eta\beta d}{2}.
\end{split}
\end{equation}

For the third term, if $B < \frac{|\sum_{m=1}^{M}(\boldsymbol{g}_{m}^{(t)})_{i}|}{b}$, we have
\begin{equation}
\begin{split}
&-\mathbb{E}\bigg[<\nabla F(w^{(t)}), \frac{1}{M+B}\bigg[\sum_{m=1}^{M}sto\text{-}sign(\boldsymbol{g}_{m}^{(t)},\boldsymbol{b})+\sum_{j=1}^{B}byzantine\text{-}sign(\boldsymbol{g}_{j}^{(t)})\bigg]>\bigg] \\
&=-\mathbb{E}\bigg[\sum_{i=1}^{d}\nabla F(w^{(t)})_{i}\frac{1}{M+B}\bigg[\sum_{m=1}^{M}sto\text{-}sign(\boldsymbol{g}_{m}^{(t)},\boldsymbol{b})+\sum_{j=1}^{B}byzantine\text{-}sign((\boldsymbol{g}_{j}^{(t)})_{i})\bigg]\bigg]\\
&\leq -\sum_{i=1}^{d}|\nabla F(w^{(t)})_{i}|\frac{1}{M+B}\bigg[\frac{|\sum_{m=1}^{M}(\boldsymbol{g}_{m}^{(t)})_{i}|}{b}-B\bigg]\\
&\leq -c||\nabla F(w^{(t)})||_{1},
\end{split}
\end{equation}
where $c$ is some positive constant.

Following the same analysis as that in the proof of Theorem \ref{EFDPSIGNConvergence2}, the convergence of Algorithm 2 can be established.
\end{proof}

\section{Discussions about $dp\text{-}sign$ with $\delta=0$}\label{DPLAPLACE}
In this section, we present the differentially private compressor $dp\text{-}sign$ with $\delta=0$.
\begin{Definition}
For any given gradient $\boldsymbol{g}_{m}^{t}$, the compressor $dp\text{-}sign$ outputs $dp\text{-}sign(\boldsymbol{g}_{m}^{t},\epsilon,0)$. In particular, the $i$-th entry of $dp\text{-}sign(\boldsymbol{g}_{m}^{t},\epsilon,0)$ is given by
\begin{equation}\label{SPdpsignsgd}
dp\text{-}sign(\boldsymbol{g}_{m}^{t},\epsilon,0)_{i} =
\begin{cases}
\hfill 1, \hfill &\text{with probability $\frac{1}{2}+\frac{1}{2}sign((\boldsymbol{g}_{m}^{t})_{i})\big(1-e^{-\frac{|(\boldsymbol{g}_{m}^{t})_{i}|}{\lambda}}\big)$,}\\
\hfill -1, \hfill &\text{with probability $\frac{1}{2}-\frac{1}{2}sign((\boldsymbol{g}_{m}^{t})_{i})\big(1-e^{-\frac{|(\boldsymbol{g}_{m}^{t})_{i}|}{\lambda}}\big)$,}\\
\end{cases}
\end{equation}
where $\lambda = \frac{\Delta_{1}}{\epsilon}$ and $\Delta_1$ is the sensitivity measures defined in (\ref{sensitivity}).
\end{Definition}

\begin{theorem}
The proposed compressor $dp\text{-}sign(\cdot,\epsilon,0)$ is $(\epsilon,0)$-differentially private.
\end{theorem}

\begin{proof}
Consider any vector $\boldsymbol{a}$ and $\boldsymbol{b}$ such that $||\boldsymbol{a} - \boldsymbol{b}||_{1} \leq \Delta_{1}$ and $\boldsymbol{v} \in \{-1,1\}^{d}$, we have
\begin{equation}
\begin{split}
\frac{P(dp\text{-}sign(\boldsymbol{a},\epsilon,0)=\boldsymbol{v})}{P(dp\text{-}sign(\boldsymbol{b},\epsilon,0)=\boldsymbol{v})} = \frac{\int_{D}e^{-\frac{||\boldsymbol{x}-\boldsymbol{a}||}{\lambda}}d\boldsymbol{x}}{\int_{D}e^{-\frac{||\boldsymbol{x}-\boldsymbol{b}||}{\lambda}}d\boldsymbol{x}},
\end{split}
\end{equation}
where $D$ is some integral area depending on $\boldsymbol{v}$. It can be verified that $e^{-\epsilon} \leq |\frac{e^{-\frac{||\boldsymbol{x}-\boldsymbol{a}||}{\lambda}}}{e^{-\frac{||\boldsymbol{x}-\boldsymbol{b}||}{\lambda}}}| \leq e^{\epsilon}$ always holds, which indicates that $e^{-\epsilon} \leq |\frac{P(dp\text{-}sign(\boldsymbol{a},\epsilon,0)=\boldsymbol{v})}{P(dp\text{-}sign(\boldsymbol{b},\epsilon,0)=\boldsymbol{v})}| \leq e^{\epsilon}$.
\end{proof}

\begin{theorem}
Let $u_{1},u_{2},\cdots,u_{M}$ be $M$ known and fixed real numbers. Further define random variables $\hat{u}_{i}=dp\text{-}sign(u_{i},\epsilon,\delta), \forall 1\leq i \leq M$. Then there always exist a constant $\sigma_{0}$ such that when $\sigma \geq \sigma_{0}$, $P(sign(\frac{1}{M}\sum_{m=1}^{M}\hat{u}_{i})\neq sign(\frac{1}{M}\sum_{m=1}^{M}u_{i})) <\big(1-x^2\big)^{\frac{M}{2}}$,
where $x = \frac{|\sum_{m=1}^{M}u_{m}|}{\gamma\lambda M}$ and $\gamma$ is some positive constant.
\end{theorem}

\begin{proof}
Without loss of generality, assume $u_{1} \leq u_{2} \leq \cdots \leq u_{K} < 0 \leq u_{K+1} \leq \cdots \leq u_{M}$ and $\frac{1}{M}\sum_{i=1}^{M}u_{i} < 0$. Note that similar analysis can be done when $\frac{1}{M}\sum_{i=1}^{M}u_{i} > 0$.

We are interested in obtaining $\Bar{p}_{dp} = \frac{1}{M}\sum_{m=1}^{M}\Pr\left(sign\left(\frac{1}{M}\sum_{m=1}^{M}u_{m}\right) \neq \hat{u}_{m}\right)$, which is given by
\begin{equation}
\begin{split}
\Bar{p}_{dp} = \frac{1}{2} - \frac{1}{M}\bigg[\sum_{m=1}^{K}P(u_{m} < n < 0)-\sum_{m=K+1}^{M}P(0 < n < u_{m})\bigg].
\end{split}
\end{equation}

where $n \sim Laplace(0,\lambda)$. Similar to the analysis for $dp\text{-}sign$ with $\delta > 0$, we can show that
\begin{equation}
\begin{split}
&\sum_{m=1}^{K}P(u_{m} < n < 0)-\sum_{m=K+1}^{M}P(0 < n < u_{m}) \\
& = P\bigg(u_{1} < n \leq -\sum_{m=2}^{M}u_{m}\bigg) + P\bigg(-\sum_{m=2}^{M}u_{m} < n < 0\bigg) - \sum_{m=2}^{M}P(0 < n < u_{m})\\
& > \bigg|\sum_{m=1}^{M}u_{m}\bigg|\bigg[\frac{1}{2\lambda}e^{-\frac{|u_{1}|}{\lambda}}\bigg] + P\bigg(-\sum_{m=2}^{M}u_{m} < n < 0\bigg) - \sum_{m=2}^{M}P(0 < n < u_{m}) \\
& > \bigg|\sum_{m=1}^{M}u_{m}\bigg|\bigg[\frac{1}{2\lambda}e^{-\frac{|u_{1}|}{\lambda}}\bigg] - \bigg|\sum_{m=2}^{M}u_{m}\bigg|\frac{1}{2\lambda}\bigg[1-e^{-\frac{|\sum_{m=2}^{M}u_{m}|}{\lambda}}\bigg] \\
&= \frac{1}{2\lambda}\bigg[\bigg|\sum_{m=1}^{M}u_{m}\bigg|e^{-\frac{|u_{1}|}{\lambda}}+\bigg|\sum_{m=2}^{M}u_{m}\bigg|\bigg[e^{-\frac{|\sum_{m=2}^{M}u_{m}|}{\lambda}}-1\bigg]\bigg].
\end{split}
\end{equation}

As a result, there exists a $\lambda_{0}$ such that when $\lambda \geq \lambda_{0}$, we have
\begin{equation}
\Bar{p}_{dp} = \frac{1}{M}\sum_{m=1}^{M}P(X_{m} = 1) \leq \frac{1}{2} + \frac{\sum_{m=1}^{M}u_{m}}{2M\lambda\gamma},
\end{equation}
where $\gamma$ is some constant larger than 1.
Following the same analysis as that in the proof of Corollary \ref{SPLemma1}, we can show that
\begin{equation}
P\bigg(sign\bigg(\frac{1}{M}\sum_{m=1}^{M}\hat{u}_{i}\bigg)\neq sign\bigg(\frac{1}{M}\sum_{m=1}^{M}u_{i}\bigg)\bigg) <\big(1-x^2\big)^{\frac{M}{2}},
\end{equation}
where $x = \frac{|\sum_{m=1}^{M}u_{m}|}{\gamma\lambda M}$ and $\gamma$ is some positive constant.
\end{proof}

\section{Discussions about the server's compressor $\mathcal{C}(\cdot)$ in Algorithm 2}\label{dpsignextend}
{\color{black}
In the following, we show that for the 1-bit compressor $q(\boldsymbol{g}_{m}^{(t)})$,
\begin{equation}\label{induc}
\bigg|\bigg|\frac{1}{M}\sum_{m=1}^{M}q(\boldsymbol{g}_{m}^{(t)})+\tilde{\boldsymbol{e}}^{(t)} - \frac{1}{M}sign\bigg(\frac{1}{M}\sum_{m=1}^{M}q(\boldsymbol{g}_{m}^{(t)})+\tilde{\boldsymbol{e}}^{(t)}\bigg)\bigg|\bigg|_2^{2} < \bigg|\bigg|\frac{1}{M}\sum_{m=1}^{M}q(\boldsymbol{g}_{m}^{(t)})+\tilde{\boldsymbol{e}}^{(t)}\bigg|\bigg|_{2}^{2},
\end{equation}

For the ease of presentation, we let $\boldsymbol{r}^{(t)}_{i} < \infty$ denote the $i$-th entry of $\frac{1}{M}\sum_{m=1}^{M}q(\boldsymbol{g}_{m}^{(t)})+\tilde{\boldsymbol{e}}^{(t)}$. Then, we can rewrite the left-hand side of (\ref{induc}) as follows,
\begin{equation}
\bigg|\bigg|\frac{1}{M}\sum_{m=1}^{M}q(\boldsymbol{g}_{m}^{(t)})+\tilde{\boldsymbol{e}}^{(t)} - \frac{1}{M}sign\bigg(\frac{1}{M}\sum_{m=1}^{M}q(\boldsymbol{g}_{m}^{(t)})+\tilde{\boldsymbol{e}}^{(t)}\bigg)\bigg|\bigg|_2^{2} = \sum_{i=1}^{d}\bigg(\boldsymbol{r}^{(t)}_{i}-\frac{1}{M}sign(\boldsymbol{r}^{(t)}_{i})\bigg)^2.
\end{equation}
In particular, we have
\begin{equation}
\bigg(\boldsymbol{r}^{(t)}_{i}-\frac{1}{M}sign(\boldsymbol{r}^{(t)}_{i})\bigg)^2 = \bigg((\boldsymbol{r}^{(t)}_{i})^2 + \frac{1}{M^2} - \frac{2|\boldsymbol{r}^{(t)}_{i}|}{M}\bigg) = \bigg[1 - \frac{1}{M(\boldsymbol{r}^{(t)}_{i})^2}\bigg(2|\boldsymbol{r}^{(t)}_{i}| - \frac{1}{M}\bigg)\bigg](\boldsymbol{r}^{(t)}_{i})^2.
\end{equation}
If $2|\boldsymbol{r}^{(t)}_{i}| - \frac{1}{M} > 0, \forall i$, then
\begin{equation}
\sum_{i=1}^{d}\bigg(\boldsymbol{r}^{(t)}_{i}-\frac{1}{M}sign(\boldsymbol{r}^{(t)}_{i})\bigg)^2 < \sum_{i=1}^{d}(\boldsymbol{r}^{(t)}_{i})^2 = \bigg|\bigg|\frac{1}{M}\sum_{m=1}^{M}q(\boldsymbol{g}_{m}^{(t)})+\tilde{\boldsymbol{e}}^{(t)}\bigg|\bigg|_{2}^{2}.
\end{equation}
In order to prove that $2|\boldsymbol{r}^{(t)}_{i}| - \frac{1}{M} > 0, \forall i$, we first show that $M(\tilde{\boldsymbol{e}}^{(t)})_{i}$ is an even number for any $t$ by induction. In particular, according to Assumption 3 and $(\tilde{\boldsymbol{e}}^{(0)})_{i} = 0$, $M\boldsymbol{r}^{(0)}_{i} = \sum_{m=1}^{M}q(\boldsymbol{g}^{(0)}_{m})_{i}$ is an odd number. Therefore, $M(\tilde{\boldsymbol{e}}^{(1)})_{i} = \sum_{m=1}^{M}q(\boldsymbol{g}^{(0)}_{m})_{i} - sign(\sum_{m=1}^{M}q(\boldsymbol{g}_{m}^{(t)})_{i})$ is an even number. In addition,
\begin{equation}
M(\tilde{\boldsymbol{e}}^{(t+1)})_i = \sum_{m=1}^{M}q(\boldsymbol{g}_{m}^{(t)})_{i} + M(\tilde{\boldsymbol{e}}^{(t)})_{i} - sign\bigg(\sum_{m=1}^{M}q(\boldsymbol{g}_{m}^{(t)})_{i} + M(\tilde{\boldsymbol{e}}^{(t)})_{i}\bigg).
\end{equation}
Given that $M(\tilde{\boldsymbol{e}}^{(t)})_{i}$ is even, we can show that $M(\tilde{\boldsymbol{e}}^{(t+1)})_i$ is even as well. Therefore, $M\boldsymbol{r}^{(t)}_{i} = \sum_{m=1}^{M}q(\boldsymbol{g}_{m}^{(t)})_{i} + M(\tilde{\boldsymbol{e}}^{(t)})_i$ is odd and $2|\boldsymbol{r}^{(t)}_{i}| \geq \frac{2}{M} > \frac{1}{M}$, $\forall t,i$.
}

\section{Details of the Implementation}\label{DetailsImplementation}
Our experiments are mainly implemented using Python 3.7.4 with packages TensorFlow 2.4.1 and numpy 1.19.2. One Intel i7-9700 CPU with 32 GB of memory and one NVIDIA GeForce RTX 2070 SUPER GPU are used in the experiments.
\subsection{Dataset and Pre-processing}
We perform experiments on the standard MNIST dataset and the CIFAR-10 dataset. MNIST is for handwritten digit recognition consisting of 60,000 training samples and 10,000 testing samples. Each sample is a 28$\times$28 size gray-level image. We normalize the data by dividing it with the max RGB value (i.e., 255.0). The  CIFAR-10 dataset contains 50,000 training samples and 10,000 testing samples. Each sample is a 32$\times$32 color image. The data are normalized with zeor-centered mean.
\subsection{Dataset Assignment}
In our experiments, we consider 31 normal workers and measure the data heterogeneity by the number of labels of data that each worker stores. We first partition the training dataset according to the labels. For each worker, we randomly generate a set of size $n$ which indicates the labels of training data that should be assigned to this worker. Then, a subset of training data from the corresponding labels is randomly sampled and assigned to the worker without replacement. The size of the subset depends on $n$ and the size of the training data for each label. More specifically, we set the size of the subset as $\lfloor60000/(31n)\rfloor$ for MNIST ($\lfloor50000/(31n)\rfloor$ for CIFAR-10) in the beginning. When there are not enough training data for a label, we reduce the size of the subset accordingly. We consider the scenarios that all the workers have the same number of distinct labels (i.e., the same $n$ for all the workers). For the results in Table 1, we set $n=2, 4$ for ``2 LABELS", ``4 LABELS", respectively. For the rest of the results, we set $n=1$.

\subsection{Neural Network Setting}
For MNIST, we implement a two-layer fully connected neural network with softmax of classes with cross-entropy loss. The hidden layer has 128 hidden ReLU units. For CIFAR-10, we implement VGG9 with 7 convolution layers. It has two contiguous blocks of two convolution layers with 64 and 128 channels, respectively, followed by a max-pooling, then it has one blocks of three convolution layers with 256 channels followed by max-pooling, and at last, we have one dense layer with 512 hidden units.

\subsection{Learning Rate Tuning}
For {\scriptsize Sto-SIGN}SGD and {\scriptsize SIGN}SGD, we use a constant learning rate $\eta$ for MNIST and tune the parameters from the set $\{1,0.1,0.01,0.005,0.003,0.001,0.0001\}$. For CIFAR-10, we tune the initial learning rate from the set $\{1,0.1,0.01,0.001,0.0001\}$, which is reduced by a factor of 2, 5, 10 and 20 at iteration 1,500, 3,000, 5,000 and 7,000, respectively. For FedAvg, the initial learning rates are tuned from the set $\{0.001,0.01,0.1,0.2,0.3,0.4,0.5,0.6,0.7,0.8,0.9,1,1.1,1.2,1.3,1.4,1.5\}$ and the set $\{0.001,0.01,0.1,0.2,0.3,0.4,0.5,0.6,0.7,0.8,0.9,1\}$ for MNIST and CIFAR-10, respectively. For MNIST, a learning rate decay of 0.99 per communication round is used, while for CIFAR-10, the learning rate decay is 0.996 per communication round.
\bibliography{Ref-FL}
\bibliographystyle{IEEEtran}
\end{document}